\documentclass[12pt]{alt2023} %

\title[An Instance-Dependent Analysis for the Cooperative Multi-Player Multi-Armed Bandit]{An Instance-Dependent Analysis for the Cooperative Multi-Player Multi-Armed Bandit}
\usepackage{times}

\usepackage{microtype}
\usepackage{graphicx}
\usepackage{booktabs} %
\usepackage{amssymb,amsfonts,amsmath}

\usepackage{algorithmic}
\usepackage{algorithm}
\usepackage[algo2e]{algorithm2e} 
\usepackage{algorithm,algorithmic}
\usepackage{thmtools,thm-restate}

\usepackage{hyperref}
\usepackage{cleveref}
\usepackage{wrapfig}

\renewenvironment{proof}{\par\noindent{\bf Proof\ }}{\hfill\BlackBox\\[2mm]}
\newtheorem{assumption}[theorem]{Assumption}

\DeclareMathOperator*{\argmax}{arg\,max}

\renewcommand{\epsilon}{\varepsilon}
\renewcommand{\ln}{\log}

\altauthor{%
 \Name{Aldo Pacchiano} \Email{apacchiano@microsoft.com}\\
 \addr Microsoft Research, NYC
 \AND
 \Name{Peter Bartlett} \Email{peter@berkeley.edu}\\
 \addr University of California, Berkeley%
 \AND
 \Name{Michael Jordan} \Email{jordan@cs.berkeley.edu }\\
 \addr University of California, Berkeley
}

\begin{document}

\maketitle

\begin{abstract}%
We study the problem of information sharing and cooperation in Multi-Player Multi-Armed bandits. We propose the first algorithm that achieves logarithmic regret for this problem when the collision reward is unknown. Our results are based on two innovations. First, we show that a simple modification to a successive elimination strategy can be used to allow the players to estimate their suboptimality gaps, up to constant factors, in the absence of collisions. Second, we leverage the first result to design a communication protocol that successfully uses the small reward of collisions to coordinate among players, while  preserving meaningful instance-dependent logarithmic regret guarantees.
\end{abstract}

\begin{keywords}%
  List of keywords%
\end{keywords}

\section{Introduction}

We consider the cooperative Multi-Player version of the Multi-Armed bandit problem. Generalizing the single-player case, the bandit instance is defined by mean rewards, $\boldsymbol{\mu} = (\mu_1, \cdots, \mu_K )\in [0,1]^K$, all of which are unknown to each of $M$ players. There is a permutation $\sigma \in \mathbb{S}_K$ (unknown to the players) such that $\mu_{\sigma_1} \geq \mu_{\sigma_2} \cdots \geq \mu_{\sigma_K}$.  At each time step $t=1, \cdots, T$, each of the players $p \in [M]$ chooses an action  $i^p_t \in [K]$ and observes the corresponding (random) reward. If two players pull the same arm, they receive a reward sampled from a distribution with unknown mean $\mu_{\mathrm{collision}} \leq \mu_{\sigma_K}$. Our objective will be to design an algorithm with sublinear pseudo-regret:
\begin{equation*}
    \mathcal{R}_T = T \max_{\mathbf{a} \in \{0,1\}^K : \sum_{i=1}^K a(i) = M} \langle\mathbf{a}, \boldsymbol{\mu} \rangle - \sum_{t=1}^T \sum_{p=1}^M \mu_{i_t^p}.
\end{equation*}
As opposed to other work, such as~\citet{avner2014concurrent,rosenski2016multi,alatur2020multi}, we do not make the assumption that collisions are announced to the players; rather, we simply assume that whenever two players select the same arm, they both observe a reward with mean $\mu_{\mathrm{collision}}$. The players do not know for certain if there was a collision or not~\citep[cf.][]{boursier2018sic,shi2020decentralized,bubeck2020coordination}. We will make use of the implicit information provided by collisions (a very low reward value) to design a communication protocol that will allow players to coordinate. We make the following boundedness assumption on the distribution of the reward signals observed by the players:

 \begin{assumption}\label{assumption::bounded_support}
 All $K$ arms have bounded distributions with support in $[0,1]$. %
 \end{assumption}

Our main contribution is to design an elimination-based algorithm whose regret satisfies instance-dependent logarithmic bounds without assuming explicit knowledge of collision information. Our algorithm generalizes all previous approaches (for example those where collisions are announced~\cite{avner2014concurrent,rosenski2016multi,alatur2020multi} or even those where collisions are unannounced but their reward equals zero~\cite{huang2021towards}) to this problem since it does not require knowledge of the mean collision reward. We use the following notation to refer to the suboptimality gaps among the arms: 
\begin{equation*}
    \Delta_{\sigma_i,\sigma_j} = \mu_{\sigma_i} - \mu_{\sigma_j},
\end{equation*}
where $i < j$. Our main result can be summarized as follows.
\begin{theorem}[simplified]\label{theorem::main}
There exists a strategy such that the regret is upper bounded by:
\begin{equation*}
      \mathcal{R}_T \leq \widetilde{ \mathcal{O}}\left(\frac{M(K-M)K^2\log(T)}{\Delta_{\sigma_M, \sigma_{M+1}}} + \mathbf{poly}(\log(T), K, M ) \right) ,
\end{equation*}
with probability at least $1-\min\left(\frac{1}{T},\frac{K}{81} \right)$ where $\widetilde{ \mathcal{O}}(\cdot)$ hides factors logarithmic in $M$ and $K$ only and $\mathbf{poly}(\log(T), K, M )$ is linear in $\log(T)$.

\end{theorem}

\section{Previous Work}

The Multi-Player bandit problem with bounded communication was first introduced in \citet{lai2008medium,liu2010distributed,anandkumar2011distributed}, and has been extensively studied since then under various assumptions on the communication patterns and the nature of the collisions~\citep{avner2014concurrent,rosenski2016multi,palicot2018multi,lugosi2018multiplayer,boursier2018sic,alatur2020multi,bubeck2020non}.  Perhaps the first instance of a centralized version of the problem we study in this work appeared in~\citet{anantharam1987asymptotically}, where the problem of a single player selecting multiple arms simultaneously is considered. 

The problem of Multi-Player, Multi-Armed bandits has commonly been motivated via its application to wireless communication and networking~\citep{liu2010distributed,rosenski2016multi}; for example as a way to model the case where several users must access a wireless channel in a decentralized manner~\citep{besson2018multi}.

We can classify the existing settings and algorithmic approaches to the Multi-Player Multi-Armed in broadly two categories. When collision information is available to the players and when it is not. In the first category, algorithms such as SIC-MMAB~\citep{boursier2018sic} or DPE1~\citep{wang2020optimal}  have been developed, the second of which achieves the same asymptotic regret as that obtained by an optimal centralized algorithm. Both of these algorithms crucially exploit the known collision information to establish communication between the players.

In the ``No Sensing'' setting where collision information is not readily available to the players, but instead players receive a diminished or zero reward, the problem of developing an optimal algorithm is substantially more challenging and has not been fully solved yet. Most importantly the three most prominent algorithms, SIC-MMAB2~\citep{boursier2018sic}, EC-SIC~\citep{shi2020decentralized} and the algorithm of~\citet{lugosi2018multiplayer} suffer from a variety of drawbacks. 

SIC-MMAB2 satisfies a regret guarantee of order $\mathcal{O}\left(\sum_{i > M} \frac{M\log(T)}{\Delta_{\sigma_M, \sigma_{i}}} + \frac{MK^2}{\mu_{\sigma_K}} \log(T) \right) $. Unfortunately SIC-MMAB2 is suboptimal in two ways. First the algorithm requires knowledge of $\mu_{\sigma_K}$, and second its regret guarantee suffers an inverse dependence on $\mu_{\sigma_K}$, a quantity that may be astronomically large. Other algorithms for the No Sensing setting such as studied in Theorem 1.2 in~\citet{lugosi2018multiplayer}, and the ADAPTED SIC-MMAB algorithm from~\citet{boursier2018sic} suffer from the same limitations (see Table 1 in~\citet{boursier2018sic}). The more recent EC-SIC algorithm~\citep{shi2020decentralized} improves the $M$ and $K$ dependence from the SIC-MMAB2 regret upper bound but still suffers from the substantial drawback of requiring knowledge of at least a lower bound to $\mu_{\sigma_K}$. Although EC-SIC achieves a better regret guarantee than SIC-MMAB2, it also requires knowledge of the gap $\Delta_{\sigma_M, \sigma_{M+1}}$ to be available to all players. Other algorithms such as the algorithm from Theorem 1.1 in~\citet{lugosi2018multiplayer} suffer from more serious problems such as quadratic dependence on the inverse gap $\frac{1}{\Delta_{\sigma_M, \sigma_{M+1}}}$. 

Some of these drawbacks have been addressed by recent work~\citep{huang2021towards}. Under the assumption the collision reward equals $0$, the authors dispense with the assumptions of shared knowledge of both $\mu_{\sigma_K}$ and $\Delta_{\sigma_M, \sigma_{M+1}}$. Their collision communication protocol makes use of a test that is very much in the spirit of ours (see the subroutine $\mathrm{CollisionTest}$ in Section~\ref{section::communication_analysis}). Communication is achieved by finding large arms (that are up to a constant proportion the scale of the largest arm) and pulling them. The authors manage to obtain a logarithmic instance dependent regret guarantee scaling as $\mathcal{O}\left( \sum_{i > M} \frac{\log(T)}{\Delta_{\sigma_M, \sigma_{i}}} + MK^2 \log(T) + KM^2 \log\left( \frac{1}{\Delta_{\sigma_M, \sigma_{M+1}}}   \right)^2\right)$. Unfortunately, their algorithm heavily depends on the assumption $\mu_{\mathrm{collision}} = 0$. 

In the present paper we avoid the aforementioned drawbacks and derive the first \emph{truly} logarithmic problem-dependent guarantee for the No Sensing Multi-Player Multi-Armed bandit problem with unknown collision rewards. We leverage the implicit communication that exists when collisions occur, namely that the mean collision reward is small. We show that a simple modification of a successive elimination strategy can be used to allow the players to estimate the suboptimality gaps up to constant factors in the absence of collisions. Using this result we design a communication protocol that successfully leverages the small reward of collisions to coordinate among players, while at the same time preserves meaningful instance-dependent logarithmic regret guarantees. %

A different setting for the Multi-Player Multi-Armed bandit problem is one in which the players are required to avoid all collisions. It was shown by \citet{bubeck2020coordination} that one can obtain the optimal regret in this setting without any collisions at all. Their result was limited to two players and three actions. A more recent version of that result~\citet{bubeck2020cooperative} shows that it is possible to achieve a regret scaling with $\sqrt{T}$ for the cooperative stochastic Multi-Player Multi-Armed bandit problem with a dependence of $K^{11}M$ in the number of arms $K$ and the number of players $M$. The algorithmic strategy proposed in~\citet{bubeck2020cooperative} relies on a clever algorithm that, with high probability, avoids collisions altogether. More recent results~\cite{liu2022pareto} suggest that achieving a logarithmic instance dependent rate is impossible in the absence of communication.

\section{Assumptions and Notation}\label{section::assumptions_notation}

 Our algorithm is based on the idea of exploiting a communication protocol that leverages collisions while maintaining favorable regret guarantees. In comparison to other work, we do not make the assumption that collisions are announced to the players; rather, we simply assume that whenever two players select the same arm, they both observe an i.i.d. reward with mean $\mu_{\mathrm{collision}} \leq \mu_{\sigma_K}$. Throughout the paper we will use the notation $t$ to index the rounds of play. In each round all $M$ players select an arm and collect a reward. 
 
  We denote by $N_i^p(t)$ the (random) number of pulls of arm $i$ by player $p$ up until and including round $t$. And let $\widehat{\mu}_i^{p}(t)$ be the empirical estimator maintained by player $p$ of the mean reward $\mu_i$ of arm $i$ at time $t$. This estimator consists of an average of $N_i^p(t)$ samples. Similarly for any $t,t'$ denote by $\widehat{\mu}_i^p(t:t')$ as the empirical mean estimator of arm $i$ computed by player $p$ during rounds $t$ to $t'$ (inclusive). Let $\delta \in (0,1)$ be a probability parameter. We will make use of the following confidence interval diameter function,
  \begin{center}
  $D : \mathbb{N} \rightarrow \mathbb{R}_{+}$ such that $D(n) = \sqrt{ \frac{2g(n)}{n}}$ where $g(n) = \ln(4n^2MK/\delta)$. 
 \end{center}
As a simple consequence of Hoeffding's inequality, for any $p \in [M]$ and $i \in [K]$, with probability at least $1-\frac{\delta}{MK}$, for all $t \in \mathbb{N}$ simultaneously, we have:
 \begin{equation}\label{equation::confidence_interval_basic}
     | \widehat{\mu}_i^p(t) - \mu_i | \leq D(N_i^p(t)).
 \end{equation}
\paragraph{The Good Event $\mathcal{E}$.} We will denote the (at least) $1-\delta$ probability event that all confidence intervals from Equation~\ref{equation::confidence_interval_basic} hold for all $p \in [M]$, all $i \in [K]$ and all $t \in \mathbb{N}$ simultaneously as $\mathcal{E}$.

\paragraph{Round Robin Schedule.} Whenever we say the arms are pulled by the players in a Round Robin schedule, we mean that during the first round player $p$ will pull arm $p$, and in subsequent rounds player $p$ will pull the arm with an index one more than the one she pulled in the previous rounds, unless she has pulled arm $K$ in which case she will pull arm $1$ the next round. Whenever all players are pulling arms according to a Round Robin schedule, they do not collide.

\paragraph{Special Rounds.} We will refer to all rounds occurring right after a complete cycle of a Round Robin round as \emph{special rounds}. At the beginning of time, before any arm is eliminated, this will occur exactly during rounds that are multiples of $K$.

We also make the following assumptions.

\begin{assumption}[Collisions]\label{assumption::collisions_zero}
Whenever two players collide, both players get a reward sampled from a distribution with mean $\mu_{\mathrm{collision}}$ such that $\mu_{\mathrm{collision}} \leq \mu_{\sigma_K}$.
\end{assumption}

Assumption~\ref{assumption::collisions_zero} is not particularly limiting. Our main contribution is to design an algorithm that does not require the \emph{identity} of colliding arms to be announced. Since we do not assume the collision reward to equal zero, the algorithm of~\cite{huang2021towards} is not applicable to our case. The techniques in~\cite{huang2021towards} rely on the overwhelming probability of seeing a nonzero after repeated sampling of a non-zero mean arm. This cannot be used in the setting where the collision reward may have a nonzero mean. 
\begin{assumption}[Shared knowledge]\label{assumption::labeled_arms}
All arms are labeled, and the labels are known by all players $p \in [M]$. 
\end{assumption}

Assumption~\ref{assumption::labeled_arms} is mild in comparison with the shared randomness assumption of previous works such as~\citet{bubeck2020cooperative} and \citet{bubeck2020coordination}. We also assume common knowledge of the problem-independent functions $f$, $B$ and $g$. %

\section{Algorithm Overview and Analysis}\label{section::algorithm}

Our algorithm starts by having all players agree on a value $t_{\mathrm{collision-test}}$ satisfying $t_{\mathrm{collision-test}} =\Theta( \frac{\log(1/\delta)}{\Delta^2_{\mathrm{collision}}})$ up to logarithmic factors where $\Delta_{\mathrm{collision}} = \mu_{\sigma_1} - \mu_{\mathrm{collision}}$ and with estimators $\widehat{\mu}_{\mathrm{collision}}^p $. Achieving this requires the same set of methods used in the following part of the algorithm, thus we defer its explanation. Subsequently all players will pull arms in a Round Robin fashion, thus not incurring any collisions. During these rounds all players maintain confidence bands for the mean values of each of the arms $i \in [K]$. All players maintain a connectivity graph with node set $[K]$, such that for any $i,j \in [K]$, the edge $(i,j)$ is present if their confidence intervals overlap. 

\begin{wrapfigure}{r}{0.5\textwidth} 
        \vspace{-5mm}
        \centering{\includegraphics[width=0.9\linewidth]{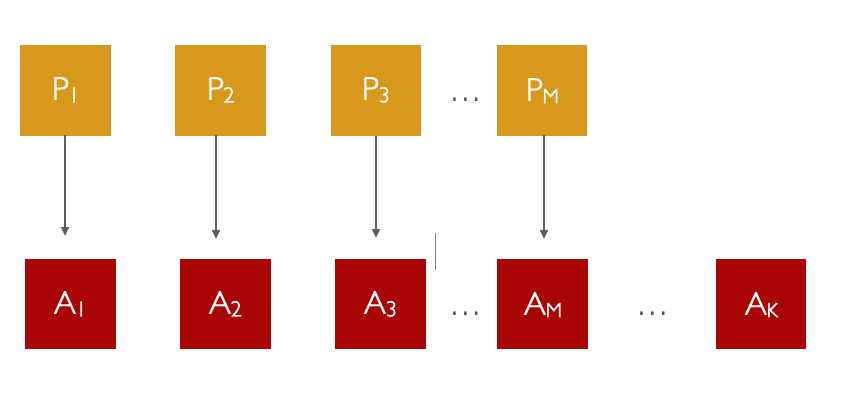}}
        --------------------------------------------------
        \centering{\includegraphics[width=0.9\linewidth]{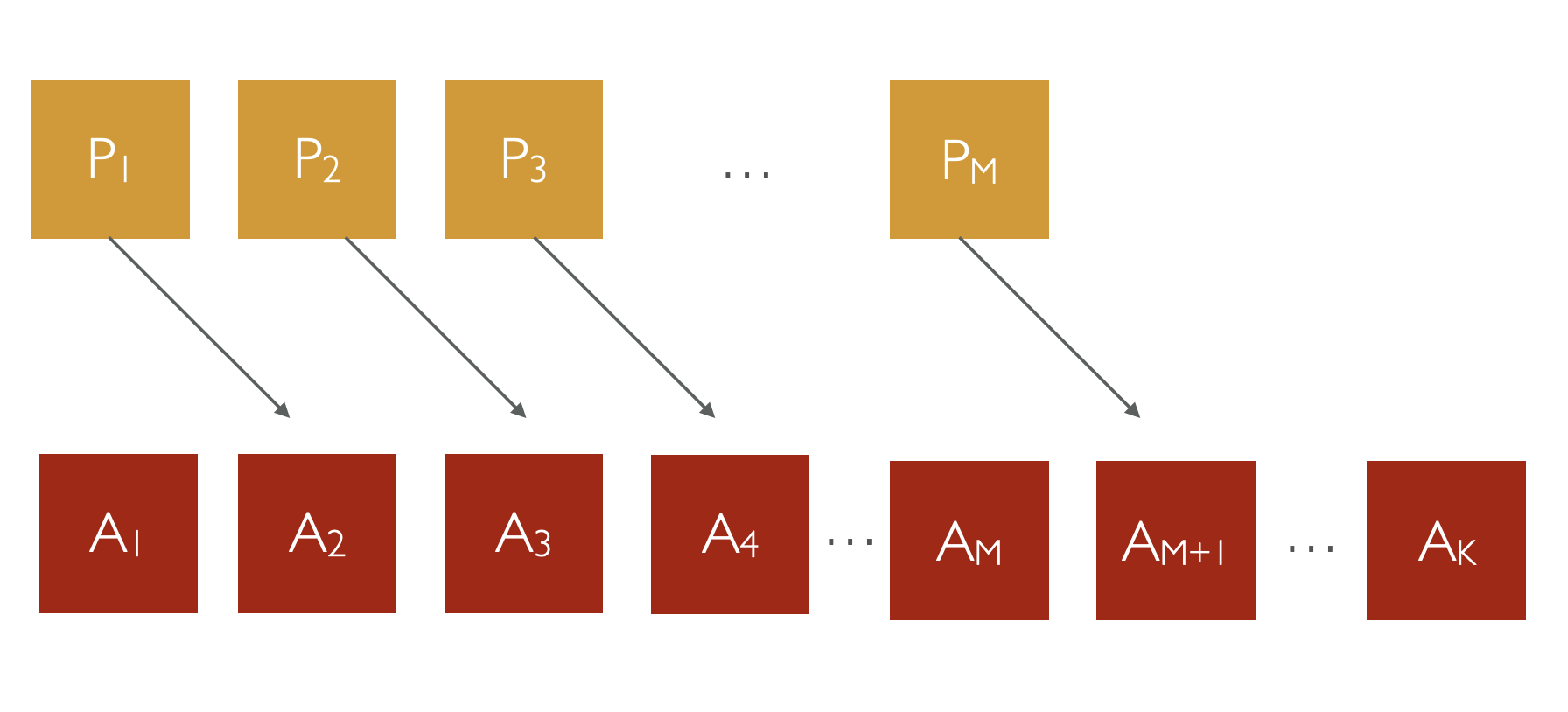}}
        \vspace{-5mm}
        \caption{Illustration of the Round Robin Schedule. \textbf{Top:} Step $1$. \textbf{Bottom:} Step $2$.}
        \label{fig::RoundRobin}
          \vspace{-5mm}
\end{wrapfigure}

Consider the first $\emph{special}$ round $t_{\mathrm{first}}^1$ when for player $1$, the number of connected components of this graph is more than one. We define $t^p_{\mathrm{first}}$ analogously for all other players $p \in [M]$. Let's refer to the connected component with the highest empirical means as $\mathcal{C}_{\mathrm{top}}$ to its complement as $\mathcal{C}_{\mathrm{bottom}} = [K]\backslash \mathcal{C}_{\mathrm{top}}$ .  Although this is not guaranteed, whenever this partition emerges the components $\mathcal{C}_{\mathrm{bottom}}$ and $\mathcal{C}_{\mathrm{top}}$ will be separated by a ``consecutive'' gap, $\min_{i \in \mathcal{C}_{\mathrm{top}}} \mu_{i} - \max_{j \in \mathcal{C}_{\mathrm{bottom}} }  \mu_j$, roughly proportional to $\max_{i } \Delta_{\sigma_i, \sigma_{i+1}}$.

 By definition of the connectivity graph, when $\mathcal{E}$ holds all arms in $\mathcal{C}_{\mathrm{top}}$ will have higher mean reward values than those of each of the arms in $\mathcal{C}_{\mathrm{bottom}}$. Our algorithm is designed to ensure that $t_{\mathrm{first}}^1$ does not occur too soon or too late. More specifically, we guarantee that up to logarithmic factors, $\frac{t^1_{\mathrm{first}}}{K} \approx \frac{1}{\max_{i } \Delta_{\sigma_i, \sigma_{i+1}}^2}$,thus ensuring that the signals $t_{\mathrm{first}}^p$ occur around the same time for all players. This is crucial to allow for successful communication between the players as this ensures player $1$ start communicating after all remaining players are ready for listening. Once $t_\mathrm{first}^1$ is triggered and player $1$ has hold of a connectivity graph with at least two connected components, she will make use of the low reward experienced by collisions to transmit the partition $(\mathcal{C}_{\mathrm{top}}, \mathcal{C}_{\mathrm{bottom}})$ of the arm space to the other players. This can be achieved by communicating $\mathcal{C}_\mathrm{top}$ to all players $p \neq 1$ using a bit string language agreed in advance. The main algorithmic challenge we face in designing this protocol is to ensure the collisions incurred during communication do not generate a substantial regret.

Once all other players have learned what arms belong to $\mathcal{C}_\mathrm{top}$, they will $\mathrm{RECURSE}$ by playing the same strategy as before but now restricted to the two subproblems induced by the set of arms $\mathcal{C}_{\mathrm{top}}$ and $\mathcal{C}_{\mathrm{bottom}}$. If $|\mathcal{C}_{\mathrm{top}}| < M$, the top indexed $| \mathcal{C}_{\mathrm{top}}|$  players will forevermore play following a Round Robin schedule over $\mathcal{C}_{\mathrm{top}}$ while the remaining $M-| \mathcal{C}_{\mathrm{top}}|$ players restart the exploration strategy over the remaining arms. If instead $| \mathcal{C}_{\mathrm{top}} | \geq M$ the $M$ players restart the exploration strategy over the set $\mathcal{C}_{\mathrm{top}}$. $\mathrm{RECURSE}$ reduces the problem to a smaller Cooperative Multi-Player Multi-Armed bandit. The players then re-index the arm labels\footnote{For example the players assigned to $\mathcal{C}_1$ will re-index these arms so they are labeled $1$ to $|\mathcal{C}_1|$ by switching the smallest label in $\mathcal{C}_1$ to a $1$, the second smallest to a $2$, and so on.}, re-computes $t_{\mathrm{collision-test}}$ (for simplicity of analysis) and starts playing in a Round Robin fashion within their assigned group, with the smallest indexed player in any group becoming the communicating player. Iterating over this procedure ensures the players converge to pull only the top $M$ arms. 

Communicating $\mathcal{C}_\mathrm{top}$ (and therefore $\mathcal{C}_\mathrm{bottom}$) can be done by sending a bit string of size $K$, where the bits corresponding to the arms in $\mathcal{C}_{\mathrm{top}}$ equal one and the remaining bits equal zero. Since the players have access to no signal other than the reward, and this can be modulated only when two players play the same arm, any communication between players needs to happen through collisions. With this in mind we design a communication mechanism that allows player $1$ to transmit a bit string to all the other players with high probability.

While all players $p \in \{ 2,\cdots,K\}$ continue playing in a Round Robin fashion, player $1$ will signal the start of the communication sequence at the second \emph{special} round $t^1_{\mathrm{comm1}}$ such that $\left\lfloor \frac{t^1_{\mathrm{comm1}}/K}{g(t^1_{\mathrm{comm1}}/K)} \right\rfloor$ is a power of nine and occurs after $t_{\mathrm{first}}^1$. At this time player $1$ will begin to pull arm $\widehat{\sigma}_1$---the arm with the largest empirical mean at time $t_{\mathrm{comm1}}^1$---for a number of rounds equal to $K t_{\mathrm{collision-test}}$. All other players $p \in\{2,\cdots, M\}$ will begin listening for the start of the communication signal from player $1$ at all \emph{special} rounds\footnote{We impose the restriction that $t_{\mathrm{listen}}^p$ is such that $\left\lfloor \frac{t_{\mathrm{listen}}^p/K-1}{g(t_{\mathrm{listen}}^p/K-1)} \right\rfloor$ is not a power of  nine. Similarly for $t_{\mathrm{comm1}}^1$.} $t_{\mathrm{listen}}^p$ such that $\left\lfloor \frac{t_{\mathrm{listen}}^p/K}{g(t_{\mathrm{listen}}^p/K)} \right\rfloor$  is a power of nine and occurs after $t_{\mathrm{first}}^p$.

Having computed empirical estimators of the arm's rewards using data up to time $t_{\mathrm{listen}}^p$, each of the players $p \in \{2, \cdots, M\}$ will have estimated the mean of $\widehat{\sigma}_1$ to sufficiently good accuracy to ensure that at the end of the next $Kt_{\mathrm{collision-test}}$ rounds any of them can detect if there have been collisions with player $1$ when pulling arm $\widehat{\sigma}_1$ during rounds $t^p_{\mathrm{listen}} + 1$ to $t^p_{\mathrm{listen}} + K t_{\mathrm{collision-test}}$. If player $p$ detects a small reward coming from an arm with a previously recorded high reward, it can conclude this arm is $\widehat{\sigma}_1$ and that player $1$ has pulled it to signal the start of a communication round. If instead, none of the high reward arms record a substantial deviation in their collected reward during rounds $t^p_{\mathrm{listen}} + 1$ to $t^p_{\mathrm{listen}} + K t_{\mathrm{collision-test}}$, player $p$ can conclude player $1$ has not started to communicate yet. By design our algorithm ensures that with high probability $t_{\mathrm{listen}}^p \leq t_{\mathrm{comm1}}^1$ and that after at most three trials $t_{\mathrm{listen}}^p = t_{\mathrm{comm1}}^1$. One of the main challenges in designing an algorithm with these properties is to ensure the listening players are able to listen for the incoming communication signal from the communicating player at the right time. The players $p \neq 1$ can only start listening for an incoming signal after they have collected enough samples to get an accurate enough estimator for $\widehat{\sigma}_1$. Since the time when this happens is a random variable we need to ensure both listening and communication protocols occur at times when all players have sufficiently accurate estimates. This is particularly challenging because the players are only aware of their own estimates.

Let's assume that $t_{\mathrm{listen}}^1 = t_{\mathrm{comm1}}^1$. If player $p \in \{ 2, \cdots, M\}$ detects a communication signal from player $1$ associated with arm $\widehat{\sigma}_1$, it will listen for the next $K^2t_{\mathrm{collision-test}}$ rounds (recall these ``listening'' players are still playing all arms in $[K]$ following a Round Robin schedule). Using the resulting $K t_{\mathrm{collision-test}}$ pulls of arm $\widehat{\sigma}_1$ player $p$ can decode the $K$ bit message sent by player $1$. With the same test used to detect the start of communication, if the $i$-th block of $t_{\mathrm{collision-test}}$ pulls of arm $\widehat{\sigma}_1$ has a low reward, player $p$ can conclude the bit value sent by player $1$ is a one. If the average reward is large, player $p$ can conclude the bit value sent by player $1$ is a zero. 

The regret accrued by all players until the successful communication of $\mathcal{C}_{\mathrm{top}}$ to all players can be decomposed in two parts: the Round Robin regret ($\mathrm{RoundRobinRegret}([K])$) and the collision regret ($\mathrm{CollisionRegret}([K])$). 

Recall that $\frac{t_{\mathrm{comm1}}^1}{K} \approx \frac{1}{\max_{i} \Delta_{\sigma_i, \sigma_{i+1}}^2}$ and observe that $\mu_{\sigma_1} - \mu_{\sigma_K} \leq K\max_{i} \Delta_{\sigma_i, \sigma_{i+1}}$. During a full Round Robin cycle over arms $\{1, \cdots, K\}$ regret is only incurred when arms in $\{ \mu_{\sigma_{i}}\}_{i=M+1}^K$ are played. Each of these pulls may incur up to $K\max_{i} \Delta_{\sigma_i, \sigma_{i+1}}$ regret. Since there are $M(K-M)$ of these pulls, the $\mathrm{RoundRobinRegret}([K])$ accrued by the $M$ players during the Round Robin plays\footnote{This is accounting for the regret collected during the time it takes for a single transmission of a partition $(\mathcal{C}_{\mathrm{top}}, \mathcal{C}_{\mathrm{bottom}})$. Since there could be up to $K-1$ such rounds, the algorithm's regret has an extra scaling with $K$ as in Theorem~\ref{theorem::main}.} is of the order at most $\frac{KM(K-M)}{\max_{i} \Delta_{\sigma_i, \sigma_{i+1}}}$. 

Now let's see what happens with the $\mathrm{CollisionRegret}([K])$. Since $t_{\mathrm{collision-test}} \approx \frac{\log(t/\delta)}{\Delta^2_{\mathrm{collision}}} $ and the number of collisions experienced by player $1$ during communication is upper bounded by $KMt_{\mathrm{collision-test}}$ the $\mathrm{CollisionRegret}([K])$ is of upper bounded by $\frac{KM\log(t_{\mathrm{comm1}}^1/\delta)}{\Delta_{\mathrm{collision}}}$ (notice this is of smaller order than $\mathrm{RoundRobinRegret}([K])$ since $\Delta_{\mathrm{collision}} \geq \max_{i} \Delta_{\sigma_i, \sigma_{i+1}}$ ). %

The main challenges in our analysis are the following:

\begin{enumerate}
    \item $t_{\mathrm{collision-test}}$ and times $t^1_{\mathrm{first}}$ and $t_{\mathrm{comm1}}^1$ are random and thus unknown to players $2, \cdots, M$. We need to ensure that times $t^p_{\mathrm{first}}$ occur at around the same time for all $p\in [M]$ in order to ensure players $p \in \{ 2, \cdots, M\}$ start ``listening'' for a potential communication start signal from player $1$ at the right time. We do so by designing a mechanism that ensures $\frac{t_{\mathrm{first}}^p}{g(t_{\mathrm{first}}^p)}$ is upper and lower bounded by a constant multiple of $\frac{1}{\max_{i} \Delta_{\sigma_i, \sigma_{i+1}}^2}$. This is the same mechanism we use in the subroutine dedicated to the estimation of $\Delta_{\mathrm{collision}}$. 
    \item Since communication occurs via collisions, the regret incurred by player $1$ (and any player colliding with it) may be linear in $\Delta_{\mathrm{collision}}$ whenever these happen. We need to ensure the time needed for communicating and therefore the number of collisions involved is small, while still being sufficiently large to convey enough information.  %
    \item As we have mentioned above, the start-communication signal is sent out by player $1$ at a round such that $\left\lfloor \frac{t_{\mathrm{comm1}}^1}{g(t_{\mathrm{comm1}}^1)} \right\rfloor$ is a power of nine. The reasoning behind this is to ensure the listening players $p \in \{ 2, \cdots, M\}$ are able to start listening at a recognizable time index. Since the times $t_{\mathrm{first}}^p$ and $t_{\mathrm{first}}^1$ are not equal, and all players $p \in \{ 2, \cdots, M\}$ only start listening after $t_{\mathrm{first}}^p$ we need to ensure that $t^1_{\mathrm{comm1}} \geq t_{\mathrm{listen}}^p$ for all $p \in \{ 2, \cdots, M\}$. 
    
\end{enumerate}
We address each of these three challenges in the sections below.

\subsection{Player \texorpdfstring{$1$}{1} Communication Protocol} \label{section::communication_protocol}

Let $I_i^p(t, \tilde{C}) = [\widehat{\mu}_i^p(t) - \tilde{C}D(N_i^p(t)) , \widehat{\mu}_i^p(t) +  \tilde{C}D(N_i^p(t)) ]$ be the $\tilde{C}-$blowup confidence interval for player $p$ at round $t$ around the mean reward of arm $i$. If $\tilde{C} \geq 1$, these confidence intervals are satisfied (i.e., $\mu_i \in I_i^p(t, \tilde{C})$) for all $t \in \mathbb{N}, i \in [K], p \in [M]$  whenever $\mathcal{E}$ holds, an event that happens with probability at least $1-\delta$. We now introduce the empirical arm connectivity graph with blowup parameter $\tilde{C}$. 

\begin{definition}[$\tilde{C}$-blowup Arm connectivity graph] Let $\tilde{C}\geq 1$.  For each player we define the (random) $\tilde{C}-$blowup arm connectivity graph as $\mathcal{G}_t^p(\tilde{\mathcal{C}}) = ([K], E_t^p(\tilde{\mathcal{C}}))$ with node set $[K]$ and edge set $E_t^p(\tilde{\mathcal{C}})$ defined as:
 \begin{equation*}
     \{ i,j\} \in E_t^p(\tilde{\mathcal{C}}), \quad \text{if } I_i^p(t, \tilde{C}) \cap I_j^p(t, \tilde{C}) \neq \emptyset.
 \end{equation*}
 \end{definition}

Graph $\mathcal{G}_t^p(\tilde{\mathcal{C}})$ is a collection of connected components. The nodes $i,j \in \mathcal{G}_{t}^p(\tilde{\mathcal{C}})$ represent arms $i,j$ and are connected by an edge in $\mathcal{G}_t^p(\tilde{\mathcal{C}})$ if their $\tilde{\mathcal{C}}$-blowup confidence intervals overlap. If we identify each node $i$ of $\mathcal{G}_t^p(\tilde{\mathcal{C}})$ with the empirical mean $\widehat{\mu}_i^p(t)$, the graph has a natural geometric representation as a collection of intervals in $[0,1]$ with each connected interval in the collection representing a connected component of $\mathcal{G}_t^p(\tilde{\mathcal{C}})$. We say that two connected components of $\mathcal{G}_t^p(\tilde{\mathcal{C}})$ are \emph{adjacent} if they are consecutive intervals in this geometric representation.

 Let $\mathbf{conn}^p(t, \tilde{\mathcal{C}})$ be the number of connected components of $\mathcal{G}_t^p(\tilde{\mathcal{C}})$. Let $\{ \mathcal{C}_j^p(t, \tilde{\mathcal{C}}) \}_{j=1}^{\mathbf{conn}^p(t, \tilde{\mathcal{C}})}$ be the collection of connected components at time $t$, ordered by adjacency in the geometric representation of $\mathcal{G}_t^p(\tilde{\mathcal{C}})$, with the empirical mean values of $\mathcal{C}_1^p(t, \tilde{\mathcal{C}})$ being the connected component with the largest empirical mean values among all connected components in $\{ \mathcal{C}_j^p(t, \tilde{\mathcal{C}}) \}_{j=1}^{\mathbf{conn}^p(t, \tilde{\mathcal{C}})}$. This is the same as $\mathcal{C}_{\mathrm{top}}$ in the previous discussion. Let's assume all players are playing using a Round Robin Schedule. Denote by $t_{\mathrm{first}}^p$ to the first  \emph{special} round of player $p$ when $\mathbf{conn}^p(t_{\mathrm{first}}^p, \tilde{\mathcal{C}}) > 1$. We start by showing that if we set  $\tilde{C} = 10$, with high probability the condition $\mathbf{conn}^p\left(sK, \tilde{C}\right) \geq 2$ is triggered for all players $p \in [M]$ at a ``special'' Round Robin round $t_{\mathrm{first}}^p$ (multiple of $K$) such that, 

\begin{equation}\label{equation::sandwich_condition_1}
\frac{128}{\max_{i} \Delta_{\sigma_i, \sigma_{i+1}}^2}    \leq \frac{N(t_{\mathrm{first}}^p)}{g(N(t_{\mathrm{first}}^p))} < \frac{1152 }{\max_{i} \Delta_{\sigma_i, \sigma_{i+1}}^2},
\end{equation}
where $N(t_{\mathrm{first}}^p) = N_i(t_{\mathrm{first}}^p ) $ for all $i \in [K]$ and therefore equal to $\frac{t_{\mathrm{first}}^p}{K}$ (since $t_{\mathrm{first}}^p$ is a special round it is a multiple of $K$). To simplify matters we will use the notation $s_{\mathrm{first}}^p$ to denote the ratios $\frac{t_{\mathrm{first}}^p}{K}$. We will use the same notational convention for all ``named'' rounds with subscripts such as $\mathrm{first}, \mathrm{comm}, \mathrm{comm1},$etc.. Let's start by showing that Equation~\ref{equation::sandwich_condition_1} is a direct consequence of the following Lemma,  setting $C = 10$. Recall that $D(n) = \sqrt{ \frac{2g(n)}{n}}$.

\begin{restatable}[Confidence Bands]{lemma}{fundamentallemmaupperlowerboundmain}\label{lemma::first_fundamental_lemma_main}
Let $\widehat{\mu}_{\sigma_i}(t)$ and $\widehat{\mu}_{\sigma_j}(t)$ be empirical estimators $\mu_{\sigma_i}$ and $\mu_{\sigma_j}$, each using $N(t)$ samples. Let $C > 3$ be a constant. If $t$ is the first special round such that 
\begin{equation}\label{equation::elimination_condition}
    \widehat{\mu}_{\sigma_i}(t) - \widehat{\mu}_{\sigma_j}(t) \geq CD(N(t)),
\end{equation}
then, whenever $\mathcal{E}$ holds, we have $  \frac{ \Delta_{\sigma_i, \sigma_j} }{2(C+2)} <  D(N(t)) \leq \frac{ \Delta_{\sigma_i, \sigma_j} }{C-2}$ and
\begin{equation}\label{equation::elimination_consecuence}
 \frac{ 2(C-2)^2}{\Delta_{\sigma_i, \sigma_j}^2} \leq    \frac{ N(t) }{g(N(t))} <  \frac{ 8(C+2)^2}{\Delta_{\sigma_i, \sigma_j}^2}.
\end{equation}
\end{restatable}

The proof of Lemma~\ref{lemma::first_fundamental_lemma_main} is in Appendix~\ref{section::supporting_t0_lemmas}. Equation~\ref{equation::sandwich_condition_1} can be derived by simply plugging in $C = 10$. With the objective of ensuring all the times $t_{\mathrm{first}}^p$ occur around the same time, let's now show that a simple function of $t_{\mathrm{first}}^p$ is always around the vicinity of a power of $9$. A simple number-theoretic implication of Equation~\ref{equation::sandwich_condition_1} is there exists a unique power of nine in the interval $\left[\frac{128}{\max_{i} \Delta_{\sigma_i, \sigma_{i+1}}^2}   , \frac{1152}{\max_{i} \Delta_{\sigma_i, \sigma_{i+1}}^2}    \right) $ (see Lemma~\ref{lemma::number_theoretic_power_of_nine} in Appendix~\ref{section::appendix_detailed_player_1_communication}). %

For all $p \in [M]$ consider the first \emph{special} round $t$ immediately following $t_{\mathrm{first}}^p$ such that $\left\lfloor \frac{t/K}{g(t/K)}\right\rfloor$ is a power of nine. Call this round $t_{\mathrm{comm}}^p$. Let $9^u$ be the unique (and player independent) power of nine in $\left[\frac{128}{\max_{i} \Delta_{\sigma_i, \sigma_{i+1}}^2}   , \frac{1152}{\max_{i} \Delta_{\sigma_i, \sigma_{i+1}}^2}    \right) $. By definition $\left\lfloor \frac{t_{\mathrm{comm}}^p/K}{g(t_{\mathrm{comm}}^p/K)}\right\rfloor \in \{9^u, 9^{u+1} \}$ for all $p \in [M]$ and thus for all players $t_{\mathrm{comm}}^p \in \left\{ \min_{t \in \mathbb{N}} \text{ s.t. } \left\lfloor \frac{t/K}{g(t/K)}\right\rfloor = 9^u,  \min_{t \in \mathbb{N}} \text{ s.t. } \left\lfloor \frac{t/K}{g(t/K)}\right\rfloor = 9^{u+1} \right\}$ is one of two values  provided $t$ is large enough so that $\frac{t/K}{g(t/K)}$ is an increasing function of $t$. %

Instead of initiating communication at round $t_{\mathrm{comm}}^1$, player $1$ will wait until $t_{\mathrm{comm1}}^1$ so that $t_{\mathrm{comm1}}^1 \in\left\{ \min_{t \in \mathbb{N}} \text{ s.t. } \left\lfloor \frac{t/K}{g(t/K)}\right\rfloor = 9^{u+1},  \min_{t \in \mathbb{N}} \text{ s.t. } \left\lfloor \frac{t/K}{g(t/K)}\right\rfloor = 9^{u+2} \right\}$. This is to ensure that no information-receiving players $p \in \{2, \cdots, M\}$ (all of which will start start listening for a communication signal either at round $ \min_{t \in \mathbb{N}} \text{ s.t. } \left\lfloor \frac{t/K}{g(t/K)}\right\rfloor = 9^{u}$ or $ \min_{t \in \mathbb{N}} \text{ s.t. } \left\lfloor \frac{t/K}{g(t/K)}\right\rfloor = 9^{u+1}$) will miss the start of player $1$'s message. The precise description of how player $1$ waits for $t_{\mathrm{first}}^1$ and $t_{\mathrm{comm1}}^1$ (see Algorithm~\ref{algorithm::prepare_start_player1_communicate}) and communicate (see Algorithm~\ref{algorithm::player_1_communicate}) can be found in Appendix~\ref{section::appendix_detailed_player_1_communication}.%

\paragraph{The $\mathrm{ENCODE}$ function.} As we explained at the beginning of Section~\ref{section::algorithm}, communicating the composition of the connected component $\mathcal{C}^1_1(t_{\mathrm{first}}^1, 10)$ can be done by sending a bit string of size $K$, where the bits corresponding to the arms in $\mathcal{C}^1_1(t_{\mathrm{first}}^1, 10)$ equal one and the remaining bits equal zero. We will name this bitstring as $\mathrm{ENCODE}(\mathcal{C}_1^1(t_{\mathrm{first}}^1, 10))$.

In the following section we discuss  how the communication protocol makes use of collisions to enable player $1$ to transmit $\mathrm{ENCODE}(\mathcal{C}_1^1(t_{\mathrm{first}}^1, 10))$, and how it is that it is possible to do so while incurring a manageable regret.

\begin{wrapfigure}{l}{0.5\textwidth} 
        \centering{\includegraphics[width=0.9\linewidth]{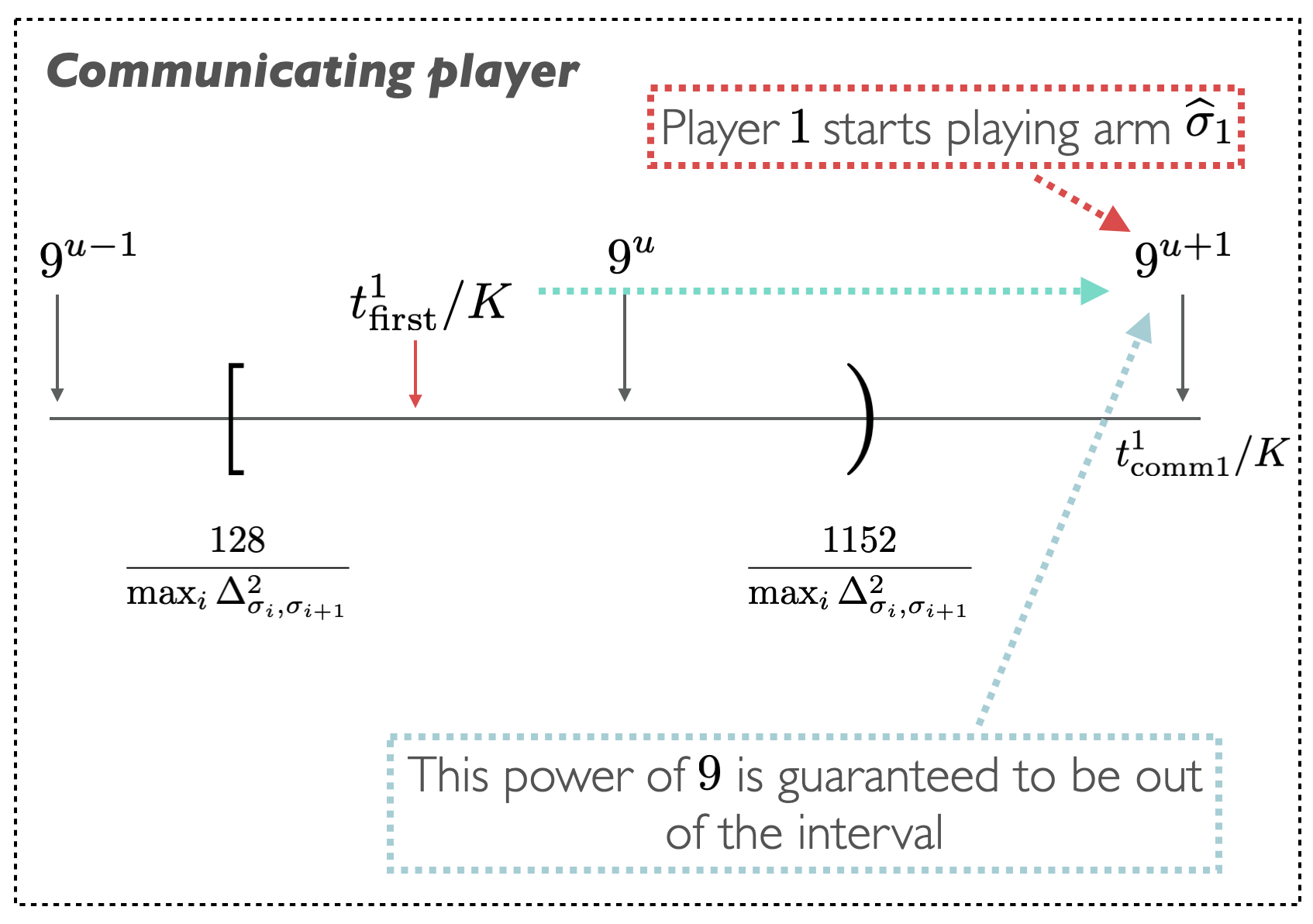}}
        \vspace{-5mm}
        \caption{Illustration of the Communicating Player Strategy (up to $\log$ factors).}
        \label{fig::CommunicatingPlayer}
        \vspace{-3mm}
\end{wrapfigure}

\subsection{Communication Analysis}\label{section::communication_analysis}

The main idea behind our communication algorithm is to make use of the small reward incurred by colliding players to enable communication between players. To gain some intuition for how our communication mechanism works, let's consider a simplified objective. Let's assume player $1$ will be tasked with sending a bit $b\in \{0,1\}$, and every other player $p \in \{ 2, \cdots, M\}$ will be tasked with figuring out the value of $b$. Just as in the preceding discussion imagine that initially the players  pull arms $[K]$ following a Round Robin schedule, and that player $1$ starts to communicate a bit at time $t_{\mathrm{start}}^1$. Let's also assume in this discussion\footnote{In the discussion that follows we will instantiate $t_{\mathrm{start}}^1$ to different rounds including the rounds where each of the players $p \in \{ 2, \cdots, M\}$ use to listen for a ``start of communication'' signal from player $1$ (denoted as $\{t_{\mathrm{listen}}^p\}$) as well as the different message boundary times of $t_{\mathrm{comm1}}^1 + (j-1)Kt_{\mathrm{collision-test}}+1$ for all $j \in \{2, \cdots, |\alpha| +1\}$ whenever player $1$ is sending a message of size $\alpha$. } that all players $p$ have knowledge of $t_{\mathrm{start}}^1$. In Algorithm~\ref{algorithm::prepare_start_player1_communicate} time $t_{\mathrm{start}}^1$ equals $t_{\mathrm{comm1}}^1$. In the following section we will show that for all $p \in [M]$ with high probability $t_{\mathrm{first}}^p \leq t_{\mathrm{comm1}}^1$ therefore in what follows of this section we will assume $t_{\mathrm{first}}^p \leq t_{\mathrm{start}}^1$. 

Let $\widehat{\sigma}_1$ be player $1$'s guess for the optimal arm at time $t_{\mathrm{first}}^1$, i.e., the arm with the largest empirical mean. Let's consider the following communication algorithm. If $b = 1$, starting from round $t_{\mathrm{start}}^1 + 1$ and until round $t_{\mathrm{start}}^1 + Kt_{\mathrm{collision-test}}$, player $1$ will pull arm $\widehat{\sigma}_1$. If instead $b= 0$, player $1$ will continue playing in a Round Robin fashion from round $t_{\mathrm{start}}^1 + 1$ until round $t_{\mathrm{start}}^1 + Kt_{\mathrm{collision-test}}$, thus avoiding collisions. If $b=1$ the average expected reward collected by players $p \in \{ 2, \cdots, M\}$ is $\mu_{\mathrm{collision}}$ when pulling arm $\widehat{\sigma}_1$ during rounds $t_{\mathrm{start}}^1 + 1$ to $t_{\mathrm{start}}^1 + Kt_{\mathrm{collision-test}}$. If $b = 0$ the average reward of pulling arm $\widehat{\sigma}_1$ during rounds  $t_{\mathrm{start}}^1 + 1$ to $t_{\mathrm{start}}^1 + Kt_{\mathrm{collision-test}}$ is $\mu_{\widehat{\sigma}_1}$. Since $t_{\mathrm{start}}^1 \geq t_{\mathrm{first}}^1$ and $t_\mathrm{first}^1$ satisfies the boundary conditions, Equation~\ref{equation::sandwich_condition_1} implies that with high probability,
\begin{equation}\label{equation::condition_t_0_description_main}
  \frac{ N_{\sigma_1}^p(t_{\mathrm{start}}^1) }{g(N_{\sigma_1}^p(t_{\mathrm{start}}^1))} \geq \frac{128}{\max_{i} \Delta_{\sigma_i, \sigma_{i+1}}^2}  \geq \frac{128}{\mu_{\sigma_1}^2}.
  \end{equation}
 Since for all special rounds $t$ the number of pulls is $N^p_i(t) \approx t/K$, Equation~\ref{equation::condition_t_0_description_main} implies that with high probability at time $t_{\mathrm{start}}^1$ player $1$ has pulled each arm $\Omega\left(\frac{1}{\mu^2_{\sigma_1}}\right)$ times up to logarithmic factors. We will now show that the following three statements hold with high probability, 
 
\begin{enumerate}
    \item \textbf{Arm $\widehat{\sigma}_1$ has a large empirical mean for all players. } 
    \begin{itemize}
        \item Arm $\widehat{\sigma}_1 \in \{ i \in [K] \text{ s.t. }\widehat{\mu}_i^p(t_{\mathrm{start}}^1) - \widehat{\mu}_{\mathrm{collision}}^p \geq \frac{1}{2} \max_{j \in [K]} \left(\widehat{\mu}_j^p(t_{\mathrm{start}}^1) - \widehat{\mu}_{\mathrm{collision}}^p  \right)\}$ for all $p \in \{2, \cdots, M\}$ and $\widehat{\mu}^p_{\widehat{\sigma}_1}(t_{\mathrm{start}}^1) -D(N_{\widehat{\sigma}_1}(t_{\mathrm{start}}^1)) \geq \frac{\mu_{\sigma_1}- \mu_{\mathrm{collision}}}{2} + \mu_{\mathrm{collision}}$. 
    \end{itemize} 
    \item \textbf{Arm $\widehat{\sigma}_1$ is comparable to $\sigma_1$. } 
    \begin{itemize}
        \item The witnesses $L_{\widehat{\sigma}_1}^p(t_{\mathrm{start}}^1) $ satisfy, $L_{\widehat{\sigma}_1}^p(t_{\mathrm{start}}^1) \in \left[\frac{3\left(\mu_{\widehat{\sigma}_1} - \mu_{\mathrm{collision}}\right)}{7}, \frac{4\left(\mu_{\widehat{\sigma}_1} - \mu_{\mathrm{collision}}\right)}{7}\right] + \mu_{\mathrm{collision}} $ for all $p \in \{ 2, \cdots, M\}$.
    \end{itemize}
    \item \textbf{When collisions are avoided the mean estimators $\widehat{\mu}^p_{\widehat{\sigma}_1}(t_{\mathrm{start}}^1 +1:  + Kf(\widehat{\Delta}_{\mathrm{collision}}, t_{\mathrm{start}}^1))$ are far from $\mu_{\mathrm{collision}}$. } 
    \begin{itemize}
        \item If $b = 1$, the estimators $\widehat{\mu}^p_{\widehat{\sigma}_1}(t_{\mathrm{start}}^1 +1:  + Kt_{\mathrm{collision-test}})\leq L_{\widehat{\sigma}_1}^p$ for all $p \in [M]$.
        \item If $b = 0$, the estimators $\widehat{\mu}^p_{\widehat{\sigma}_1}(t_{\mathrm{start}}^1:t_{\mathrm{start}}^1 + Kt_{\mathrm{collision-test}}) > L_{\widehat{\sigma}_1}^p$ for all $p \in [M]$.
    \end{itemize}
\end{enumerate}

If players $p \in \{ 2, \cdots, M\}$ have knowledge of $t_{\mathrm{start}}^1$, all they need to do to decode $b$ is to test for all arms with a large empirical mean (as computed by these players up to time $t_{\mathrm{start}}^1$), and compare these values with the empirical means computed during rounds $t_{\mathrm{start}}^1+1$ through $t_{\mathrm{start}}^1 + Kt_{\mathrm{collision-test}} $, when potential collisions may have taken place. If the two estimators are vastly different, players $p \in \{ 2, \cdots, M\}$ can conclude that $b$ equals one. If instead these values are ``similar,'' they can conclude that  $b$ equals zero. We formalize this idea via the following $\mathrm{CollisionTest}$ algorithm. We use the notation $t_{\mathrm{test}}^p$ to denote the rounds when each of the players $p \in \{2,\cdots, M\}$ starts probing for a communication signal from player $1$. If $t_{\mathrm{start}}^1$ is common knowledge, the $\mathrm{CollisionTest}$ algorithm can be instantiated by setting $t_\mathrm{test}^p = t_{\mathrm{start}}^1$. 

\begin{algorithm}[H]
\textbf{Input} $t_{\mathrm{test}}^p$, witnesses $\{ L_i^p(t_{\mathrm{test}}^p)\}_{i \in [K]}$, empirical means $\{ \widehat{\mu}_i^p(t_{\mathrm{test}}^p)\}_{i \in [K]}$, communication round empirical means $\{ \widehat{\mu}_i^p(t_{\mathrm{test1}}^p+1:t_{\mathrm{test1}}^p+Kt_{\mathrm{collision-test}}\}_{i \in [K]}$\\
$\widehat{\mathbf{MaxArms}} \leftarrow \left\{ i \in [K] \text{ s.t. }\widehat{\mu}_i^p(t_{\mathrm{test}}^p) - \widehat{\mu}_{\mathrm{collision}}^p\geq \frac{1}{2} \max_{j \in [K]} \left(\widehat{\mu}_j^p(t_{\mathrm{test}}^p) - \widehat{\mu}_{\mathrm{collision}}^p\right)\right\}$ \\
\If{$\exists i \in \widehat{\mathbf{MaxArms}}$ s.t. $\widehat{\mu}_i^p(t_{\mathrm{test1}}^p+1: t_{\mathrm{test1}}^p + Kt_{\mathrm{collision-test}}< L_i^p(t_{\mathrm{test}}^p) $ }{
\textbf{Return} 1 
}%
\textbf{Return} 0 
    
\caption{$\mathrm{CollisionTest}$ (player $p\neq 1$)}
\label{algorithm::zero_test_intro}
\end{algorithm} 
 
Notice that in contrast with the discussion that preceded it $\mathrm{CollisionTest}$ allows the indices $t_{\mathrm{test}}^p$ and $t_{\mathrm{test1}}^p$ to be different but the length of the communication rounds remains $Kt_{\mathrm{collision-test}}$. Our first result is to show this procedure allows players $p \in \{2, \cdots, M\}$ to recover $b$ with high probability.

\begin{restatable}[One Bit Recovery]{lemma}{exactrecoverycollisiontestonebit}\label{lemma::collision_test_works_main}
Let $A^p = \{ L_i^p(t_{\mathrm{start}}^1)\}_{i \in [K]}$, $B^p = \{ \widehat{\mu}_i^p(t_{\mathrm{start}}^1)\}_{i \in [K]}$ , $C^p = \{ \widehat{\mu}_i^p(t_{\mathrm{start}}^1+1:t_{\mathrm{start}}^1+K t_{\mathrm{collision-test}} )\}_{i \in [K]}$ for all $p$. If the good event $\mathcal{E}$ holds then, with probability at least $1-\frac{\delta}{4K^2}$, all players $p \in \{ 2, \cdots, M\}$ will be able to recover exactly the bit transmitted by player $1$ by calling $\mathrm{CollisionTest}( A^p, B^p, C^p )$.
\end{restatable}

The complete proof of Lemma~\ref{lemma::zero_test_works_main} and an in-depth discussion on the $\mathrm{CollisionTest}$ concept can be found in Appendices~\ref{section::communication_analysis_discussion_supporting_results} and~\ref{section::unknown_collision_reward_appendix}. The logic behind why the $\mathrm{CollisionTest}$ works is as follows. At time $t_{\mathrm{test}}^p$ (in our case equal to $t_{\mathrm{comm1}}^1$ all players have access to a constant accuracy estimator for  $\max_i \Delta_{\sigma_i, \sigma_{i+1}}$ (i.e. the empirical gap between the arms at the boundary of the two connected components of $\mathcal{C}_1^1(t_{\mathrm{first}}^1, 10)$). Therefore they can estimate $\mu_{\sigma_1}$ up to a $\Delta_{\mathrm{collision}}$-accuracy. By Hoeffding inequality testing at an accuracy of $c\Delta_{\mathrm{collision}} $ the mean of an arm $\widehat{\sigma}_1$ satisfying $\mu_{\widehat{\sigma}_1} > \mu_{\mathrm{collision}} + c'( \mu_{\sigma_1}-\mu_{\mathrm{collision}})$ with $c' > c$ only requires $\widetilde{\mathcal{O}}\left( \frac{1}{\Delta^2_{\mathrm{collision}}}\right)$ samples (up to log factors and each with regret at most $\Delta_{\mathrm{collision}}$), thus incurring regret of only $\mathcal{O}(1/\Delta_{\mathrm{collision}}) $ (up to log factors).

\subsection{The Listening Players}\label{section::listening_players}

We have all the necessary pieces in place to spell out the details of the listening protocol used by players $\{2, \cdots, M\}$. The listening players shall wait until the first special round such that $\mathbf{conn}^p\left(sK, 10\right) \geq 2$. After this round has passed the players will start listening for an incoming signal from player $1$ during subsequent special rounds  $t = Ks$ such that $\left\lfloor \frac{s}{g(s)} \right\rfloor$ is a power of nine. The precise description of how to initialize the listening protocol (see Algorithm~\ref{algorithm::prepare_and_start_listening}) can be found in Appendix~\ref{section::listening_players_discussion_supporting_results}. If the player detects a signal, she will start listening for a size $K$ bit string using the $\mathrm{DECODE}$ function (see Algorithm~\ref{algorithm::decode} in Appendix~\ref{section::listening_players_discussion_supporting_results} for a detailed explanation). $\mathrm{DECODE}$ consists of $K$ consecutive $\mathrm{CollisionTest}$ calls. 

Combining these communication and listening protocols for player $1$ and players $\{2, \cdots, M\}$ respectively we can guarantee that with high probability the value of $t_{\mathrm{listen}}^p$ passed down to the $\mathrm{DECODE}$ function satisfies $t_{\mathrm{listen}}^p = t_{\mathrm{comm1}}^1$ and that the listening players $\{2, \cdots, M\}$ will be able to decode the message handed down by player $1$. 

\begin{restatable}[Message Recovery]{lemma}{multibitrecovery}\label{lemma::listen_agrees_comm}
If $\mathcal{E}$ holds then with probability at least $1-\frac{\delta}{K}$ for all $p \in [M]$ the value of $t_{\mathrm{listen}}^p$ sent to the $\mathrm{DECODE}$ function of Algorithm~\ref{algorithm::decode} satisfies $t_{\mathrm{listen}}^p = t_{\mathrm{comm1}}^1$ and $\mathrm{DECODE}$ will recover the exact $K-$bit message sent by player $1$.
\end{restatable}
The proof of Lemma~\ref{lemma::listen_agrees_comm} is in Appendix~\ref{section::listening_players_discussion_supporting_results}. The discussion is complemented by Appendix~\ref{section::agreeing_t_collision_test}.

\begin{figure}%
        \centering{\includegraphics[width=0.45\linewidth]{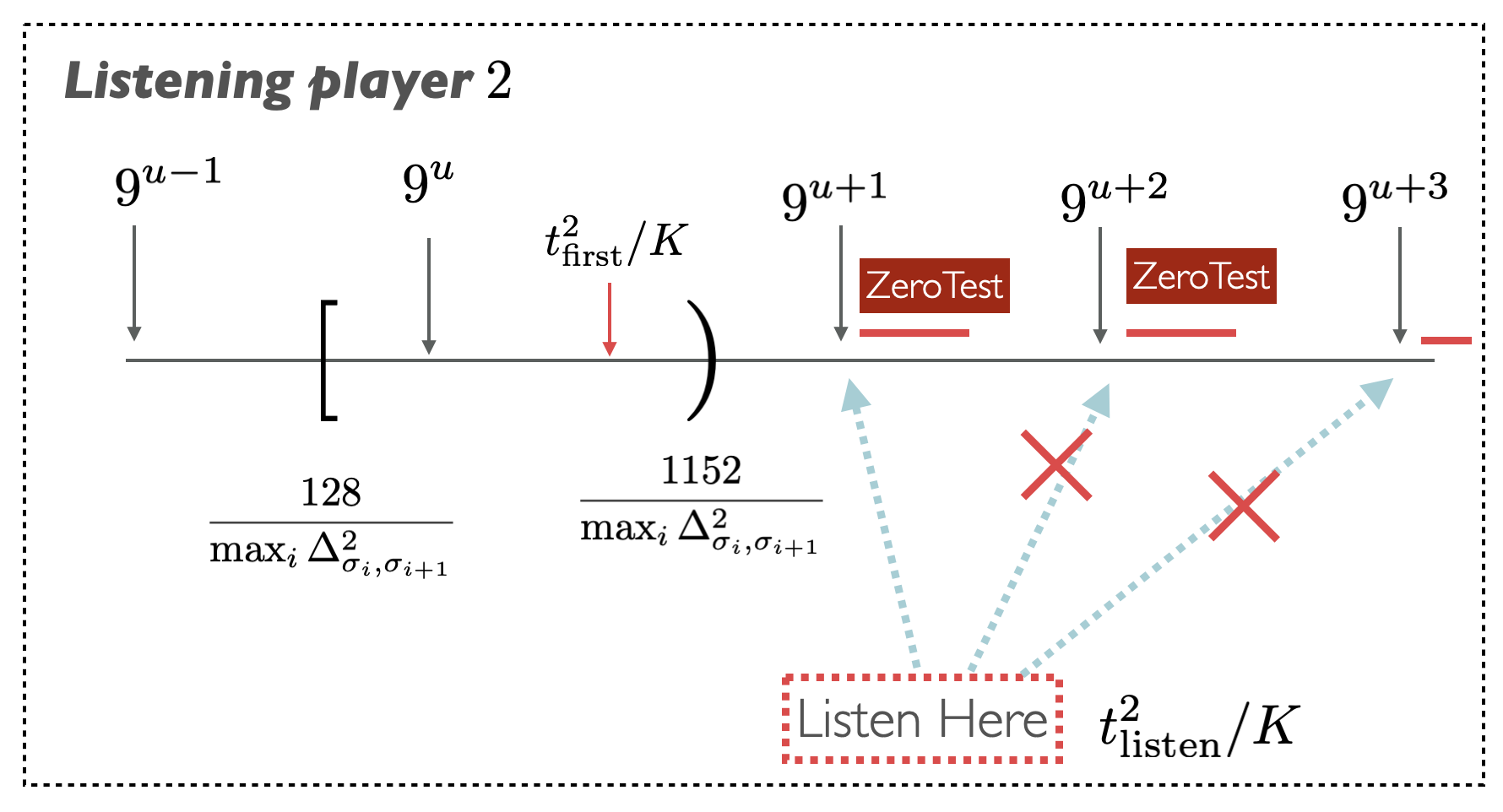}}
        \centering{\includegraphics[width=0.45\linewidth]{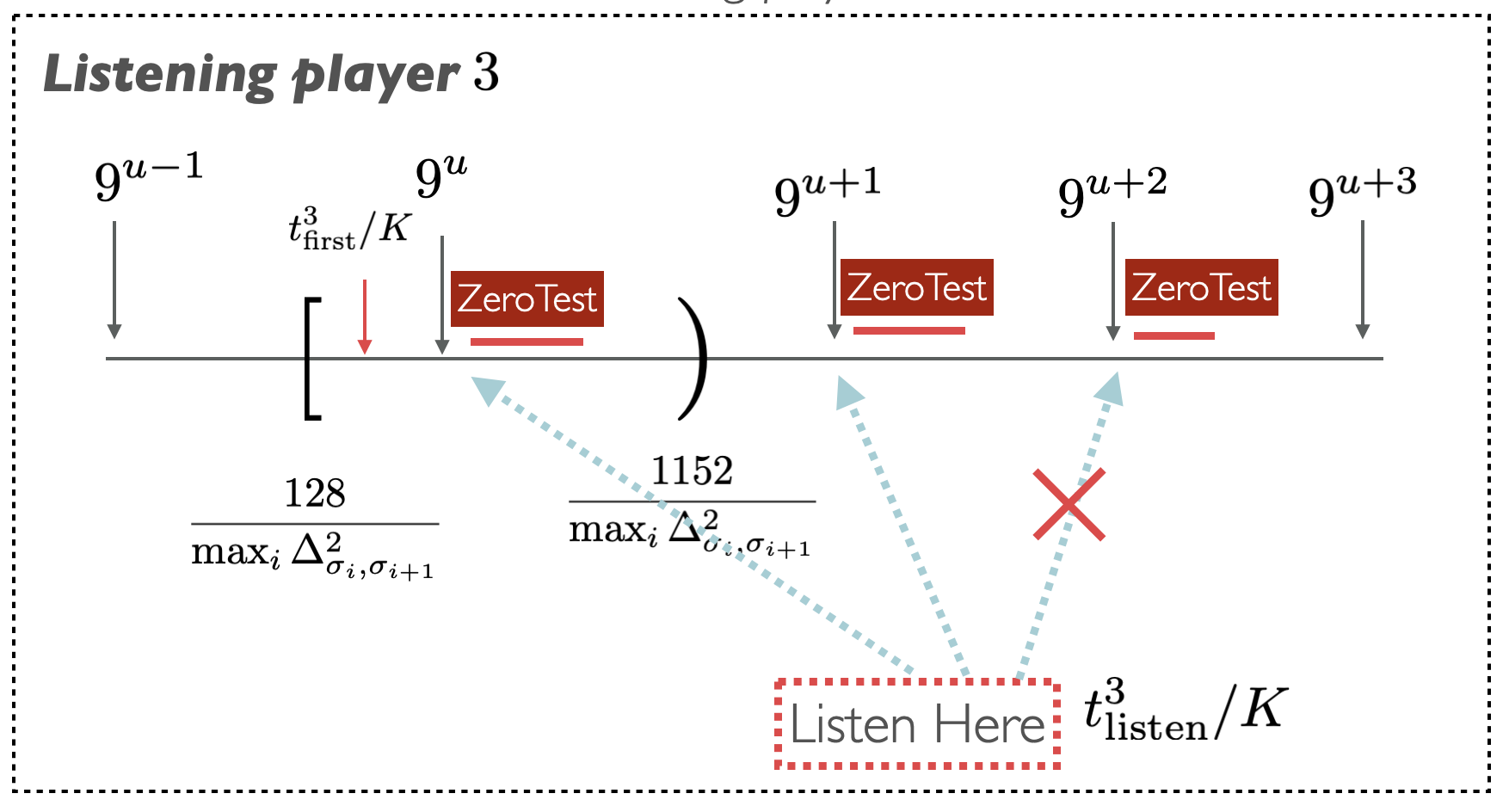}}
        \caption{Illustration of the Listening Player Strategy (up to $\log$ factors) depending on what side $\frac{t_{\mathrm{first}}^p/K}{g(t_{\mathrm{first}}^p/K)}$ lies with respect to the power of nine in \begin{small}$\left[\frac{128}{\max_{i} \Delta_{\sigma_i, \sigma_{i+1}}^2}   , \frac{1152}{\max_{i} \Delta_{\sigma_i, \sigma_{i+1}}^2}    \right) $\end{small} .}
        \label{fig::ListeningPlayers}
        \vspace{-10mm}
\end{figure}

\subsection{Bounding the Regret}\label{section::bounding_regret_main}

To finalize our regret analysis we first bound the regret incurred by the players during the first communication round and player $1$ has communicated to all others the composition of $\mathcal{C}_1^1(t_{\mathrm{first}}^1, 10)$. We denote this quantity as $\mathrm{FirstPartitionRegret}([K], [M])$. We proceed to bound this quantity by splitting regret in two, the $\mathrm{CollisionRegret}([K], [M])$ incurred by the players during communication induced by player $1$ pulling arm $\widehat{\sigma}_1$ and the $\mathrm{RoundRobinRegret}([K], [M])$ incurred by the players during the steps these followed a Round Robin schedule.

The length of each bit communication round is of order $t_{\mathrm{collision-test}} \approx \frac{\log\left(1/{\delta}\right)}{\Delta^2_{\mathrm{collision}}}$. During this communication round the number of collisions between player $1$ and any other player is at most $(K+1)\times t_{\mathrm{collision-test}} $ because the total number of bits required to communicate $\mathrm{ENCODE}(\mathcal{C}_1^1(t_{\mathrm{first}}^1, 10))$ equals $K+1$, one ON bit to signal the start of communication and $K$ to transmit $\mathcal{C}_1^1(t_{\mathrm{first}}^1, 10)$. Since each collision incurs in at most regret $\Delta_{\mathrm{collision}}$, $\mathrm{CollisionRegret}([K], [M]) \leq \widetilde{\mathcal{O}}\left( \frac{KM \log\left( {t_{\mathrm{comm1}}^1}/{\delta}\right)}{\Delta_{\mathrm{collision}} }\right)$. The precise statement of this bound is in Corollary~\ref{corollary::collision_regret} in Section~\ref{section::bounding_first_partition_regret}.  

To bound the $\mathrm{RoundRobinRegret}([K], [M])$ note the regret during a single Round Robin round (i.e. $K$ steps of all $M$ players) is upper bounded by $M(K-M)K\max_i\Delta_{\sigma_i, \sigma_{i+1}}$. The total duration of the whole communication protocol for $\mathcal{C}_1^1(t_{\mathrm{first}}^1, 10)$ is of the order of $\frac{\log\left( {t_{\mathrm{first}}^1}/{\delta}\right)}{\max_i \Delta^2_{\sigma_i, \sigma_{i+1}}}$ and therefore $\mathrm{RoundRobinRegret}([K], [M])\leq \widetilde{\mathcal{O}}\left(  \frac{M(K-M)K \log\left( t_{\mathrm{first}}^1/\delta\right)}{\max_i \Delta_{\sigma_i, \sigma_{i+1}}} \right) $.  The precise statement of this bound can be found in Corollary~\ref{corollary::roundrobin_regret} in Section~\ref{section::bounding_first_partition_regret}. Combining these results and using $\Delta_{\mathrm{collision}} \geq \max_{i} \Delta_{\sigma_i, \sigma_{i+1}}$ we conclude the regret to communicate $\mathcal{C}_1^1(t_{\mathrm{first}}^1, 10)$ is $\widetilde{\mathcal{O}}\left(  \frac{M(K-M)K \log\left( t_{\mathrm{first}}^1/\delta\right)}{\max_i \Delta_{\sigma_i, \sigma_{i+1}}} \right)$.

We can now put together the whole algorithm and prove Theorem~\ref{theorem::main}. As we explained at the beginning of Section~\ref{section::algorithm},  after communication of $\mathcal{C}_1^1(t_{\mathrm{first}}^1, 10)$ and all players have learned what arms belong to $\mathcal{C}_1^1(t_{\mathrm{first}}^1, 10)$, they will split themselves into one or two groups. If $| \mathcal{C}_1^1(t_{\mathrm{first}}^1, 10)| \geq M$, from then on all players will pull arms from $\mathcal{C}_1^1(t_{\mathrm{first}}^1, 10)$ exclusively. If $|\mathcal{C}_1^1(t_{\mathrm{first}}^1, 10)| < M$, then players $1$ to $|\mathcal{C}_1^1(t_{\mathrm{first}}^1, 10)|$ will pull arms from $\mathcal{C}_1^1(t_{\mathrm{first}}^1, 10)$ following a Round Robin schedule thereafter (these players have identified a set of arms to exploit) while the remaining players will play arms from $[K] \backslash \mathcal{C}_1^1(t_{\mathrm{first}}^1, 10)$ following the communication protocol we have outlined in the previous section but now restricted to a set of arms of size smaller than $M$. In both cases, the problem has been reduced to a smaller Cooperative Multi-Player Multi-Armed problem. The players then re-index the arm labels\footnote{For example the players assigned to $\mathcal{C}_1$ will re-index these arms so they are labeled $1$ to $|\mathcal{C}_1|$ by switching the smallest label in $\mathcal{C}_1$ to a $1$, the second smallest to a $2$, and so on.} and start playing following a Round Robin schedule within their assigned group and the smallest indexed player in any group will become the communicating player.  %
 We call the function that splits and iterates over smaller Cooperative Multi Armed Multi Player problems the $\mathrm{RECURSE}$ function.  It is defined formally in Algorithm~\ref{algorithm::recurse} in Appendix~\ref{section::analyzing_RECURSE_function}.

The $\mathrm{RECURSE}$ function will converge to a steady state where the players are only pulling the top $M$ arms as soon as $\Delta_{\sigma_M, \sigma_{M+1}}$ is the basis of the communicated arm partition. 

It can be shown the gaps recursed over are always in (roughly) decreasing order. Thus, each call to $\mathrm{RECURSE}$ will incur regret upper bounded by $\widetilde{\mathcal{O}}\left(  \frac{M(K-M)K \log\left( t_{\mathrm{first}}^1/\delta\right)}{\max_i \Delta_{\sigma_M, \sigma_{M+1}}} \right)$. Finally, since there can be at most $K-1$ calls to the $\mathrm{RECURSE}$ function, setting $\delta = \frac{1}{T}$ we conclude,
\begin{equation*}
    \mathcal{R}_T \leq   \widetilde{ \mathcal{O}}\left(\frac{M(K-M)K^2\log(T)}{\Delta_{\sigma_M, \sigma_{M+1}}} \right) ,
\end{equation*}
with probability at least $1-T$. This finalizes the proof of Theorem~\ref{theorem::main}. The detailed explanation of each of these steps can be found in Appendices~\ref{section::bounding_regret_discussion_supporting_results} and~\ref{section::agreeing_t_collision_test}. The discussion on how to compute and agree on $t_{\mathrm{collison-test}}$ for all players can also be found in Appendix~\ref{section::agreeing_t_collision_test} and makes use of the same synchronization and communication ideas we have explained here. The regret incurred during that round is of order $\mathcal{O}\left(\frac{M(K-M)K^2\log(T)}{\Delta_{\mathrm{collision}}}\right)$.

\section{Conclusion}

We have proposed a series of algorithms for the Multi-Player Multi-Armed bandit problem. We achieve a regret guarantee logarithmic in the number of rounds and inversely proportional to the sub-optimality gap. In contrast with previous work, we make no assumptions regarding the player's knowledge about the nature of the reward vector (such as assuming a known lower bound for the minimum reward value) and even the collision reward. This paper finally solves the no-sensing multi-player multi-armed bandit problem in its entirety when collisions are allowed. We believe the techniques we have introduced in this work (including for example the Bernstein $\mathrm{ZeroTest}$ for the zero collision reward setting discussed in Appendix~\ref{section::detailed_discussion_appendix}) can be used to prove sharp instance dependent guarantees in many decentralized bandit problems such as decentralized matching markets (see~\citet{liu2021bandit}). We hope future research is also spent on simplifying our algorithmic implementations.

\bibliography{ref}

\appendix
\newpage

\tableofcontents
\addtocontents{toc}{\protect\setcounter{tocdepth}{2}}
\clearpage

\section{Guide to Appendix}

The Appendix will be split in different sections. In Section~\ref{section::detailed_discussion_appendix} we include a much more detailed discussion of the different algorithmic components sketched out in the main paper with an emphasis in the zero collision reward case. In Section~\ref{section::extensions} we extend our results to a variety of settings more general than those presented in the main. The remaining Appendix sections contain missing proofs and supporting technical results.

\section{Detailed Discussion of the zero collision reward setting }\label{section::detailed_discussion_appendix}

In this section we will flesh out the details of the different algorithmic components to derive our instance dependent rates for the Multi-Player, Multi-Armed bandit problem with communication through collisions. Throughout the proofs we will make use of the following ``Boundary conditions". In this section the following function will be useful,

\begin{itemize}
  \item \textbf{Zero Collision Reward Length of a communication round.} $f : \mathbb{N} \rightarrow \mathbb{R}_{+}$ such that $f(n)$ is the first integer such that $\frac{f(n)}{B\left(f(n), \frac{\delta}{4K^2M}\right)} \geq 24 \sqrt{\frac{n/K}{2g(n/K)} }$ where $B(n, \delta') = 2\ln \ln(2n) + \ln \frac{5.2}{\delta'}$. %
\end{itemize}

\paragraph{Boundary Conditions} We will consider the following four boundary conditions. 
\begin{enumerate}
    \item  Let $t_{\mathrm{boundary1}} $ be such that $\frac{f(n)-1}{B(f(n)-1, \frac{\delta}{4K^2 M})} \geq 1$ for all $n \geq t_{\mathrm{boundary1}}$.
    \item  Let $s_{\mathrm{boundary2}}$ be such that $\frac{\partial D(s)}{\partial s} \leq 0$ for all $s \geq s_{\mathrm{boundary2}}$ ($s_{\mathrm{boundary2}}$ can also be defined in terms of $g$ as $\frac{\partial g(s)}{\partial s} \geq 0$ for all $s \geq s_{\mathrm{boundary2}}$). 
    \item Assume $\delta \leq \frac{1}{162}$so that $4s_{\mathrm{boundary2}}^2 MKL/\delta \geq 162$.
    \item  Let $t_{\mathrm{boundary3}}$ be such that $f(n) \leq  n/K $ for all $n \geq t_{\mathrm{boundary3}}$.
\end{enumerate}
Let $t_{\mathrm{firstBoundary}}$ be the first special round (i.e. $t_{\mathrm{firstBoundary}}$ is a multiple of $K$) such that $$t_{\mathrm{firstBoundary}}\geq \max\left( t_{\mathrm{boundary1}}, K s_{\mathrm{boundary2}}, t_{\mathrm{boundary3}} \right).$$

The analysis sketch for the regret rates achieved by our algorithm presented in Section~\ref{section::algorithm} is only satisfied for timesteps larger than $t_{\mathrm{firstBoundary}}$, these `boundary conditions' appear in statements such as 

```By definition $\left\lfloor \frac{t_{\mathrm{comm}}^p/K}{g(t_{\mathrm{comm}}^p/K)}\right\rfloor \in \{9^u, 9^{u+1} \}$ for all $p \in [M]$ and thus for all players $t_{\mathrm{comm}}^p \in \left\{ \min_{t \in \mathbb{N}} \text{ s.t. } \left\lfloor \frac{t/K}{g(t/K)}\right\rfloor = 9^u,  \min_{t \in \mathbb{N}} \text{ s.t. } \left\lfloor \frac{t/K}{g(t/K)}\right\rfloor = 9^{u+1} \right\}$ is one of two values  provided $t$ is large enough so that $\frac{t/K}{g(t/K)}$ is an increasing function of $t$'''

from Section~\ref{section::communication_protocol} (this statement here corresponds to $s_{\mathrm{boundary2}}$. Throughout our detailed analysis of the regret rate of our algorithm presented in this section, we kept track of all the emerging boundary conditions and have compiled them in this list. We then show $t_{\mathrm{firstBoundary}}$ is a function of $\log\left(\frac{1}{\delta}\right)$, $K$ and $M$ and that it depends polynomially on $K$ and $M$ and linearly on $\log(\frac{1}{\delta})$. See Section~\ref{section::bounding_regret_discussion_supporting_results} for a formal proof. The test used in the zero collision reward setting that forms the basis of our encoding and decoding strategy will be called the $\mathrm{ZeroTest}$.

   \begin{algorithm}[H]
\textbf{Input} Player $p\neq 1$, witnesses $\{ L_i^p(t_\mathrm{start}^1)\}_{i \in [K], p \in [M]}$\\
    \For{ $i$ such that $\widehat{\mu}_i^p(t_\mathrm{start}^1) \geq \frac{1}{2}\max_{j\in [K] } \widehat{\mu}_j^p(t_\mathrm{start}^1)$}{
    \If{$\widehat{\mu}_i^p(t_\mathrm{start}^1+1: t_\mathrm{start}^1 + Kf(t_\mathrm{start}^1)) < L_i^p(t_\mathrm{start}^1) $ }{\textbf{Return} 1 }}
    \textbf{Return} 0\\
\caption{Zero Test}
\label{algorithm::zero_test_appendix}
\end{algorithm}

\subsection{Analysis Desiderata}

The main challenges in our analysis for the zero collision reward setting are the following:

\begin{enumerate}
    \item Same as in item 1. of Section~\ref{section::algorithm}. 
    \item Since communication occurs via collisions, the regret incurred by player $1$ (and any player colliding with it) may be linear whenever these happen. We need to ensure the time needed for communicating and therefore the number of collisions involved is small, while still being sufficiently large to convey enough information. Aided by Bernstein-style bounds we show the number of collisions needed to transmit information scales linearly with $\frac{1}{\mu_1} \leq \frac{1}{\max_{i} \Delta_{\sigma_i, \sigma_{i+1}}}$ and not quadratically. %
    \item Same as item 3. of Section~\ref{section::algorithm}. 
\end{enumerate}

\subsection{Detailed Discussion and Missing Supporting Results Player $1$ Communication Protocol.}\label{section::appendix_detailed_player_1_communication}

The following supporting Lemma allows us to show there exists a unique power of nine in the interval $\left[\frac{128}{\max_{i} \Delta_{\sigma_i, \sigma_{i+1}}^2}   , \frac{1152}{\max_{i} \Delta_{\sigma_i, \sigma_{i+1}}^2}    \right) $,
\begin{lemma}\label{lemma::number_theoretic_power_of_nine}
Let $x \in \mathbb{R}$ be a positive real number and let $n \in \mathbb{N}$ be a natural number. There exists a unique power of $n$ in the interval $[x, nx)$. 
\end{lemma}
\begin{proof}
Let $n^\alpha$ with $\alpha \in \mathbb{N}$ be the largest power of $n$ such that $n^\alpha < x$. Multiplying both sides of this inequality by $n$ we obtain $n^{\alpha+1} < nx$. Since by assumption $n^\alpha$ is the largest power of $n$ not in $[x, nx)$, it must be the case that $n^{\alpha+1}$ lies in $[x, nx)$. This proves there must be at least one power of $n$ in the interval $[x, nx)$. To show uniqueness again consider $x \leq n^{\alpha+1} < nx$ and multiply all parts of this inequality by $n$. We see that $nx \leq n^{\alpha+2}$, thus showing that $n^{\alpha+2} \not\in [x, nx)$. 
\end{proof}

Algorithm~\ref{algorithm::player_1_communicate} contains the protocol used by player $1$ to broadcast an arbitrary bit message of size $\alpha$. The communicating player starts by sending a ``ping.'' During rounds $t_{\mathrm{comm1}}^1 +1$ to $t_{\mathrm{comm1}}^1 + Kf(t_{\mathrm{comm1}}^1)$ player $1$ will pull arm $\widehat{\sigma}_1$ with the intention of transmitting a single ON bit to signal to the receiving players that an information transmission session has started. After this, the transmission of the bit string takes place. The total number of rounds necessary to transmit a bit string of size $\alpha$ is $(\alpha+1)Kf(t_{\mathrm{comm1}}^1)$.  One bit to signal the start of the communication and $\alpha$ bits corresponding to the message.

\begin{algorithm}[h]
\textbf{Input} Player $1$\\
\textbf{Initialize } $\mathrm{FLAG}\leftarrow \mathrm{NONE}$\\
    \For{ special rounds $s = 1, \cdots $ }{
    Pull arm $K-1$  (player $1$ has just finished a Round Robin round) \\
    \If{ $\mathbf{conn}^1\left(sK, 10\right) \geq 2$ and $\mathrm{FLAG} = \mathrm{NONE}$}{
     $t_{\mathrm{first}}^1 \leftarrow s K$ \\ 

    $\mathrm{FLAG} \leftarrow \mathrm{FINDPOWER}$
   }
   \If{$\mathrm{FLAG}=\mathrm{FINDPOWER}$,  $\left\lfloor\frac{s}{g(s)}\right\rfloor= 9^w$ for some $w \in \mathbb{N}$ and $\left\lfloor \frac{s-1}{g(s-1)} \right\rfloor\neq  9^w$ }{
    $t_{\mathrm{comm}}^1 \leftarrow Ks $ .\\
        $\mathrm{FLAG}\leftarrow \mathrm{PRECOMM}$.\\
        
   }
   \ElseIf{$\mathrm{FLAG}=\mathrm{PRECOMM}$, $\left\lfloor \frac{s}{g(s)} \right\rfloor= 9^w$ for some $w \in \mathbb{N}$ and $\left\lfloor \frac{s-1}{g(s-1)} \right\rfloor\neq 9^w$}{
   $t_{\mathrm{comm1}}^1 \leftarrow Ks $\\
  Compute guess $\widehat{\sigma}_1 \in [K]$ for the maximal arm $\sigma_1$ :
    \begin{equation*}
       \widehat{\sigma}_1 = \argmax_{i\in [K]} \widehat{\mu}^1_{i}(t_{\mathrm{comm1}}^1 ) - D(N^1_{i}(t_{\mathrm{comm1}}^1 ))
    \end{equation*}

   $\mathrm{MESSAGE} \leftarrow \mathrm{ENCODE}(\mathcal{C}_1^1(t_{\mathrm{first}}^1, 10))$ \\
    Run    $\mathrm{COMMUNICATE}(t_{\mathrm{comm1}}^1 , \widehat{\sigma}_1, \mathrm{MESSAGE})$ using Algorithm~\ref{algorithm::player_1_communicate}.\\
   }
    }
\caption{Prepare and Start Communication (Player 1) }
\label{algorithm::prepare_start_player1_communicate}
\end{algorithm}

\paragraph{Properties of $t_{\mathrm{comm1}}^1$.} By design, the $t_{\mathrm{comm1}}^1$ index passed to the $\mathrm{COMMUNICATE}$ function is defined to be the \emph{second} special round equal or larger to $t_{\mathrm{first}}^1$ satisfying $t_{\mathrm{comm1}}^1 = Ks_{\mathrm{comm1}}^1$ where $\left\lfloor \frac{s_{\mathrm{comm1}^1}}{s_{\mathrm{comm1}}^1} \right\rfloor = 9^w$ for some $w \in \mathbb{N}$. Similarly the $t_{\mathrm{comm}}^1$ index is the \emph{first} special round equal or larger to $t_{\mathrm{first}}^1$ satisfying $t_{\mathrm{comm}}^1 = Ks_{\mathrm{comm}}^1$ where $\lfloor \frac{s_{\mathrm{comm1}^1}}{s_{\mathrm{comm1}}^1} \rfloor = 9^w$ for some $w \in \mathbb{N}$. The following lemma addresses the question of how much larger can $t_{\mathrm{comm1}}^1$ be than $t_{\mathrm{first}}^1$.

\begin{restatable}{lemma}{lemmaboundingscommone}\label{lemma::upper_bound_ratio_scomm1}
Let $t_\mathrm{first}^1 = Ks_\mathrm{first}^1$ and $t_\mathrm{comm1}^1 = Ks_\mathrm{comm1}^1$. If $s_{\mathrm{first}}^1 \geq s_{\mathrm{boundary2}}$ and $\delta \leq \frac{1}{162}$ then $s_{\mathrm{comm1}}^1 \leq 162 s_{\mathrm{first}}^1 $ and
\begin{equation*}
    \frac{s_{comm1}^1}{g(s_{comm1}^1)} \leq \frac{162 s_\mathrm{first}^1}{g( s_\mathrm{first}^1)}.
\end{equation*}
\end{restatable}

The proof of Lemma~\ref{lemma::upper_bound_ratio_scomm1} can be found in Appendix~\ref{section::proof_upper_bound_ratio_scomm1}. We can leverage Lemma~\ref{lemma::upper_bound_ratio_scomm1} to derive an explicit bound for $s_{\mathrm{comm1}}^1$ that is satisfied whenever $\mathcal{E}$ holds.

\begin{restatable}{lemma}{lemmaboundingscommonenolog}\label{lemma::upper_bounding_s_comm_1}
If $\mathcal{E}$ holds, $s_{\mathrm{first}}^1 \geq s_{\mathrm{boundary2}}$ and $\delta \leq \frac{1}{162}$  then 
\begin{equation*}
    s_{\mathrm{comm1}}^1 \leq \frac{746496 }{\max_{i} \Delta_{\sigma_i, \sigma_{i+1}}^2} \log\left(  \frac{746496 MK }{\delta\max_{i} \Delta_{\sigma_i, \sigma_{i+1}}^2}  \right)
\end{equation*}
\end{restatable}

The proof of Lemma~\ref{lemma::upper_bounding_s_comm_1} can be found in Appendix~\ref{section::proof_upper_bounding_s_comm_1}.

In Algorithm~\ref{algorithm::prepare_start_player1_communicate} we detail player $1$'s steps to initialize the communication protocol.  

\begin{algorithm}[h]
\textbf{Input} Round number $t_{\mathrm{comm1}}^1$, communicating arm $i_{\mathrm{comm}}$, message $\mathbf{b} \in \{ 0,1\}^{\alpha}$, communication rounds length function $f : \mathbb{R}\rightarrow \mathbb{N}$ \\
\For{$t = t_{\mathrm{comm1}}^1+1, \cdots, t_{\mathrm{comm1}}^1 + (\alpha+1) K f(t_{\mathrm{comm1}}^1 ) $}{
  \If{ $t \in [t_{\mathrm{comm1}}^1 +  1, \cdots, t_{\mathrm{comm1}}^1 + K f(t_{\mathrm{comm1}}^1)]$ }{
        \textbf{Start ping}. Play communicating arm $i_{\mathrm{comm}}$.  \\
   }
  \ElseIf{ $j \geq 2$ and \begin{small}$t \in [t_{\mathrm{comm1}}^1 + (j-1) K f(t_{\mathrm{comm1}}^1) + 1, \cdots, t_{\mathrm{comm1}}^1 + j K f(t_{\mathrm{comm1}}^1)]$\end{small} and 
  $\mathbf{b}_{j-1} = 1$ }{
        Play communicating arm $i_{\mathrm{comm}}$. \\
   }\ElseIf{$j \geq 2$ and \begin{small}$t \in [t_{\mathrm{comm1}}^1 + (j-1) K f(t_{\mathrm{comm1}}^1) + 1, \cdots, t_{\mathrm{comm1}}^1 + j K f(t_{\mathrm{comm1}}^1)]$\end{small} and 
  $\mathbf{b}_{j-1} = 0$}{
   Play Round Robin arm $t - \lfloor \frac{t-1}{K}\rfloor K$. \\
   }
}
    
\caption{COMMUNICATE (Player 1) }
\label{algorithm::player_1_communicate}
\end{algorithm}

\subsection{Detailed Discussion and Missing Supporting Results for Communication Analysis}\label{section::communication_analysis_discussion_supporting_results}

The main objective of this section is to present a proof of

\begin{restatable}[One Bit Recovery Zero Collision Reward]{lemma}{exactrecoveryzerotestonebit}\label{lemma::zero_test_works_main}
Let $A^p = \{ L_i^p(t_{\mathrm{start}}^1)\}_{i \in [K]}$, $B^p = \{ \widehat{\mu}_i^p(t_{\mathrm{start}}^1)\}_{i \in [K]}$ , $C^p = \{ \widehat{\mu}_i^p(t_{\mathrm{start}}^1+1:t_{\mathrm{start}}^1+Kf(t_{\mathrm{start}}^1) )\}_{i \in [K]}$ for all $p$. If the good event $\mathcal{E}$ holds then, with probability at least $1-\frac{\delta}{4K^2}$, all players $p \in \{ 2, \cdots, M\}$ will be able to recover exactly the bit transmitted by player $1$ by calling $\mathrm{ZeroTest}( A^p, B^p, C^p )$.
\end{restatable}

thus establishing the reliability of the $\mathrm{ZeroTest}$. We will now show that the following three statements hold with high probability, 
 
\begin{enumerate}
    \item \textbf{Arm $\widehat{\sigma}_1$ has a large empirical mean for all players. } 
    \begin{itemize}
        \item Arm $\widehat{\sigma}_1 \in \{ i \in [K] \text{ s.t. }\widehat{\mu}_i^p(t_{\mathrm{start}}^1) \geq \frac{1}{2} \max_{j \in [K]} \widehat{\mu}_j^p(t_{\mathrm{start}}^1)\}$ for all $p \in \{2, \cdots, M\}$ and $\widehat{\mu}^p_{\widehat{\sigma}_1}(t_{\mathrm{start}}^1) -D(N_{\widehat{\sigma}_1}(t_{\mathrm{start}}^1)) \geq \frac{3\mu_{\sigma_1}}{4}$.
    \end{itemize} 
    \item \textbf{Arm $\widehat{\sigma}_1$ is comparable to $\sigma_1$. } 
    \begin{itemize}
        \item The witnesses $L_{\widehat{\sigma}_1}^p = \frac{\widehat{\mu}^p_{\widehat{\sigma}_1}(t_{\mathrm{start}}^1) - D(N_{\widehat{\sigma}_1}(t_{\mathrm{start}}^1))}{2} \in [\frac{\mu_{\widehat{\sigma}_1}}{3}, \frac{\mu_{\widehat{\sigma}_1}}{2}]$ for all $p \in \{ 2, \cdots, M\}$.
    \end{itemize}
    \item \textbf{When collisions are avoided the empirical mean estimators $\widehat{\mu}^p_{\widehat{\sigma}_1}(t_{\mathrm{start}}^1 +1:  + Kf(t_{\mathrm{start}}^1))$ are far from zero. } 
    \begin{itemize}
        \item If $b = 1$, the estimators $\widehat{\mu}^p_{\widehat{\sigma}_1}(t_{\mathrm{start}}^1 +1:  + Kf(t_{\mathrm{start}}^1)) \leq L_{\widehat{\sigma}_1}^p$ for all $p \in [M]$.
        \item If $b = 0$, the estimators $\widehat{\mu}^p_{\widehat{\sigma}_1}(t_{\mathrm{start}}^1:t_{\mathrm{start}}^1 + Kf(t_{\mathrm{start}}^1)) > L_{\widehat{\sigma}_1}^p$ for all $p \in [M]$.
    \end{itemize}
\end{enumerate}

We start by presenting an 'easy-to-read' proof sketch, followed by an in-depth discussion of the more nuaced aspects of the proof.

\begin{proof}[sketch] Let's assume that $\mathcal{E}$ holds. We start by showing that because $t_{\mathrm{start}}^1 \geq \frac{128}{\mu_{\sigma_1}^2} g(t_{\mathrm{start}}^1/K)$ (see Equation~\ref{equation::condition_t_0_description_main}) for all $p \in [M]$ the lower confidence bound estimator around the optimal arm $\sigma_1$ computed at time $t_{\mathrm{start}}^1$ is at least a constant proportion of its magnitude:
\begin{equation}\label{equation::lower_bound_equation_main}
 \widehat{\mu}^p_{\sigma_1}(t_{\mathrm{start}}^1) - D(N_{\sigma_1}^p(t_{\mathrm{start}}^1)) \geq \frac{3\mu_{\sigma_1}}{4}.
\end{equation}
See Lemma~\ref{lemma::lower_bound_LCB} for a proof of Equation~\ref{equation::lower_bound_equation_main}. Since $\widehat{\sigma}_1 = \arg\max_{i \in [K]} \widehat{\mu}_{i}(t_{\mathrm{start}}^1) - D(N_{\sigma_1}^p(t_{\mathrm{start}}^1)) $ and $\mathcal{E}$ holds, we conclude that $\mu_{\widehat{\sigma}_1} \geq \frac{3\mu_{\sigma_1}}{4}$. In other words, the true mean of $\widehat{\sigma}_1$ is at least a constant ($3/4$) multiple of the mean reward of the maximal arm. This observation can be used to show that whenever the good event holds: 
\begin{enumerate}
    \item[A)]  Arm $\widehat{\sigma}_1$ is always in the set of arms inspected during the ZeroTest; i.e., 
    \begin{equation*}
     \widehat{\mu}_{\widehat{\sigma}_1}(t_{\mathrm{start}}^1) \geq \frac{1}{2} \max_{j \in [K]} \mu_{j\in [K]} \widehat{\mu}_j^p(t_{\mathrm{start}}^1)   .
    \end{equation*}
    \item[B)] The witnesses $L_{\widehat{\sigma}_1}^p(t_{\mathrm{start}}^1 ) \in \left[ \frac{\mu_{\widehat{\sigma}_1}}{3}, \frac{\mu_{\widehat{\sigma}_1}}{2} \right]$
\end{enumerate}
Both A) and B) hold whenever $\mathcal{E}$ holds, and thus occur with probability at least $1-\delta$. See Lemma~\ref{lemma:witness_bounds} for a proof of this claim. The fundamental property of the witnesses $L_{\widehat{\sigma}_1}^p(t_{\mathrm{start}}^1)$ is that they are neither close to zero nor close to $\mu_{\widehat{\sigma}_1}$. Since $\mu_{\widehat{\sigma}_1} \geq \frac{3\mu_{\sigma_1}}{4}$, the witnesses are bounded away from zero and from $\mu_{\widehat{\sigma}_1}$ by a factor of at least $\frac{\mu_{\sigma_1}}{4}$. 

The remainder of this proof sketch is based on the discussion and results from Appendix~\ref{section::the_zero_test}. Let $\mathbb{P}_X$ be a distribution with support over $[0,1]$ and mean $\mu_X$ equal to either the null (always zero) distribution or to the reward distribution of arm $\widehat{\sigma}_1$. We consider and answer the following question. Provided we have knowledge of the witnesses $\{L_{\widehat{\sigma}_1}^p(t_{\mathrm{comm1}}^1 )\}_{p \in \{2, \cdots, M\}}$, how many i.i.d. samples from $\mathbb{P}_X$ are needed to be able to distinguish with probability $1-\delta$ what type of distribution the samples come from (either the zero distribution or the reward distribution of arm $\widehat{\sigma}_1$). 

Since the witnesses are all in the interval $\left[ \frac{\mu_{\widehat{\sigma}_1}}{3}, \frac{\mu_{\widehat{\sigma}_1}}{2} \right]$, it follows that $\mu_{\widehat{\sigma}_1} - L_{\widehat{\sigma}_1}^p\geq \frac{\mu_{\widehat{\sigma}_1}}{2} \geq \frac{3\mu_{\sigma_1}}{8}$. This implies that in order to distinguish between these distributions it is enough to estimate $\mu_X$ up to accuracy $\mathcal{O}(\mu_{\sigma_1})$. 

As a consequence of Assumption~\ref{assumption::bounded_support} we see the variance of the reward distribution of arm $\widehat{\sigma}_1$ is upper bounded by $\mu_{\widehat{\sigma}_1}(1-\mu_{\widehat{\sigma}_1}) \leq \mu_{\sigma_1}$ and therefore the variance of $\mathbb{P}_X$ is also upper bounded by $\mu_{\sigma_1}$. Using this variance bound along with a Uniform Empirical Bernstein bound (see Lemmas~\ref{lemma:uniform_emp_bernstein} and~\ref{lemma::bernstein_concentration_application} in Appendix~\ref{section::appendix_zero_test_supporting}) we can show that with probability at least $1-\delta$ it is enough for player $p$ to look at the empirical average of $N_p$ samples from $\mathbb{P}_X$, where $N_p$ is such that $\frac{N_p}{B(N_p, \frac{\delta}{4K^2M})} \geq \frac{48}{L_{\widehat{\sigma}_1}^p}$. 

We require the length of the communication rounds to be player independent. Since all witnesses are in $\left[ \frac{\mu_{\widehat{\sigma}_1}}{3}, \frac{\mu_{\widehat{\sigma}_1}}{2} \right]$ and $t_{\mathrm{start}}^1$ can be related to $\mu_{\sigma_1}$ via Equation~\ref{equation::condition_t_0_description_main}, we instead define $N$ (common to all players $p$) to be a function of $t_{\mathrm{start}}^1$. We show that instead of $N_p$ we can use a player-independent value $N = f(t_{\mathrm{start}}^1)$ such that $f(t_{\mathrm{start}}^1)$ is the first integer such that $\frac{f(t_{\mathrm{start}}^1)}{B(f(t_{\mathrm{start}}^1), \frac{\delta}{4K^2M})} \geq 24\sqrt{\frac{t_{\mathrm{start}}^1/K}{2g(t_{\mathrm{start}}^1/K)} }$ also achieves this goal. This definition of $f$ gracefully extends to all integers $t$. %

\end{proof}

The following discussion will be devoted to present a detailed proof of Lemma~\ref{lemma::zero_test_works_main}.

\paragraph{Detailed Proof of Lemma~\ref{lemma::zero_test_works_main}}  Algorithm~\ref{algorithm::communication_protocol} is a simplified version of the full communication protocol which we introduced in Section~\ref{section::algorithm}. First, we assume the algorithm takes as input a known time index $t_\mathrm{start}^1$, a bit $b\in \{0,1\}$ that player $1$ wishes to communicate to all other players and the communication rounds length function $f: \mathbb{R}\rightarrow \mathbb{N}$ that will determine the length of the communication rounds as a function of $t$. The algorithm works as follows. Initially all players in $\{1, \cdots , M\}$ play arms $1$ through $K$ in a Round Robin fashion until time $t_\mathrm{start}^1$, which is assumed to be a special round such that for all $i, j \in [K]$ and $p, p' \in [M]$, the counts $N_i^p(t_{\mathrm{start}}^1) = N_j^{p'}(t_{\mathrm{start}}^1)$ (this means $t_\mathrm{start}^1$ must be a multiple of $K$). After this time has elapsed each player $p$ has access to empirical estimators $\widehat{\mu}_i^p(t_{\mathrm{start}}^1)$ of $\mu_i$ for all $i \in [K]$.  At this time, player $1$ (the communicating player) computes a guess for the maximal arm $\widehat{\sigma}_1 \in [K]$. This will be the arm that player $1$ will use to transmit bit $b$. The remaining $f(t_{\mathrm{start}}^1)$ rounds, which we call the communication rounds (where $f$ is a function known by all players $p \in [M]$) are used by player $1$ to transmit $b$ and by the other players $p \neq 1$ to receive it.

If $b=1$, player $1$ will keep playing in a Round Robin fashion during the communication rounds, while if $b=0$, player $1$ will instead play arm $\widehat{\sigma}_1$. When transmitting a single bit, the length of the communication round equals $Kf(t_{\mathrm{start}}^1)$ rounds. At the end of the communication rounds, all players (except possibly player $1$ if transmitting $b=0$) will have played each arm $f(t_{\mathrm{start}}^1)$ times. 

Using these $f(t_{\mathrm{start}}^1)$ samples per arm the players $p \in \{2, \cdots, M\}$ will conduct a test with the objective of verifying if there was any arm whose reward estimator $\widehat{\mu}(t_{\mathrm{start}}^1+1 : t_{\mathrm{start}}^1 + Kf(t_{\mathrm{start}}^1) )$ was substantially lower than the average values $\widehat{\mu}_i^p(t_{\mathrm{start}}^1)$ computed up to time $t_{\mathrm{start}}^1$. If this is the case, they can safely conclude player $1$ was pulling the same arm throughout the communication rounds $[t_\mathrm{start}^1 + 1, \cdots, t_\mathrm{start}^1+Kf(t_\mathrm{start}^1)]$, and therefore incurring in collisions with player $1$, while if they do not detect any substantial difference in their estimators, they can conclude that player $1$ continued to play in a Round Robin fashion. We explain this procedure in more detail in Algorithms~\ref{algorithm::communication_protocol} and~\ref{algorithm::zero_test_appendix}.

When $b= 1$ player $1$ plays arm $\widehat{\sigma}_1$ during all rounds from $t=t_\mathrm{start}^1+1$ to $t =t_\mathrm{start}^1 + Kf(t_\mathrm{start}^1)$ while the remaining players are playing in a Round Robin fashion. Collisions will occur anytime a player distinct from 1 attempts to pull arm $\widehat{\sigma}_1$. At this moment, both players will receive a reward of zero. In case $b=0$, no collisions will occur during rounds $t = t_\mathrm{start}^1+1, \cdots, t_\mathrm{start}^1 + Kf(t_\mathrm{start}^1)$ and the reward each player receives from pulling arm $\widehat{\sigma}_1$ should have a mean value of $\mu_{\widehat{\sigma}_1}$. In order for player $p' \neq 1$ to discern if player $1$ is transmitting a one or a zero, there are two challenges. First, since the optimal arm index estimator $\widehat{\sigma}_1$ is random, none of the players $p' \neq 1$ knows the precise identity of arm $\widehat{\sigma}_1$. Second, none of the players knows the exact value of the true mean $\mu_{\widehat{\sigma}_1}$ and therefore, discerning between samples of the null (always zeros) distribution $\mathbb{P}_{\mathrm{null}}$ and $\mathbb{P}_{\widehat{\sigma}_1}$ may prove challenging. We address these issues below.

Algorithm~\ref{algorithm::communication_protocol} below contains the detailed description of the simplified Communication Protocol.

   \begin{algorithm}[H]
\textbf{Input} Players $[M]$, arms $[K]$, bit to communicate $b \in \{ 0,1\}$, communication round start $t_\mathrm{start}^1$, communication rounds length function $f : \mathbb{R} \rightarrow \mathbb{N}$.\\
   \For{$t=1, \cdots, t_\mathrm{start}^1 $}{
    All players $p \in [M]$ play arms $[K]$ in a Round Robin fashion. \\
    \For{$p \in [M]$}{
        Play following a Round Robin schedule.  
    
    }
    }
    Player $1$ computes a guess $\widehat{\sigma}_1 \in [K]$ for the maximal arm $\sigma_1$ :
    \begin{equation*}
       \widehat{\sigma}_1 = \argmax_{i\in [K]} \widehat{\mu}^1_{i}(t_\mathrm{start}^1) - D(N^1_{i}(t_\mathrm{start}^1))
    \end{equation*}
    \For{$t = t_\mathrm{start}^1 + 1, \cdots, t_\mathrm{start}^1 + K f(t_\mathrm{start}^1)$}{
        \If{ $b=1$ }{
            Player $1$ plays arm $\widehat{\sigma}_1$. \\}
        \Else{ Player $1$ plays arms $[K]$ following a Round Robin schedule. }            
                All players $p \neq 1$ continue playing arms $[K]$ following a Round Robin schedule. \\

    }

\caption{One Bit Communication Protocol Appendix}
\label{algorithm::communication_protocol}
\end{algorithm} 

In order to recover the value of the transmitted bit $b$ at the end of round $t_\mathrm{start}^1 + Kf(t_\mathrm{start}^1)$ all players $p \neq 1$ will compare $\widehat{\mu}_i^p(t_\mathrm{start}^1+1 : t_\mathrm{start}^1 + Kf(t_\mathrm{start}^1))$ with the values $\widehat{\mu}_i^p(t_\mathrm{start}^1)$. If player $p \neq 1$ detects  $\widehat{\mu}_i^p(t_\mathrm{start}^1+1 : t_\mathrm{start}^1 + Kf(t_\mathrm{start}^1))$ to be substantially lower than $\widehat{\mu}_i^p(t_\mathrm{start}^1)$ for any arm $i$, then it can conclude there have been collisions and that $b = 1$, otherwise it will conclude that $b=0$. We turn these intuitions into a precise mechanism in the discussion that follows.

     For all $p \in [M]$ and all $i \in [K]$ define the lower bound 'witnesses' as,
  \begin{equation*}
  L_i^p(t) := \frac{\widehat{\mu}^p_i(t) - D(N_i^p(t))}{2}
  \end{equation*}

We now consider the following test to be executed by all players $p \in [M]$ with the data collected during rounds $t_0+1,  \cdots, t_0+ K f(t_0)$ and designed to decode bit $b$ transmitted by player $1$. Let $\widehat{\mu}_i^p(t_\mathrm{start}^1+1: t_\mathrm{start}^1 + Kf(t_\mathrm{start}^1))$ be the empirical mean of arm $i$ computed by player $p$ and using only samples from $t = t_\mathrm{start}^1+1, \cdots, t_\mathrm{start}^1+Kf(t_\mathrm{start}^1)$. Since all players $p \neq 1$ are playing the arms in $[K]$ following a Round Robin schedule this estimator will consist of $f(t_\mathrm{start}^1)$ samples for each arm $i \in [K]$. The following test will be used by all players $p \in [M]$ such that $p \neq 1$ to decode $b$,

Recall that $t_{\mathrm{start}}^1$ is assumed to satisfy $t_\mathrm{start}^1 \geq t_{\mathrm{first}}^1$. Since $t_\mathrm{start}^1$ is assumed to be a special round, for all players $p, p' \in [M]$ and arms $i, j \in [K]$ the number of pulls satisfies $N_{i}^p(t_\mathrm{start}^1) = N_{j}^{p'}(t_\mathrm{start}^1)=\frac{t_\mathrm{start}^1}{K}$ and for the maximal arm $\sigma_1$,
  \begin{equation}\label{equation::condition_t_0}
\frac{2}{ D^2(N_{\sigma_1}^p(t_\mathrm{start}^1) )} =    \frac{ N_{\sigma_1}^p(t_\mathrm{start}^1) }{g(N_{\sigma_1}^p(t_\mathrm{start}^1))} \geq \frac{128}{\mu_{\sigma_1}^2} 
  \end{equation}
For the remainder of this subsection and for all $p \in [M]$ and $i \in [K]$ we denote $N_{i}^p(t_\mathrm{start}^1)$ as $N(t_\mathrm{start}^1)$ to refer to $\frac{t_\mathrm{start}^1}{K}$, the number of times each arm was played by any player $p \in [M]$ up to time $t_\mathrm{start}^1$. In Lemma~\ref{lemma::lower_bound_LCB} we see that whenever $t_\mathrm{start}^1$ satisfies Equation~\ref{equation::condition_t_0}, the lower confidence bound estimators around the optimal arm $  \widehat{\mu}_{\sigma_1}^p(t_\mathrm{start}^1) - D(N_{\sigma_1}^p(t_\mathrm{start}^1))$ are at least a constant fraction of the value of $\mu_{\sigma_1}$. 
\begin{lemma}\label{lemma::lower_bound_LCB}
If $t_\mathrm{start}^1$ satisfies Equation~\ref{equation::condition_t_0} and $\mathcal{E}$ holds, the lower confidence bound estimator around the optimal arm $\sigma_1$ is at least a constant proportion of its magnitude %
\begin{equation*}
     \widehat{\mu}_{\sigma_1}^p(t_\mathrm{start}^1) - D(N_{\sigma_1}^p(t_\mathrm{start}^1)) \geq \frac{3\mu_{\sigma_1}}{4}. 
\end{equation*}
\end{lemma}

\begin{proof}
If $\mathcal{E}$ holds, $\widehat{\mu}_{\sigma_1}^p(t_\mathrm{start}^1) \in [\mu_{\sigma_1} -  D(N_{\sigma_1}^p(t_\mathrm{start}^1)), \mu_{\sigma_1} +  D(N_{\sigma_1}^p(t_\mathrm{start}^1))]$ and therefore $\widehat{\mu}_{\sigma_1}(t_\mathrm{start}^1) - D(N_{\sigma_1}^p(t_\mathrm{start}^1)) \geq \mu_{\sigma_1} - 2 D(N_{\sigma_1}^p(t_\mathrm{start}^1))$.  Since $t_\mathrm{start}^1$ satisfies Equation~\ref{equation::condition_t_0}, $$D(N_{\sigma_1}^p(t_\mathrm{start}^1)) \leq \frac{\mu_{\sigma_1}}{8}$$ and we conclude that $\widehat{\mu}^p_{\sigma_1}(t_\mathrm{start}^1) - 2 D(N_{\sigma_1}^p(t_\mathrm{start}^1)) \geq \frac{3\mu_{\sigma_1}}{4}$. \end{proof}

We can use the result of Lemma~\ref{lemma::lower_bound_LCB} to show that whenever $\mathcal{E}$ is true and Equation~\ref{equation::condition_t_0} holds for $t_\mathrm{start}^1$ the witnesses of $\widehat{\sigma}_1$ are both upper and lower bounded by constant multiples of the true mean $\mu_{\widehat{\sigma}_1}$. %

  \begin{restatable}{lemma}{lemmawitnessbounds}\label{lemma:witness_bounds}
    Whenever $\mathcal{E}$ holds and Equation~\ref{equation::condition_t_0} holds for $t_\mathrm{start}^1$, all witnesses of arm $\widehat{\sigma}_1$ for all $p \in [M]$ satisfy
  \begin{equation}\label{equation::L_conditions}
   L_{\widehat{\sigma}_1}^p(t_\mathrm{start}^1) \in \left[ \frac{ \mu_{\widehat{\sigma}_1} }{3}, \frac{\mu_{\widehat{\sigma}_1}}{2}  \right]   
  \end{equation}
 Furthermore $\mu_{\widehat{\sigma}_1} \geq \frac{3\mu_{\sigma_1}}{4}$ and $\widehat{\mu}_{\widehat{\sigma}_1}(t_\mathrm{start}^1) \geq\frac{1}{2}\max_{j\in [K] } \widehat{\mu}_j^p(t_\mathrm{start}^1)$. 
  \end{restatable}

   \begin{proof}
  Recall that $\widehat{\sigma}_1 = \argmax_{i\in [K]} \widehat{\mu}^1_{i}(t_\mathrm{start}^1) - D(N^1_{i}(t_\mathrm{start}^1))$. It follows that 
  \begin{align*}
   \widehat{\mu}^1_{\widehat{\sigma}_1}(t_\mathrm{start}^1) - D(N^1_{\widehat{\sigma}_1}(t_\mathrm{start}^1)) \geq \widehat{\mu}^1_{\sigma_1}(t_\mathrm{start}^1) - D(N^1_{\sigma_1}(t_\mathrm{start}^1))   \stackrel{(i)}{\geq} \frac{3\mu_{\sigma_1}}{4}\stackrel{(ii)}{\geq} \frac{3\mu_{\widehat{\sigma}_1}}{4},
  \end{align*}
where inequality $(i)$ holds by Lemma~\ref{lemma::lower_bound_LCB} and inequality $(ii)$ holds by definition of $\mu_{\sigma_1}$. Furthermore, notice that since $\mathcal{E}$ holds, 
 \begin{equation}\label{equation::upper_bound_LCB}
     \widehat{\mu}^p_{\widehat{\sigma}_1}(t_\mathrm{start}^1) - D(N^p_{\widehat{\sigma}_1}(t_\mathrm{start}^1)) \leq \mu_{\widehat{\sigma}_1}.
 \end{equation}
for all $p \in [M]$.

Combining these two inequalities together we see that $\mu_{\widehat{\sigma}_1} \geq \frac{3\mu_{\sigma_1}}{4}$. Recall that $t_\mathrm{start}^1$ is a special round. By definition we have $N_{i}^p(t_\mathrm{start}^1) = N_{j}^{p'}(t_\mathrm{start}^1)$ for all $i,j \in [K]$ and all $p, p' \in [M]$, and $t_\mathrm{start}^1 \geq t_{\mathrm{first}}^1$ therefore since $\mathcal{E}$ holds Equation~\ref{equation::condition_t_0} must be satisfied for all players. Equation~\ref{equation::condition_t_0} implies that $D(N_{\widehat{\sigma}_1}^p(t_\mathrm{start}^1)) \leq \frac{\mu_{\sigma_1}}{8}$ and therefore that $D(N_{\widehat{\sigma}_1}^p(t_\mathrm{start}^1)) \leq \frac{\mu_{\widehat{\sigma}_1}}{6}$. Combining these two observations we see that for all $p \in [M]$, 
 \begin{equation}\label{equation::lower_bound_LCB}
   \widehat{\mu}^p_{\widehat{\sigma}_1}(t_\mathrm{start}^1) -D(N_{\widehat{\sigma}_1}^p(t_\mathrm{start}^1))   \stackrel{(i)}{\geq} \mu_{\widehat{\sigma}_1} - 2D(N_{\widehat{\sigma}_1}^p(t_\mathrm{start}^1)) \stackrel{(ii)}{\geq} \frac{ 2\mu_{\widehat{\sigma}_1}}{3} 
 \end{equation}
Inequality $(i)$ is satisfied because $\mathcal{E}$ holds, and $(ii)$ is a consequence of the observation that $D(N_{\widehat{\sigma}_1}^p(t_\mathrm{start}^1)) \leq \frac{\mu_{\widehat{\sigma}_1}}{6}$. Combining Equations~\ref{equation::upper_bound_LCB} and~\ref{equation::lower_bound_LCB} we conclude that
 \begin{equation*}
    \frac{\mu_{\widehat{\sigma}_1(t_\mathrm{start}^1)}}{3} \leq  L_{\widehat{\sigma}_1}^p(t_\mathrm{start}^1) = \frac{\widehat{\mu}_{\widehat{\sigma}_1}(t_\mathrm{start}^1) -D(N_{\widehat{\sigma}_1}^p(t_\mathrm{start}^1))  }{2} \leq \frac{\mu_{\widehat{\sigma}_1(t_\mathrm{start}^1)}}{2}.
 \end{equation*}
 For the second part, the following sequence of inequalities holds:
 \begin{equation*}
     \widehat{\mu}_{\widehat{\sigma}_1}^p(t_\mathrm{start}^1) \stackrel{(i)}{\geq} \mu_{\widehat{\sigma}_1} - D(N_{\widehat{\sigma}_1}^p(t_\mathrm{start}^1)) \stackrel{(ii)}{\geq} \frac{3\mu_{\sigma_1}}{4} - D(N_{\widehat{\sigma}_1}^p(t_\mathrm{start}^1)) \stackrel{(iii)}{\geq} \frac{5\mu_{\sigma_1}}{8}.
 \end{equation*}
Inequality $(i)$ is satisfied because Equation~\ref{equation::lower_bound_LCB} is true whenever $\mathcal{E}$ holds. Inequality $(ii)$ is a consequence of $\mu_{\widehat{\sigma}_1} \geq \frac{3 \mu_{\sigma_1}}{4}$ and inequality $(iii)$ holds because $D(N_{\widehat{\sigma}_1}^p(t_\mathrm{start}^1)) \leq \frac{\mu_{\sigma_1}}{8}$. The later implies that
\begin{equation*}
   \frac{9}{8} \mu_{\sigma_1} \geq \mu_{\sigma_1} + D(N_{\widehat{\sigma}_1}^p(t_\mathrm{start}^1)) \geq \max_{j \in [K]} \widehat{\mu}_j^p(t_\mathrm{start}^1),
\end{equation*}
where the inequality on the RHS holds because $\mathcal{E}$ is satisfied and $\mu_{\sigma_1}$ is the maximal arm. We conclude that
 \begin{equation*}
     \widehat{\mu}_{\widehat{\sigma}_1}^p(t_\mathrm{start}^1) \geq \frac{5 \max_{j \in [K]} \widehat{\mu}_j^p(t_\mathrm{start}^1) }{9}.
 \end{equation*}
Since $5/9 > 1/2$ the result follows.

  \end{proof}
  
 Among other things Lemma~\ref{lemma:witness_bounds} implies that whenever $\mathcal{E}$ holds, arm $\widehat{\sigma}_1$ (the \emph{random} arm used by player $1$ to communicate) is among the arms inspected by all players $p \neq 1$ during the Zero Test in Algorithm~\ref{algorithm::zero_test_appendix}. In section~\ref{section::the_zero_test} we show how to choose $f$ in order to ensure that  Algorithm~\ref{algorithm::zero_test_appendix} allows all players $p \neq 1$ to decode $b$ with high probability. This is the same functional form that is highlighted at the start of Section~\ref{section::assumptions_notation}.

\subsubsection{Detailed Analysis of the Zero Test}\label{section::the_zero_test}

We will take a small step back and consider a simplified version of the bit communication protocol which will be useful in analyzing the Zero Test of Algorithm~\ref{algorithm::zero_test_appendix}. Let $X$ be a random variable with support in $[0,1]$ and mean $\mu_X$. Assume $\mu_X > L$ for $L$ known. Let $Z_1, \cdots, Z_N$ be $N$ i.i.d. samples from either $\mathbb{P}_X$ or the null distribution $\mathbb{P}_{\mathrm{null}}$ (all $Z_i = 0$) and let $\widehat{\mu}_Z = \frac{1}{N} \sum_{i=1}^N Z_i$. The problem we consider is the following: 

\emph{How many i.i.d. samples are required to determine with high probability from which of the two distributions ($\mathbb{P}_X$ or $\mathbb{P}_{\mathrm{null}}$) do the samples of $\{Z_i\}_{i=1}^N$ come from?} 

We will analyze the following simplified version of the zero test,

\begin{equation}\label{equation::zero_test}
    \text{If } \widehat{\mu}_Z \geq L, \text{ then output }\mathbb{P}_X \text{ else output }\mathbb{P}_{\mathrm{null}} \tag{Zero-Test-Simple}
\end{equation}

We use the fact that for any random variable $\mathbb{P}_X$ with support in $[0,1]$ and mean $\mu_X$ the variance can be upper bounded by $\mu_X(1-\mu_X)$ (see Lemma~\ref{lemma::variance_bounded_rv} in Appendix~\ref{section::appendix_zero_test_supporting}) in conjunction with a Uniform Empirical Bernstein bound (see Lemmas~\ref{lemma:uniform_emp_bernstein} and~\ref{lemma::bernstein_concentration_application} in Appendix~\ref{section::appendix_zero_test_supporting}) to show the following bound on the required number of samples $N$,
\begin{restatable}{lemma}{zerotestpreliminarybound}\label{lemma::zero_test_preliminary_bound}
Let $\delta' \in (0,1)$. If $X$ is a random variable with support in $[0,1]$, distribution $\mathbb{P}_X$ and mean $\mu_X$ satisfying $L \leq \mu_X$, then with probability at least $1-\delta'$ for all $N$ such that $\frac{N}{B(N, \delta')} \geq \max\left(  \frac{2}{\mu_X - L}, \frac{16 \min(\mu_X, 1-\mu_X)}{(\mu_X - L)^2} \right)$ we have,
\begin{equation*}
    \widehat{\mu}_X \geq L, 
\end{equation*}
where $B(n, \delta') = 2\ln \ln(2n) + \ln \frac{5.2}{\delta'}$.
\end{restatable}

\begin{proof}

Let $\alpha = \min(\mu_X, 1-\mu_X)$. A simple use of Lemma~\ref{lemma::bernstein_concentration_application} implies that with probability at least $1-\delta'$ for all $n\in \mathbb{N}$:
\begin{align*}
    \widehat{\mu}_X \geq \mu_X - 2 \sqrt{ \frac{ \alpha B(n, \delta')}{n}} - \frac{ B(n, \delta')}{n}.
\end{align*}
The LHS of this inequality attains a value of at least $L$ whenever:
\begin{equation*}
    \mu_X -L \geq  2 \sqrt{ \frac{ \alpha B(n, \delta')}{n}} + \frac{ B(n, \delta')}{n}.
\end{equation*}
We finalize  the proof by noting that for all $n$ such that $\frac{n}{B(n, \delta')} \geq \max\left(  \frac{2}{\mu_X - L}, \frac{16 \min(\mu_X, 1-\mu_X)}{(\mu_X - L)^2} \right)$ we have that $\frac{\mu_X - L}{2} \geq 2 \sqrt{ \frac{ \alpha B(n, \delta')}{n}}$ and $\frac{\mu_X - L}{2} \geq \frac{B(n, \delta')}{n}$.  
\end{proof}

 The bound in Lemma~\ref{lemma::zero_test_preliminary_bound} implies that when the sampling distribution for $Z$ equals $\mathbb{P}_X$, the empirical mean  $\widehat{\mu}_Z$ will be larger than $L$ with high probability provided the number of test samples $\{Z_i\}_{i=1}^N$ is large enough. Observe that $N$ has an inverse dependence on the gap $\mu_X - L$. We will be applying this result to the case where although $L \neq \mu_X$ it is of the order of $\mu_X$. 

\begin{lemma}\label{lemma::zero_test_success_conditions}
Let $\delta' \in (0,1)$. Assume that $\frac{\mu_X}{2} \geq L \geq \frac{\mu_X}{3}$ and let be $N$ be an integer such that $\frac{N}{B(N, \delta')} \geq \frac{48}{L}$, then \ref{equation::zero_test} succeeds with probability at least $1-\delta'$.  
\end{lemma}
\begin{proof}
This follows from Lemma~\ref{lemma::zero_test_preliminary_bound} by noting that in this case $\frac{1}{\mu_X - L} \leq \frac{1}{L}$ and $\frac{16 \min(\mu_X, 1-\mu_X)}{(\mu_X - L)^2} \leq \frac{48}{L}$.\end{proof}

Lemma~\ref{lemma::zero_test_success_conditions} is an instantiation of the results of Lemma~\ref{lemma::zero_test_preliminary_bound} when $\mu_X - L$ is of the order of $\mu_X$. This result says that up to logarithmic factors, it is enough for $N \approx \frac{1}{\mu_X}$ for the empirical estimator $\widehat{\mu}_Z$ to be at least a constant fraction of the true mean $\mu_X$.

We can now apply these results to the Communication Protocol in Algorithm~\ref{algorithm::communication_protocol} and the Zero Test in Algorithm~\ref{algorithm::zero_test_appendix}. Recall that by definition of $t_{\mathrm{start}}^1$ we have $t_{\mathrm{start}}^1 \geq t_{\mathrm{first}}^1$ $N_{i}^p(t_{\mathrm{start}}^1) = N_{j}^{p'}(t_{\mathrm{start}}^1) = N(t_{\mathrm{start}}^1)$ for all $i,j \in [K]$ and all $p, p' \in [M]$.

\begin{lemma}
Let $\delta \in (0,1)$ and $t_{\mathrm{start}}^1$ satisfy Equation~\ref{equation::condition_t_0}. If $f(t_{\mathrm{start}}^1)$ is an integer such that $\frac{f(t_{\mathrm{start}}^1)}{B(f(t_{\mathrm{start}}^1), \frac{\delta}{4K^2M})}  \geq 24\sqrt{ \frac{t_{\mathrm{start}}^1/K}{2g(t_{\mathrm{start}}^1/K)}}$ then whenever the good event $\mathcal{E}$ holds, at the end of the Communication protocol in Algorithm~\ref{algorithm::communication_protocol} all players $p \in [M]$ with $p \neq 1$ will be able to recover exactly the bit transmitted by player $1$ via the Zero Test of Algorithm~\ref{algorithm::zero_test_appendix} with probability at least $1-\frac{\delta}{4K^2}$. 
\end{lemma}

\begin{proof}
Let's assume $\mathcal{E}$ holds. Equation~\ref{equation::condition_t_0} implies that $\mu_{\sigma_1} \geq 8 \sqrt{ \frac{g(N(t_{\mathrm{start}}^1))}{N(t_{\mathrm{start}}^1)} }$ Lemma~\ref{lemma::lower_bound_LCB} tells us that $\mu_{\widehat{\sigma}_1} \geq \frac{3\mu_{\sigma_1}}{4}$ and therefore $\mu_{\widehat{\sigma}_1} \geq 6 \sqrt{\frac{g(N(t_{\mathrm{start}}^1))}{N(t_{\mathrm{start}}^1)} }$. Furthermore, by  Equation~\ref{equation::L_conditions} of Lemma~\ref{lemma:witness_bounds}, $L_{\widehat{\sigma}_1}^p(t_{\mathrm{start}}^1) \in \left[ \frac{ \mu_{\widehat{\sigma}_1} }{3}, \frac{\mu_{\widehat{\sigma}_1}}{2}  \right]$ and therefore $    L^p_{\widehat{\sigma}_1}  \geq 2 \sqrt{ \frac{g(N(t_{\mathrm{start}}^1))}{N(t_{\mathrm{start}}^1)} }$. We can then conclude that $\frac{48}{L^p_{\widehat{\sigma}_1}} \leq 24\sqrt{ \frac{N(t_{\mathrm{start}}^1)}{g(N(t_{\mathrm{start}}^1))}} $ for all $p \in [M]$. 

Recall that $N(t_{\mathrm{start}}^1) = \frac{t_{\mathrm{start}}^1}{K}$ (a known function of $t_{\mathrm{start}}^1$). If we define $f(t_{\mathrm{start}}^1)$ to be the any integer\footnote{For example the first one that satisfies this bound.} such that $\frac{f(t_{\mathrm{start}}^1)}{B(f(t_{\mathrm{start}}^1), \frac{\delta}{4K^2M})}  \geq 24\sqrt{ \frac{N(t_{\mathrm{start}}^1)}{g(N(t_{\mathrm{start}}^1))}} = 24\sqrt{ \frac{t_{\mathrm{start}}^1/K}{g(t_{\mathrm{start}}^1/K)}} $, the conditions of Lemma~\ref{lemma::zero_test_success_conditions} are satisfied  since $\frac{f(t_{\mathrm{start}}^1)}{B(f(t_{\mathrm{start}}^1), \frac{\delta}{4K^2M})}  \geq \frac{48}{L_{\widehat{\sigma}_1}^p}$ and $L_{\widehat{\sigma}_1}^p(t_{\mathrm{start}}^1) \in \left[ \frac{ \mu_{\widehat{\sigma}_1} }{3}, \frac{\mu_{\widehat{\sigma}_1}}{2}  \right] $. We can conclude that the Zero Tests performed by each of the players $p \in [M]$ are successful in recovering player $1$'s transmission over the pulls of arm $\widehat{\sigma}_1$ with probability at least $1-\frac{\delta}{4K^2M}$ each. A union bound over the $M$ players yields the result.
\end{proof}

This finalizes the formal proof of Lemma~\ref{lemma::zero_test_works_main}.

\subsection{Detailed Discussion and Missing Supporting Results for The Listening Players}\label{section::listening_players_discussion_supporting_results}

In this section we present the full versions of the algorithms used by the listening players to 1) prepare and start listening (see Algorithm~\ref{algorithm::prepare_and_start_listening}) and 2) decode player 1's message (see Algorithm~\ref{algorithm::decode}). We also present the proof of Lemma~\ref{lemma::listen_agrees_comm}, which we restate for readability.

\begin{algorithm}[H]
\textbf{Input} Players $p \in \{ 2, \cdots, K\}$\\
\textbf{Initialize } $\mathrm{FLAG}\leftarrow \mathrm{NONE}$\\
    \For{ \emph{special rounds} $s = 1, \cdots $ }{
    Pull arm $K-p$ (Round Robin schedule) %
    \If{$\mathbf{conn}^p\left(sK, 10\right) \geq 2$}{

    $\mathrm{FLAG} \leftarrow \mathrm{FINDPOWER}$
   }
   \If{$\mathrm{FLAG}=\mathrm{FINDPOWER}$,  $\left\lfloor \frac{s}{g(s)}\right\rfloor= 9^w$ for some $w \in \mathbb{N}$ and $\left\lfloor \frac{s-1}{g(s-1)} \right\rfloor\neq 9^w$  }{
    $t_{\mathrm{listen}}^p \leftarrow Ks$\\
    $A^p \leftarrow \{ L_i^p(t_{\mathrm{listen}}^p)\}_{i \in [K]}$, $B^p \leftarrow  \{ \widehat{\mu}_i^p(t_{\mathrm{listen}}^p)\}_{i \in [K]}$\\
    Start listening for a communication start signal from player $1$ \\
    $\mathrm{FLAG}\leftarrow \mathrm{LISTENCOMM1}$ \\
   }
   \ElseIf{ $\mathrm{FLAG} = \mathrm{LISTENCOMM1}$ and $t = t_{\mathrm{listen}}^p + Kf(t_{\mathrm{listen}}^p)$ }{
   \If{$\mathrm{ZeroTest}(A^p,B^p, C^p) = 0$}{
   $C^p \leftarrow \{ \widehat{\mu}_i^p(t_{\mathrm{listen}}^p+1:t_{\mathrm{listen}}^p+Kf(t_{\mathrm{listen}}^p) )\}_{i \in [K]} $ \\
   The algorithm did not detect a communication start signal:\\
   $\mathrm{FLAG} \leftarrow \mathrm{FINDPOWER}$ \\

   }\Else{
   $\mathbf{DecodedMessage} \leftarrow \mathrm{DECODE}(t^p_{\mathrm{listen}}, A^p, B^p, C^p, \text{message size} = K)$ using Algorithm~\ref{algorithm::decode}. 
   
   }
   
   } 
  
    }
\caption{Prepare and Start Listening  ($p \in \{2,\cdots, M\}$)}
\label{algorithm::prepare_and_start_listening}
\end{algorithm}

\begin{algorithm}[H]
\textbf{Input} Round number $t_{\mathrm{listen}}^p$,   witnesses $\{ L_i^p(t_{\mathrm{listen}}^p)\}_{i \in [K]}$, empirical means $\{ \widehat{\mu}_i^p(t_{\mathrm{listen}}^p)\}_{i \in [K]}$, message size $\alpha$ \\
 $A^p \leftarrow \{ L_i^p(t_{\mathrm{listen}}^p)\}_{i \in [K]}$, $B^p \leftarrow \{ \widehat{\mu}_i^p(t_{\mathrm{listen}}^p)\}_{i \in [K]}$\\
\For{$t = t_{\mathrm{listen}}^p+1, \cdots, t_{\mathrm{listen}}^p + \alpha K f(t_{\mathrm{listen}}^p ) $}{
    \If{$t = t_{\mathrm{listen}}^p + j Kf(t_{\mathrm{listen}})$}{

       $C^p \leftarrow \{ \widehat{\mu}_i^p(t_{\mathrm{listen}}^p+(j-1)K f(t_{\mathrm{listen}}^p) + 1:t_{\mathrm{listen}}^p+jKf(t_{\mathrm{listen}}^p) )\}_{i \in [K]}$ \\
       
       $\mathbf{DecodedMessage}_j  \leftarrow \mathrm{ZeroTest}(A^p, B^p, C^p)$ \\

    }

Return $\mathbf{DecodedMessage}$.
}
    
\caption{DECODE }
\label{algorithm::decode}
\end{algorithm}

\multibitrecovery*

\begin{proof}
Let $9^u$ be the unique power of nine in the interval $\left[\frac{128}{\max_{i} \Delta_{\sigma_i, \sigma_{i+1}}^2}   , \frac{1152}{\max_{i} \Delta_{\sigma_i, \sigma_{i+1}}^2}    \right)$. Recall that whenever the good event $\mathcal{E}$ holds by design the \underline{first} $t_{\mathrm{listen}}^p$ of Algorithm~\ref{algorithm::prepare_and_start_listening} satisfies $$t_{\mathrm{listen}}^p \in \left\{ \min_{t \in \mathbb{N}} \text{ s.t. } \left\lfloor \frac{t/K}{g(t/K)}\right\rfloor = 9^u,  \min_{t \in \mathbb{N}} \text{ s.t. } \left\lfloor \frac{t/K}{g(t/K)}\right\rfloor = 9^{u+1} \right\}.$$ Similarly recall that whenever the good event $\mathcal{E}$ holds $$t_{\mathrm{comm1}}^1 \in \left\{ \min_{t \in \mathbb{N}} \text{ s.t. } \left\lfloor \frac{t/K}{g(t/K)}\right\rfloor = 9^{u+1},  \min_{t \in \mathbb{N}} \text{ s.t. } \left\lfloor \frac{t/K}{g(t/K)}\right\rfloor = 9^{u+2} \right\}.$$ 

This means that each player $p \in \{2,\cdots, M\}$ may require at most three invocations to the $\mathrm{ZeroTest}$ function to detect the $b=1$ bit that player $1$ will use to signal the start of the communication sequence. Applying Lemma~\ref{lemma::zero_test_works_main} with $t_{\mathrm{start}}^1$ equals the different $t_{\mathrm{listen}}^p$ guesses of the listening players and over the $K$ bits transmitted by player $1$, we see that a union bound over at most $K+3$ uses of Lemma~\ref{lemma::zero_test_works_main} are required. Since $K+3 \leq 4K$ the result follows. 
\end{proof}

\subsection{Detailed Discussion and Missing Supporting Results for Bounding Regret}\label{section::bounding_regret_discussion_supporting_results}

In  this section we present the missing detailed proofs of Section~\ref{section::bounding_regret_main} in the main. The discussion is divided in two parts. First we analyze the $\mathrm{FirstPartitionRegret}([K], [M])$ derived from a single communication and listening round required to transmit $\mathcal{C}_1^1(t_{\mathrm{first}}^1, 10)$ (see Section~\ref{section::bounding_first_partition_regret}). The next section~\ref{section::analyzing_RECURSE_function} deals with the $\mathrm{RECURSE}$ function and with assembling the final algorithm.

\subsubsection{Bounding the First Partition Regret}\label{section::bounding_first_partition_regret}

We have now the necessary ingredients to characterize the regret of the strategy to communicate the composition of  $\mathcal{C}_1^1(t_{\mathrm{first}}^1, 10 ) $ from player $1$ to all other players ($\mathrm{FirstPartitionRegret}([K], [M])$) Observe that regret is generated only when collisions occur. During the communication interaction between player $1$ and any other single player $p \in \{2, \cdots, M\}$, the number of collisions is upper bounded by $(K+1)f(t_{\mathrm{comm1}}^1)$. Thus the total $\mathrm{CollisionRegret}$ experienced by the $M$ players during the first communication round (when player $1$ informs players $\{ 2, \cdots , M\}$ of the partition resulting from the first time $\mathcal{G}_{t_{\mathrm{first}}^1}^1( 10 )$ has more than one connected component) is upper bounded by $M(K+1)f(t_{\mathrm{comm1}}^1)$. Let's prove an upper bound for $f(t_{\mathrm{comm1}}^1)$.

\begin{restatable}{lemma}{lemmaboundingfcommone}\label{lemma::bounding_collision_regret_unit2}
If $\mathcal{E}$ holds, $t_{\mathrm{comm1}} \geq \max\left( t_{\mathrm{boundary1}}, t_\mathrm{boundary3}\right) $, $s_{\mathrm{first}}^1 \geq s_{\mathrm{boundary2}}$ and $\delta \leq \frac{1}{162}$ then,
\begin{equation}\label{equation::upper_bounding_fcomm}
    f(t_{\mathrm{comm1}}^1) \leq  \frac{20736 B\left(\frac{186624}{\max_{i} \Delta_{\sigma_i, \sigma_{i+1}}^2}, \frac{\delta}{4K^2M}\right) }{\max_{i} \Delta_{\sigma_i, \sigma_{i+1}}}.
\end{equation}
\end{restatable}

The proof of Lemma~\ref{lemma::bounding_collision_regret_unit2} can be found in Appendix~\ref{section::proof_bounding_collision_regret_unit2}. We can now combine Lemmas~\ref{lemma::listen_agrees_comm} and~\ref{lemma::bounding_collision_regret_unit2} to bound the regret incurred by Algorithms~\ref{algorithm::prepare_start_player1_communicate},~\ref{algorithm::player_1_communicate},~\ref{algorithm::prepare_and_start_listening} and~\ref{algorithm::decode} during the transmission of the partition message $\mathrm{ENCODE}(\mathcal{C}_1^1(t_{\mathrm{first}}^1, 10))$. 

\begin{corollary}[First Partition Collision Regret]\label{corollary::collision_regret}
If $\mathcal{E}$ holds and  $\delta \leq \frac{1}{162}$ then with probability at least $1-\frac{\delta}{K}$ the total collision regret (the regret generated by collisions occurring during the communication rounds used to communicate the composition of  $\mathcal{C}_1^1(t_{\mathrm{first}}^1, 10 ) $ satisfies,

\begin{equation*}
     \mathrm{CollisionRegret}([K], [M]) \leq \frac{20736(K+1)M B\left(\frac{186624}{\max_{i} \Delta_{\sigma_i, \sigma_{i+1}}^2}, \frac{\delta}{4K^2M}\right) }{\max_{i} \Delta_{\sigma_i, \sigma_{i+1}}}  + \tilde{c}(\delta, K, M),
\end{equation*}

where $\tilde{c}(\delta, K, M)$ is a logarithmic problem independent cost resulting from the regret incurred\footnote{A slightly more careful algorithm that uses an estimator of $\mu_{\widehat{\sigma}_1}$ as the input to determine the length of the one bit communication rounds yields a regret bound of the form $\mathrm{CollisionRegret}([K], [M]) = \mathcal{O}\left( \frac{KM\log\left( t/\delta \right)}{\mu_{\sigma_1}}\right)$. Since this quantity would be dominated by the $\mathrm{RoundRobinRegret}([K], [M])$ it wouldn't change the final result.  } before the boundary conditions $t_{\mathrm{comm1}} \geq \max\left( t_{\mathrm{boundary1}}, t_\mathrm{boundary3}\right) $, $s_{\mathrm{first}}^1 \geq s_{\mathrm{boundary2}}$  hold\footnote{We will provide a bound for this quantity in the following section. }.

\end{corollary}

We can also get a bound for the $\mathrm{RoundRobinRegret}$. By Lemma~\ref{lemma::upper_bounding_s_comm_1} whenever $\mathcal{E}$ holds, $s_{\mathrm{comm1}}^1 \leq \frac{746496 }{\max_{i} \Delta_{\sigma_i, \sigma_{i+1}}^2} \log\left(  \frac{746496 MK }{\delta\max_{i} \Delta_{\sigma_i, \sigma_{i+1}}^2}  \right)$. During a single Round Robin cycle regret is only incurred when arms in $\{ \mu_{\sigma_{i}}\}_{i=M+1}^K$ are played. A full cycle consists of $KM$ arm pulls, all players pull each arm once. Out of these $KM$ pulls the $M^2$ pulls of arms $\mu_{\sigma_1}, \cdots, \mu_{\sigma_{M}}$ incurr in no regret. The remaining $(K-M)M$ pulls incur in a regret of
\begin{equation*}
(K-M)\left(  \sum_{i=1}^M \mu_{\sigma_i} \right)    - M  \left( \sum_{i=M+1}^K  \mu_{\sigma_i} \right) 
\end{equation*}

We can further upper bound this quantity as follows,

\begin{align*}
    (K-M)\left(  \sum_{i=1}^M \mu_{\sigma_i} \right)    - M  \left( \sum_{i=M+1}^K  \mu_{\sigma_i} \right)  &= \sum_{i=1}^M \sum_{j=M+1}^K \Delta_{\sigma_i, \sigma_j} \\
    &\leq M(K-M)K \max_{i } \Delta_{\sigma_i, \sigma_{i+1}}
\end{align*}

Where we have used the bound $\Delta_{\sigma_{i_1}, \sigma_{i_2}} \leq K \max_{i} \Delta_{\sigma_i, \sigma_{i+1}}$ for all $i_1 < i_2$. Thus the $\mathrm{RoundRobinRegret}$ incurred by the algorithm in the rounds preceding active communication can be upper bounded by 

\begin{equation*}
     M(K-M)K \max_{i} \Delta_{\sigma_i, \sigma_{i+1}} s_{\mathrm{comm1}}^1.
\end{equation*}

Recall that during the communication rounds, all players $p \in \{2,\cdots, M\}$ are still using a Round Robin schedule. This goes on for $(K+1)Kf(t_{\mathrm{comm1}}^1)$  rounds after $t_{\mathrm{comm1}}^1$, thus completing a total of $(K+1)f(t_{\mathrm{comm1}}^1)$ Round Robin cycles. Using Equation~\ref{equation::upper_bounding_fcomm} from Lemma~\ref{lemma::bounding_collision_regret_unit2} we can upper bound the Round Robin regret incurred during these rounds as 
\begin{align*}
    M(K-M)K \max_{i } \Delta_{\sigma_i, \sigma_{i+1}} &\times  (K+1)f(t_{\mathrm{comm1}}^1) \\
    &\leq 20736(K+1)M(K-M)K  B\left(\frac{186624}{\max_{i} \Delta_{\sigma_i, \sigma_{i+1}}^2}, \frac{\delta}{4K^2M}\right).
\end{align*}

These observations imply the following upper bound for $\mathrm{RoundRobinRegret}$,

\begin{corollary}[First Partition Round Robin Regret]\label{corollary::roundrobin_regret}
If $\mathcal{E}$ holds and  $\delta \leq \frac{1}{162}$ then with probability at least $1-\frac{\delta}{K}$ the Round Robin regret is bounded as follows:
\begin{align*}
    \mathrm{RoundRobinRegret}([K], [M]) &\leq \frac{746496 M(K-M)K }{\max_{i} \Delta_{\sigma_i, \sigma_{i+1}}} \log\left(  \frac{746496 MK }{\delta\max_{i} \Delta_{\sigma_i, \sigma_{i+1}}^2}  \right) \\
    &\quad 20736(K+1)M(K-M)K  B\left(\frac{186624}{\max_{i} \Delta_{\sigma_i, \sigma_{i+1}}^2}, \frac{\delta}{4K^2M}\right)+\\
    &\quad \tilde{c}(\delta, K, M),
\end{align*}
where $\tilde{c}(\delta, K, M)$ is a logarithmic problem independent cost resulting from the regret incurred before the boundary conditions $t_{\mathrm{comm1}} \geq \max\left( t_{\mathrm{boundary1}}, t_\mathrm{boundary3}\right) $, $s_{\mathrm{first}}^1 \geq s_{\mathrm{boundary2}}$  hold
and is the same as in Corollary~\ref{corollary::collision_regret}.

\end{corollary}

Combining Corollaries~\ref{corollary::collision_regret} and~\ref{corollary::roundrobin_regret} we can infer that if $\mathcal{E}$ holds and  $\delta \leq \frac{1}{162}$ then with probability at least $1-\frac{\delta}{K}$ the total regret incurred up to time $t_{\mathrm{comm1}}^1 + (K+1)f(t_{\mathrm{comm1}}^1)$, when all players are aware of the composition of $\mathcal{C}_1^1(t_{\mathrm{first}}^1, 10)$ is upper bounded by

\begin{align}
    \mathrm{FirstPartitionRegret}([K], [M]) &\leq \frac{746496 M(K-M)K }{\max_{i} \Delta_{\sigma_i, \sigma_{i+1}}} \log\left(  \frac{746496 MK }{\delta\max_{i} \Delta_{\sigma_i, \sigma_{i+1}}^2}  \right) + \notag \\
    &\quad 20736(K+1)M(K-M)K  B\left(\frac{186624}{\max_{i} \Delta_{\sigma_i, \sigma_{i+1}}^2}, \frac{\delta}{4K^2M}\right)+ \notag \\
    &\quad \frac{20736(K+1)M B\left(\frac{186624}{\max_{i} \Delta_{\sigma_i, \sigma_{i+1}}^2}, \frac{\delta}{4K^2M}\right) }{\max_{i} \Delta_{\sigma_i, \sigma_{i+1}}}+ \tilde{c}(\delta, K, M), \label{equation::first_partition_regret}
\end{align}

Where the term $\tilde{c}(\delta, K, M)$ captures a crude linear upper bound on the regret collected before the boundary conditions hold true. We can also invoque the results in Lemma~\ref{lemma::upper_bounding_s_comm_1} and~\ref{lemma::bounding_collision_regret_unit2} to bound on the total number of rounds needed until all players are aware of the composition of $\mathcal{C}_1^1(t_{\mathrm{first}}^1, 10)$,

\begin{align*}
    \mathrm{Runtime}([K], [M]) &\leq \frac{746496 K }{\max_{i} \Delta_{\sigma_i, \sigma_{i+1}}^2} \log\left(  \frac{746496 MK }{\delta\max_{i} \Delta_{\sigma_i, \sigma_{i+1}}^2}  \right)  + \\
    &\quad \frac{20736(K+1)K B\left(\frac{186624}{\max_{i} \Delta_{\sigma_i, \sigma_{i+1}}^2}, \frac{\delta}{4K^2M}\right) }{\max_{i} \Delta_{\sigma_i, \sigma_{i+1}}}.
\end{align*}

\paragraph{Bounding $\tilde{c}(\delta, K, M)$} The cost of satisfying the boundary conditions bound is not additive between the $\mathrm{RoundRobinRegret}$ and the $\mathrm{CollisionRegret}$ components of $\mathrm{FirstPartititionRegret}$. In order to deal with these we introduce a slight modification to the communication and listening protocols of Algorithms~\ref{algorithm::prepare_start_player1_communicate},~\ref{algorithm::player_1_communicate},~\ref{algorithm::prepare_and_start_listening} and~\ref{algorithm::decode} by modifying the definition of $t_{\mathrm{first}}^p$ for all $p \in [M]$. Instead we use $\tilde{t}_{\mathrm{first}}^p$ times defined as $\tilde{t}_{\mathrm{first}}^p = \max( t_{\mathrm{first}}^p, t_{\mathrm{firstBoundary}})$. It is easy to see this will not affect the regret too much. If $t_{\mathrm{first}}^p \geq t_{\mathrm{firstBoundary}}$ for all $p \in [M]$, or $t_{\mathrm{firstBoundary}} \geq t_{\mathrm{first}}^p$ for some $p \in [M]$ but not for all, the analysis will remain unchanged. If instead $t_{\mathrm{firstBoundary}} > t_{\mathrm{first}}^p$ for all $p \in [M]$, all the player's $\tilde{t}_{\mathrm{first}}^p  = t_{\mathrm{firstBoundary}}$. This definition induces that of $\tilde{s}_{\mathrm{first}}^p$ $\forall p \in [M]$, $\tilde{t}^1_{\mathrm{comm1}}$ and  $\tilde{s}^1_{\mathrm{comm1}}$. As a consequence of Lemma~\ref{lemma::upper_bounding_s_comm_1} we see that  $\tilde{s}_{\mathrm{comm1}}^1 \leq 162 \tilde{s}_{\mathrm{first}}^1 = 162\frac{t_{\mathrm{firstBoundary}}}{K}$. Notice that by definition of the $4-$th boundary condition $f(\tilde{t}_{\mathrm{comm1}}^1) \leq \frac{\tilde{t}_{\mathrm{comm1}}^1/K}{g(\tilde{t}_{\mathrm{comm1}}^1/K)} \leq \tilde{t}_{\mathrm{comm1}}^1/K \leq 162 t_{\mathrm{firstBoundary}}/K$.

The protocol thus ensures communicating the composition of $\mathcal{C}_1^p(t_{\mathrm{first}}^1, 10)$ (notice that we are still transmitting the composition of $\mathcal{C}_1^p(t_{\mathrm{first}}^1, 10)$ and not $\mathcal{C}_1^p(\tilde{t}_{\mathrm{first}}^1, 10)$ ) can be achieved while incurring regret of at most,
\begin{equation*}
M(K-M)K\max_i \Delta_{\sigma_i, \sigma_{i+1}} \left( \tilde{s}_{\mathrm{comm1}} + (K+1)f(\tilde{t}_{\mathrm{comm1}}^1) \right) + M(K+1)f(\tilde{t}_{\mathrm{comm1}}^1).  
\end{equation*}
We define $\tilde{c}(\delta, K, M) $ to be a problem independent upper bound of this quantity. 
\begin{align*}
    \tilde{c}(\delta, K, M) &= M(K-M)K \left( \frac{162 t_{\mathrm{firstBoundary}}}{K}  + (K+1) \frac{162 t_{\mathrm{firstBoundary}}}{K}  \right. \\
    &\quad \left. + M(K+1)\frac{162 t_{\mathrm{firstBoundary}}}{K} \right)\\
    &=\mathbf{poly}\left(\log\left(\frac{1}{\delta}\right), K, M \right)
\end{align*}
And where the dependence on $\log\left(\frac{1}{\delta}\right)$ is linear.

\subsubsection{Analyzing the $\mathrm{RECURSE}$ Function}\label{section::analyzing_RECURSE_function}

 In order to implement the recursion strategy described at the start of Section~\ref{section::algorithm} and at the end of Section~\ref{section::bounding_regret_main}, when faced with a smaller Cooperative Multi-Player Multi-Armed problem the players will restart their empirical mean estimators from scratch. In Appendix~\ref{appendix::complex_restart} we describe a warm-start strategy that allows the players to start their empirical mean estimators using a constant proportion of the samples that have been gathered so far. The two strategies have the same performance up to constant factors.

\begin{algorithm}[H]
\textbf{Input} players $p \in \{1, \cdots, M\}$, arm indices $\{1, \cdots, K\}$ \\

Run Algorithms~\ref{algorithm::prepare_start_player1_communicate},~\ref{algorithm::player_1_communicate}, ~\ref{algorithm::prepare_and_start_listening}, and~\ref{algorithm::decode} to communicate $\mathcal{C}_1^1(t_{\mathrm{first}}^1, 10)$ to all players. \\

\If{ $|\mathcal{C}_1^1(t_{\mathrm{first}}^1, 10)| > M$}{
$\mathrm{RECURSE}$ on $\mathcal{C}_1^1(t_{\mathrm{first}}^1, 10)$ with players $[M]$.
}
\Else{
Run $\mathrm{RoundRobin}$ on $\mathcal{C}_1^1(t_{\mathrm{first}}^1, 10)$ with players $\{ 1, \cdots, |\mathcal{C}_1^1(t_{\mathrm{first}}^1, 10)|\}$.\\
$\mathrm{RECURSE}$ on $[K] \backslash \mathcal{C}_1^1(t_{\mathrm{first}}^1, 10)$ with players $[M]\backslash \{ 1, \cdots, |\mathcal{C}_1^1(t_{\mathrm{first}}^1, 10)|\}$.
} 
\caption{$\mathrm{RECURSE}$}
\label{algorithm::recurse}
\end{algorithm} 

To analyze the regret guarantees of Algorithm~\ref{algorithm::recurse} let's start by noting that in case $|\mathcal{C}_1^1(t_{\mathrm{first}}^1, 10)| \leq < M$, the subset of players $\{ 1, \cdots, |\mathcal{C}_1^1(t_{\mathrm{first}}^1, 10)|\}$ that has been assigned to $\mathrm{RoundRobin}$ over the subset $\mathcal{C}_1^1(t_{\mathrm{first}}^1, 10)$ will not incur in any more regret. It is therefore only necessary to bound the regret incurred by the algorithm during each of its successive calls to the $\mathrm{RECURSE}$ subroutine. 

The main difficulty we face in deriving an upper bound for the regret of $\mathrm{RECURSE}$ is that a successful execution of the communication protocol may not imply the gap between the arms in $\mathcal{C}_1^1(t_{\mathrm{first}}^1, 10)$ and those in $[K]\backslash \mathcal{C}_1^1(t_{\mathrm{first}}^1, 10)$ equals $\max_i \Delta_{\sigma_i, \sigma_{i+1}}$. Nevertheless we can show they cannot be more than a constant multiple fraction apart,

\begin{lemma}\label{lemma::upper_bound_3}
In the event the communication protocol succeeds in transmitting the composition of $\mathcal{C}_1^1(t_{\mathrm{first}}^1, 10)$ to all players $p \in [P]$. If $\tilde{\Delta} = \min_{\sigma_i \in \mathcal{C}_1^1(t_{\mathrm{first}}^1, 10)} \mu_{\sigma_i} - \max_{\sigma_j \in [K] \backslash \mathcal{C}_1^1(t_{\mathrm{first}}^1, 10)} \mu_{\sigma_j}$. Then,
\begin{equation*}
  \max_{i} \Delta_{\sigma_i, \sigma_{i+1}} \leq     3 \tilde{\Delta}.
\end{equation*}
\end{lemma}

\begin{proof}
If the communication protocol succeeded, then the sandwich property of Equation~\ref{equation::elimination_consecuence} holds for all $\sigma_i, \sigma_j$ and therefore,
\begin{equation*}
    \frac{s_{\mathrm{first}}^1}{g(s_{\mathrm{first}}^1)} \in \left[ \frac{128}{\tilde{\Delta}^2} ,  \frac{1152}{\tilde{\Delta}^2}\right]  
\end{equation*}
Since the `trigger' time $\bar{s}_{\mathrm{first}}^1$ for $\max_i \Delta_{\sigma_i, \sigma_{i+1}}$ satisfies,
\begin{equation*}
    \frac{\bar{s}_\mathrm{first} }{g(\bar{s}_\mathrm{first})} \in  \left[ \frac{128}{\max_{i} \Delta_{\sigma_i, \sigma_{i+1}}^2} ,  \frac{1152 }{\max_{i} \Delta_{\sigma_i, \sigma_{i+1}}^2}\right]
\end{equation*}
And by definition $s_{\mathrm{first}}^1\leq \bar{s}_{\mathrm{first}}^1$, we can conclude that $\frac{128}{\tilde{\Delta}^2} \leq \frac{1152 }{\max_{i} \Delta_{\sigma_i, \sigma_{i+1}}^2}$ and therefore,
\begin{equation*}
    \max_{i} \Delta_{\sigma_i, \sigma_{i+1}} \leq 3\tilde{\Delta}.
\end{equation*}\end{proof}

Let's consider the set of consecutive gaps $\{ \Delta_{\sigma_i, \sigma_{i+1}}\}_{i=1}^{K-1}$ and assume their ordering to be $\bar{\Delta}_1 \geq \cdots \geq \bar{\Delta}_{K-1}$ where $\bar{\Delta}_i = \Delta_{\sigma_{\ell(i)}, \sigma_{\ell(i)+1}}$ for some bijective mapping $\ell : [K-1] \rightarrow [K-1] $. The inverse mapping $\ell^{-1}(i)$ satisfies, $\Delta_{\sigma_i, \sigma_{i+1}} = \bar{\Delta}_{\ell^{-1}(i)}$.

For all $i \in [K]$ denote by $J_{\mathrm{up}}(i)$ to the index,
\begin{equation*}
    J_{\mathrm{up}}(i) = \argmax \{ j \text{ s.t. } 3\bar{\Delta}_j \geq \bar{\Delta}_i \}.
\end{equation*}

By definition $J_{\mathrm{up}}(i) \geq i$. As an immediate consequence of Lemma~\ref{lemma::upper_bound_3}, the first sub-problem the $\mathrm{RECURSE}$ algorithm will solve (in case the communication protocol was successful) will break the arm set through one of the gaps in $\{\Delta_{\sigma_{\ell(i)}, \sigma_{\ell(i)+1}}, i \leq J_{\mathrm{up}}(1)\}$. Furthermore Lemma~\ref{lemma::upper_bound_3} also implies the $\mathrm{RECURSE}$ algorithm will break the $\bar{\Delta}_1 = \Delta_{\sigma_{\ell(1))}, \sigma_{\ell(1)+1}}$ gap in at most $J_{\mathrm{up}}(1)$ recursive calls. In fact the same argument holds for all $i \in [K-1]$. The $\mathrm{RECURSE}$ algorithm will break the $\bar{\Delta}_i = \Delta_{\sigma_{\ell(i)}, \sigma_{\ell(i)+1}}$ gap in at most $J_\mathrm{up}(i)$ recursive calls.

After the $\mathrm{RECURSE}$ algorithm has successfully broken the $\Delta_{\sigma_M, \sigma_{M+1}}$ gap, the players will cease to experience any regret. As a result of the previous discussion this will happen in at most $J_{\mathrm{up}}(\ell^{-1}(M))$ recursive calls. We are ready to bound the regret of the $\mathrm{RECURSE}$ Algorithm.

For any $\Delta >0 $ we define the partition regret function for gap $\Delta$, number of arms $\bar{K}$ and number of players $\bar{M}$ as,

\begin{align*}
    \mathrm{PartitionRegret}(\Delta, \bar{K}, \bar{M}, \delta) &= \frac{746496 \bar{M}(\bar{K}-\bar{M})\bar{K} }{ \Delta } \log\left(  \frac{746496 \bar{M}\bar{K} }{\delta \Delta^2}  \right) + \\
    &\quad 20736(\bar{K}^2+\bar{K})\bar{M}(\bar{K}-\bar{M})\  B\left(\frac{186624}{\Delta^2}, \frac{\delta}{4\bar{K}^2\bar{M}}\right)+\\
    &\quad \frac{20736(\bar{K}+1)\bar{M} B\left(\frac{186624}{ \Delta^2}, \frac{\delta}{4\bar{K}^2\bar{M}}\right) }{ \Delta}+ \tilde{c}(\delta, \bar{K}, \bar{M}) 
\end{align*}

This is the parametric form of the upper bound in Equation~\ref{equation::first_partition_regret}. Recall that for any sub-problem of $\bar{K}$ arms the communication protocol succeeds with probability at least $\underbrace{1-\delta}_{\text{satisfying } \mathcal{E}} - \frac{\delta}{\bar{K}} $ (see the discussion surrounding Equation~\ref{equation::first_partition_regret}). This concludes the proof of one of our main results. 

\begin{theorem}
If $\delta \leq \frac{1}{162}$ with probability $1-\delta \left( J_{\mathrm{up}}(\ell^{-1}(M)) + \sum_{i=1}^{J_{\mathrm{up}}(\ell^{-1}(M))} \frac{1}{i} \right)$ the regret of Algorithm~\ref{algorithm::recurse} satisfies
\begin{align*}
    \mathrm{RegretRECURSE}([K], [M]) \leq \sum_{i=1}^{J_{\mathrm{up}}(\ell^{-1}(M))}  \mathrm{PartitionRegret}(\bar{\Delta}_i, \ell(i), M, \delta)
\end{align*}

\end{theorem}

Using the definition $\delta = \frac{\xi}{2K}$ and  the fact that $\mathrm{PartitionRegret}$ is monotonic w.r.t. $\frac{1}{\Delta}$, $\bar{K}$ and $\mathrm{M}$ as well as the inequalities $\ell(i) \leq K$ and $J_{\mathrm{up}}(\ell^{-1}(M)) \leq M$ and $3\Delta_{J_{\mathrm{up}}(\ell^{-1}(M))} \geq $ we can conclude the following,

\begin{corollary}\label{corollary::main_corollary}
If $\frac{\xi}{2K} \leq \frac{1}{162}$ then with probability at least $1-\xi$ the regret of Algorithm~\ref{algorithm::recurse} satisfies 
\begin{align*}
     \mathrm{RegretRECURSE}([K], [M]) &\leq K \cdot \mathrm{PartitionRegret}\left(\frac{\Delta_{\sigma_M, \sigma_{M+1}}}{3}, K, M, \frac{\xi}{2K} \right) \\
     &= \frac{3\times746496 M(K-M)K^2 }{ \Delta_{\sigma_M, \sigma_{M+1}} } \log\left(  \frac{18\times746496 MK^2 }{\xi \Delta_{\sigma_M, \sigma_{M+1}}^2}  \right) + \\
    &\quad 20736(K^3+K^2)M(K-M)  B\left(\frac{9\times 186624}{\Delta_{\sigma_M, \sigma_{M+1}}^2}, \frac{\xi}{8K^3M}\right)+\\
    &\quad \frac{20736(K^2+K)M B\left(\frac{9 \times 186624}{ \Delta_{\sigma_M, \sigma_{M+1}}^2}, \frac{\xi}{8K^3M}\right) }{ \Delta_{\sigma_M, \sigma_{M+1}}}+ K \tilde{c}\left(\frac{\xi}{2K}, K, M\right) 
\end{align*}
\end{corollary}

Note that both $g(n)$ and $B(n, \delta)$ are of the order of $\widetilde{\mathcal{O}}( \log(n/\delta) )$ where $\widetilde{\mathcal{O}}(\cdot)$ hides logarithmic factors in $K$ and $M$ only. By setting $\xi = \min\left(\frac{1}{T}, \frac{K}{81}\right )$ we can easily turn the results of Corollary~\ref{corollary::main_corollary} into the following Corollary that corresponds to the statement of Theorem~\ref{theorem::main},

\begin{corollary}[Main-Simplified]\label{corollary::main_simplified_bottom}
There exists a strategy such that the regret is upper bounded by:
\begin{equation*}
      \mathcal{R}_T \leq \widetilde{ \mathcal{O}}\left(\frac{M(K-M)K^2\log(T)}{\Delta_{\sigma_M, \sigma_{M+1}}} + \mathbf{poly}(\log(T), K, M ) \right) ,
\end{equation*}
with probability at least $1-\min\left(\frac{1}{T},\frac{K}{81} \right)$ where $\widetilde{ \mathcal{O}}(\cdot)$ hides factors logarithmic in $M$ and $K$ \textbf{only}\footnote{More careful analysis may be possible that could ameliorate the dependence on the number of arms by at least one factor of $K$. This may be achieved by not upper bounding the $\mathrm{RoundRobinRegret}$ and $\mathrm{PartitionRegret}$ in each partition by those of the partition defined by $\Delta_{\sigma_{M}, \sigma_{M+1}}$. We leave this sharpening for future work. }.

\end{corollary}

Our results also imply anytime guarantees,

\begin{corollary}[Main-Simplified Anytime]\label{corollary::main_simplified_bottom_anytime}
Let $\delta \in (0,1)$. There exists a strategy such that the regret is upper bounded by:
\begin{equation*}
      \mathcal{R}_t \leq \widetilde{ \mathcal{O}}\left(\frac{M(K-M)K^2\log(t/\delta)}{\Delta_{\sigma_M, \sigma_{M+1}}} + \mathbf{poly}(\log(t/\delta), K, M ) \right) ,
\end{equation*}
with probability at least $1-\delta$ for all $t \in \mathbb{N}$ and where $\widetilde{ \mathcal{O}}(\cdot)$ hides factors logarithmic in $M$ and $K$ \textbf{only}. 
\end{corollary}

\section{Sharpening of the Zero Collision Reward Setting}\label{section::extensions}

In this section we describe a couple of Extensions of our main results.

\subsection{Problem independent Collision Regret}\label{section::problem_independent_collision_regret}

Recall the the logic behind why the $\mathrm{ZeroTest}$ works. At time $t_{\mathrm{test}}^p$ (in our case equal to $t_{\mathrm{comm1}}^1$ all players have access to a constant accuracy estimator for  $\max_i \Delta_{\sigma_i, \sigma_{i+1}}$ (i.e. the empirical gap between the arms at the boundary of the two connected components of $\mathcal{C}_1^1(t_{\mathrm{first}}^1, 10)$). Therefore they can estimate $\mu_{\sigma_1}$ up to a constant accuracy. By Freedman's inequality, testing at an accuracy of $c\mu_{\sigma_1}$ the mean of an arm $\widehat{\sigma}_1$ satisfying $\mu_{\widehat{\sigma}_1} > c' \mu_{\sigma_1}$ with $c' > c$ only requires $\widetilde{\mathcal{O}}\left( \frac{1}{\mu_{\sigma_1}}\right)$ samples (up to log factors) since the variance of $\widehat{\sigma}_1$ is upper bounded by $\mu_{\sigma_1}$. 

Thus, it is enough for the $\mathrm{ZeroTest}$ to succeed to use $\widetilde{\mathcal{O}}\left( \frac{\mathrm{poly}(K,M)\log(1/\delta)}{\mu_{\sigma_1}}\right)$ collisions to transmit each bit. Since each collision incurs in regret of at most $\mu_{\sigma_1}$, this implies the $\mathrm{CollisionRegret}$ can be upper bounded by a problem independent term of the form $\mathcal{O}(\mathrm{poly}(K,M) \log(T))$.

This can be achieved by substituting the communication length function $f$ with an agreed estimator $t_{\mathrm{communication-length}}$ of order $\approx 1/\mu_{\sigma_1}$. This can be agreed upon by all players using the same initial procedure as in the $\mu_{\mathrm{collision}} > 0$ case, which would yield an estimator of $1/\mu_{\sigma_1}^2$ since $\Delta_{\mathrm{collision}} = \mu_{\sigma_1}$ in the zero collision reward setting.

\subsection{Unknown number of players}

Our algorithms work in the setting where each player has a known player index but does not know how many players may exist with a larger index. The $\mathrm{RECURSE}$ algorithm needs to be slightly modified. Instead of running $\mathrm{RoundRobin}$ on $\mathcal{C}_1^1(t_{\mathrm{first}}^1, 10)$ if $| \mathcal{C}_1^1(t_{\mathrm{first}}^1, 10)|\leq M$, the players will simply $\mathrm{RECURSE}$ and run the full communication protocol on the two sub-problems induced by $\mathcal{C}_1^1(t_{\mathrm{first}}^1, 10)$ and $[K]\backslash \mathcal{C}_1^1(t_{\mathrm{first}}^1, 10)$. Since all players are aware of their index, upon receiving  $\mathcal{C}_1^1(t_{\mathrm{first}}^1, 10)$ each of them can determine what sub-problem it is meant to play after a call to $\mathrm{RECURSE}$, either $\mathcal{C}_1^1(t_{\mathrm{first}}^1, 10)$ or the $[K]\backslash \mathcal{C}_1^1(t_{\mathrm{first}}^1, 10)$. Within each of the sub-problems each player is perfectly capable of inferring if it should be the communicating player or not. The techniques we have used to derive the regret guarantees of Corollary~\ref{corollary::main_simplified_bottom} can be used to derive the same instance dependent logarithmic regret rate for this slightly more complex algorithm.

\subsection{Unknown lower bound for the collision reward}\label{section::agreeing_t_collision_test}
Our algorithms work in the setting where the collision reward is a random variable with mean $\mu_{\mathrm{collision}} \in [0,1]$ satisfying the condition $\mu_{\mathrm{collision}}\leq \mu_{\sigma_K}$ and unknown to the learners in advance. Our algorithm will consist of an initial phase aimed at discovering an estimator for $\mu_{\sigma_1} - \mu_{\mathrm{collision}}$. This phase, not present in the vanilla version of the algorithm discussed in the previous sections will be executed at the start of any $\mathrm{RECURSE}$ subroutine. The objective of this discovery phase is to ensure player $1$ can convey the identity of a round value $t_{\mathrm{collision-test}}$ to all remaining players $p \in \{ 2, \cdots, M\}$ that will be used to determine the length of the communication rounds in the second phase of the algorithm.  During the second phase of the algorithm, the players will engage in a similar interaction as that described by Algorithms~\ref{algorithm::player_1_communicate},~\ref{algorithm::prepare_and_start_listening} and~\ref{algorithm::decode} where instead of using $f$ to infer the length of the communication rounds, the players will use $t_{\mathrm{collision-test}}$. Transmitting the identify of $t_{\mathrm{collision-test}}$ from player $1$ to all players $p \in \{ 2, \cdots, M\}$ is achieved via a bastardized version of the $\mathrm{ZeroTest}$ we outline below.
\begin{enumerate}
    \item All players go through a modified version of the $\mathrm{RoundRobin}$ schedule that we'll call $\mathrm{CollisionRoundRobin}$. Instead of cycling in batches of $K$ rounds, they will use a cycle length of size $K+M$. Let's see how the very first such cycle works. All the remaining ones are a repetition of this basic structure. During the first $K$ rounds all players cycle through the $K$ arms in $[K]$ following the usual $\mathrm{RoundRobin}$ schedule. From round $K+1$ to $K+M$, players $\{2, \cdots, M\}$ will continue pursuing a traditional $\mathrm{RoundRobin}$ schedule while player $1$ will instead pull arm $M$. All players $ p\in[M]$ will build their empirical estimators of $\mu_{1}, \cdots, \mu_K$ using only the samples collected during the first $K$ rounds of each $K+M$ $\mathrm{CollisionRoundRobin}$ cycle. To estimate $\mu_{\mathrm{collision}}$ player $1$ will use the samples collected at time $K+1$ of each $\mathrm{CollisionRoundRobin}$ cycle while players $p \in \{2, \cdots, M\}$ will use the samples collected at times $K+M-p+1$. All other samples will be discarded. If the number of players is unknown, we could extend the length of a $\mathrm{CollisionRoundRobin}$ cycle to be of size $2K$ instead. 
    \item Define $\widehat{\sigma}_1^p(t) =\argmax_{i \in [K]} \widehat{\mu}_k^p(t)$ be player $p$'s guess for the largest arm during round $t$. Define as $\widehat{\mu}_{\mathrm{collision}}^p(t)$ to be the empirical estimator of $\mu_\mathrm{collision}$ by player $p \in [M]$ at time $t$. Let $t_\mathrm{first-collision}^p$ be the first special round $t$ such that when $I^p_{\widehat{\sigma}_1^p(t)}(t, 10 ) \cap I^p_{\mathrm{collision}}(t, 10) = \emptyset$. The same logic used to prove Equations~\ref{equation::sandwich_condition_1} and~\ref{equation::condition_t_0_description_main} can be used to show that whenever $\mathcal{E}$ holds and for all $p \in [M]$,
 \begin{equation}\label{equation::support_equation_collision_estimation}
    \frac{128}{\left(\mu_{\sigma_1} - \mu_{\mathrm{collision}}\right)^2}  \leq    \frac{N_{\sigma_1}^p(t^p_\mathrm{first-collision})}{g(N_{\sigma_1}^p(t_\mathrm{first-collision}^p)} \leq  \frac{1152}{\left(\mu_{\sigma_1} - \mu_{\mathrm{collision}}\right)^2} 
\end{equation}
\item Define $t_{\mathrm{comm1-collision}}^1$ and $t_{\mathrm{comm-collision}}^1$ as functions of $t_{\mathrm{first-collision}}^1$ the same way as $t_{\mathrm{comm1}}^1$ and $t_{\mathrm{comm}}^1$ are functions of $t_{\mathrm{first}}^1$ in the previous sections and set the length of the communicating sequence starting at $t^1_{\mathrm{comm1-collision}}$ to be of size $t^1_{\mathrm{comm1-collision}}$ instead of $Kf(t_{\mathrm{comm1}}^1)$. The listening protocol for players $p \in \{2, \cdots M\}$ remains unchanged except for the communication rounds length $f$. Player $1$ will communicate a bit by pulling arm $\widehat{\sigma}_1^1 = \widehat{\sigma}_1^1(t^1_{\mathrm{first-collision}})$ from round $t^1_{\mathrm{comm1-collision}} + 1$ to round $2t^1_{\mathrm{comm1-collision}}$.  The listening protocol for players $p \in \{2, \cdots, M\}$ remains mostly unchanged. All players $p \in \{2, \cdots, M\}$ will compute a set of large empirical reward arms $\widehat{\mathbf{MaxArms}}^p$ defined as ,
\begin{align*}
    \widehat{\mathbf{MaxArms}}^p &\leftarrow \Big\{ i \in [K] \text{ s.t. } \\
  & \quad \widehat{\mu}_i^p(t_{\mathrm{listen-collision}}^p) - \widehat{\mu}_{\mathrm{collision}}^p(t_{\mathrm{listen-collision}}^p) \geq \\
  &\frac{1}{2} \max_{j \in [K]} \left( \widehat{\mu}_j^p(t_{\mathrm{listen-collision
    }}^p) - \widehat{\mu}_{\mathrm{collision}}^p(t_{\mathrm{listen-collision}}^p) \right)\Big\} .
\end{align*}

The players will then compute witness estimators $L_{i}^{p}(t_{\mathrm{listen-collision}}^p) $ for all arms in $i \in \widehat{\mathbf{MaxArms}}^p$ defined as,
\begin{equation*}
    L_{i}^{p}(t_{\mathrm{listen-collision}}^p)  = \frac{\widehat{\mu}_{i}^p(t^p_{\mathrm{listen-collision}} ) - \widehat{\mu}^p_{\mathrm{collision}}(t^p_{\mathrm{listen-collision}} )}{2} + \widehat{\mu}^p_{\mathrm{collision}}(t^p_{\mathrm{listen-collision}} ).
\end{equation*}

\item A bit $b$ with value $1$ is communicated by player $1$ to make sure all players $p \in \{2, \cdots, M\}$ learn $t^1_{\mathrm{comm1-collision}}$. During the listening protocol all players $p \in \{2, \cdots, M\}$ will collect samples from $t_{\mathrm{listen-collision}}^p+1$ to  $2t_{\mathrm{listen-collision}}^p$ following a $\mathrm{CollisionRoundRobin}$ schedule.  These samples will be used by players $p \in \{2, \cdots, M\}$ as input to the $\mathrm{CollisionTest}$ to figure if player $1$ has been pulling arm $\widehat{\sigma}_1^1$. Since $t^1_{\mathrm{comm1-collision} } \geq 2t^1_{\mathrm{comm-collision}} \geq 4t^1_{\mathrm{first-collision}}$ the different $t_{\mathrm{listen-collision}}^{p}$ times do not overlap with the sample collection for the $\mathrm{CollisionZeroTest}$:
\begin{align}
    &\text{If }  \exists i \in  \widehat{\mathbf{MaxArms}}^p \text{ s.t. } \widehat{\mu}_i^p( t_{\mathrm{listen-collision}}^p+1:2t_{\mathrm{listen-collision}}^p)  < L_i^p(t_{\mathrm{listen-collision}}^p): \notag\\
    &\quad \text{Return } b=1 \notag\\
    &\text{Else}: \notag \\
    &\quad \text{Return } b=0 \label{equation::collision_zero_test}
\end{align}
If $\mathcal{E}$ holds, $t_{\mathrm{comm1-collision}}^1$ equals the first $t_{\mathrm{listen-collision}}^p$ that returns $b=1$ for all $p \in [M]$. Once this signal has been received all players have shared knowledge of the value of $t_{\mathrm{comm1-collision}}^1$. 
\end{enumerate}
The regret incurred during this phase of the algorithm is at most $2\left(\mu_{\sigma_1} - \mu_{\mathrm{collision}} \right)t_{\mathrm{comm1-collision}}^p$. Since $t_{\mathrm{comm1-collision}}^p$ satisfies Equation~\ref{equation::support_equation_collision_estimation} the same argument as in Corollary~\ref{corollary::roundrobin_regret} implies - ignoring any polynomial factors of $K$ and $M$ -  the regret is upper bounded by $\mathcal{O}\left( \frac{\log(1/\delta )}{\mu_{\sigma_1} - \mu_{\mathrm{collision}} } \right)$. Once having established shared knowledge of $t_{\mathrm{comm1-collision}}^1$ among all players, the second phase of the algorithm starts. During this phase the player all players are to restart their empirical mean estimators for all arms although all listening players will keep a copy of their last witness values $L_{i}^{p}(t_{\mathrm{comm1-collision}}^p) $ for all arms in $i \in \widehat{\mathbf{MaxArms}}^p$.  The second phase of the algorithm bears more resemblance with Algorithms~\ref{algorithm::player_1_communicate},~\ref{algorithm::prepare_and_start_listening} and~\ref{algorithm::decode}. It is designed to transmit the composition of $\mathcal{C}_1^1(t_{\mathrm{first}}^1, 10)$. The players will start pulling arms following a traditional $\mathrm{RoundRobin}-$schedule and follow the exact same logic as Algorithms~\ref{algorithm::player_1_communicate} and~\ref{algorithm::prepare_and_start_listening} including the definitions of $t_{\mathrm{first}}^p$, $t_{\mathrm{comm}}^1$, $t_{\mathrm{comm1}}^1$ and $t_{\mathrm{listen}}^p$. The only difference is that instead of using $f$ to decide the length of the communication rounds, each bit is to be transmitted by player $1$ using the same protocol described above where empirical estimators of the rewards of arms $i \in \widehat{\mathbf{MaxArms}}^p$ are compared with the witness values $L_{i}^{p}(t_{\mathrm{comm1-collision}}^p)$. Thus each bit transmission costs -ignoring polynomial factors of $K$ and $M$-  at most $\mathcal{O}\left( \frac{\log(1/\delta)}{\mu_{\sigma_1} - \mu_{\mathrm{collision}} }\right)$ regret.  By the same argument as in the previous sections the lead-up to communicating $\mathcal{C}_1^1(t_{\mathrm{first}}^1, 10)$ incurrs in regret of order $\mathcal{O}\left( \frac{\log(1/\delta)}{\max_i \Delta_{\sigma_i, \sigma_{i+1}}}\right)$. Since $\frac{1}{\mu_{\sigma_1} - \mu_{\mathrm{collision}}} \leq \frac{1}{\max_i \Delta_{\sigma_i, \sigma_{i+1}}}$, we conclude the total regret -ignoring polynomial factors in $K$ and $M$ - up to the time all players have knowledge of $\mathcal{C}_1^1(t_{\mathrm{first}}^1, 10)$ is upper bounded by $\mathcal{O}\left( \frac{\log(1/\delta)}{\max_i \Delta_{\sigma_i, \sigma_{i+1}} }\right)$. The polynomial factors in $K$ and $M$ remain the same as in the zero collision reward setting. We flesh out this strategy in more detail below.

\subsection{Detailed Analysis of Unknown Collision Reward}\label{section::unknown_collision_reward_appendix}

In this section we explore the setting where the collision reward is a random variable with a mean of $\mu_{\mathrm{collision}}$ that is unknown to the learner. We assume that $\mu_{\mathrm{collision} } \leq \mu_{\sigma_K}$. We focus on showing the $\mathrm{CollisionTest}$ works as intended. We will follow the analysis of the communication protocol in Section~\ref{section::communication_analysis}. Let $t_{\mathrm{start-collision}}^1$ be the known start of the communication sequence.

Since $\max_{i} \Delta_{\sigma_i, \sigma_{i+1}} \leq \Delta_{\sigma_1, \sigma_{K}}   \leq \mu_{\sigma_1} - \mu_{\mathrm{collision}}$, the same logic used to prove Equations~\ref{equation::sandwich_condition_1} and~\ref{equation::condition_t_0_description_main} can be utilized to prove that whenever $\mathcal{E}$ holds, in Phase $1$ we have that
\begin{equation}\label{equation::support_equation_unkonown_collision_reward1_phase1}
    \frac{N_{\sigma_1}^p(t^1_\mathrm{first-collision})}{g(N_{\sigma_1}^p(t_\mathrm{first-collision}^1))}  \geq \frac{128}{\left(\mu_{\sigma_1} - \mu_{\mathrm{collision}}\right)^2} 
\end{equation}
and in Phase $2$ of the protocol,
\begin{equation}\label{equation::support_equation_unkonown_collision_reward1}
    \frac{N_{\sigma_1}^p(t^1_\mathrm{first})}{g(N_{\sigma_1}^p(t_\mathrm{first}^1))} \geq  \frac{128}{\Delta_{\sigma_1, \sigma_{K}}^2}  \geq \frac{128}{\left(\mu_{\sigma_1} - \mu_{\mathrm{collision}}\right)^2} 
\end{equation}
thus during Phase $1$ and Phase $2$,
\begin{equation}\label{equation::support_equation_unkonown_collision_reward}
    \frac{N_{\sigma_1}^p(t^1_\mathrm{start})}{g(N_{\sigma_1}^p(t_\mathrm{start}^1))}   \geq \frac{128}{\left(\mu_{\sigma_1} - \mu_{\mathrm{collision}}\right)^2} 
\end{equation}
Let's call $\widehat{\sigma}_{1}$ to player $1$'s guess for the largest arm at time $t_{\mathrm{comm1-collision}}^1$. Equation~\ref{equation::support_equation_unkonown_collision_reward} can be used to show that a similar but stronger set of properties as those presented at the beginning of Section~\ref{section::communication_analysis} holds with high probability, 

\paragraph{1. If $t_{\mathrm{start}}^1 = t_{\mathrm{comm1-collision}}^1$ then arm $\widehat{\sigma}_1$ has a large empirical mean for all players.} Arm $\widehat{\sigma}_1 \in \{ i \in [K] \text{ s.t. }\widehat{\mu}_i^p(t_{\mathrm{start}}^1) - \widehat{\mu}_{\mathrm{collision}}^p(t_{\mathrm{start}}^1) \geq \frac{1}{2} \max_{j \in [K]} \left(\widehat{\mu}_j^p(t_{\mathrm{start}}^1) - \widehat{\mu}_{\mathrm{collision}}^p (t_{\mathrm{start}}^1)  \right)\}$ for all $p \in \{2, \cdots, M\}$ and $\widehat{\mu}^p_{\widehat{\sigma}_1}(t_{\mathrm{start}}^1) -D(N_{\widehat{\sigma}_1}(t_{\mathrm{start}}^1)) \geq \frac{\mu_{\sigma_1}- \mu_{\mathrm{collision}}}{2} + \mu_{\mathrm{collision}}$. 
\begin{itemize}
    \item To see a formal proof of this statement refer to Lemma~\ref{lemma::arm_sigma_1_large_empirical_mean_unknown_collision} in Appendix~\ref{section::unknown_collision_reward_appendix}. Set $\tilde{C} = 10$. This result makes sure that $\widehat{\sigma}_1 \in \widehat{\mathbf{MaxArms}}^p$ with high probability. As a side product of this result it is possible to show that $\mu_{\widehat{\sigma}_1} - \mu_{\mathrm{collision}} \geq \frac{3}{4}(\mu_{\sigma_1} - \mu_{\mathrm{collision}})$. In other words, the gap $\mu_{\widehat{\sigma}_1} - \mu_{\mathrm{collision}}$ is at least a constant multiple of the gap $\mu_{\sigma_1} - \mu_{\mathrm{collision}}$. 
\end{itemize}

\paragraph{2. Arm $\widehat{\sigma}_1$ is comparable to $\sigma_1$.}  If $t_{\mathrm{start}}^1 = t_{\mathrm{comm1-collision}}^1$ the witnesses $L_{\widehat{\sigma}_1}^p(t_{\mathrm{start}}^1) $ satisfy, $L_{\widehat{\sigma}_1}^p(t_{\mathrm{start}}^1) \in \left[\frac{3\left(\mu_{\widehat{\sigma}_1} - \mu_{\mathrm{collision}}\right)}{7}, \frac{4\left(\mu_{\widehat{\sigma}_1} - \mu_{\mathrm{collision}}\right)}{7}\right] + \mu_{\mathrm{collision}} $ for all $p \in \{ 2, \cdots, M\}$.

\begin{itemize}
    \item To see a formal proof of this statement refer to Lemma~\ref{lemma::arm_hat_sigma_one_comparable_sigma_one_unknown_collision} in Appendix~\ref{section::unknown_collision_reward_appendix}. Set $\tilde{C} = 10$. This results guarantees we can compute a `witness' value that is a constant multiple of $\mu_{\sigma_1} - \mu_{\mathrm{collision}}$ away from $\mu_{\sigma_1}$ and $\mu_{\mathrm{collision}}$. Indeed it is easy to see $\mu_{\widehat{\sigma}_1} - \mu_{\mathrm{collision}} \geq \frac{3}{4}(\mu_{\sigma_1} - \mu_{\mathrm{collision}})$ (see Equation~\ref{equation::lower_bound_empirical_gaps_true_gap}) implies $L_{\widehat{\sigma}_1}^p(t_{\mathrm{start}}^1) \in \left[\frac{9\left(\mu_{\sigma_1} - \mu_{\mathrm{collision}}\right)}{28}, \frac{16\left(\mu_{\sigma_1} - \mu_{\mathrm{collision}}\right)}{28}\right] + \mu_{\mathrm{collision}} $  for all $p \in \{ 2, \cdots, M\}$.
\end{itemize}

\paragraph{3. When collisions are avoided with high probability the empirical mean estimators $\widehat{\mu}^p_{\widehat{\sigma}_1}(t_{\mathrm{start}}^1 +1:  + 2t_{\mathrm{start}}^1)$ are far from $\mu_{\mathrm{collision}}$. } If $\mathcal{E}$ holds and
    \begin{itemize}
        \item $b = 1$, the estimators $\widehat{\mu}^p_{\widehat{\sigma}_1}(t_{\mathrm{start}}^1 +1:   2t_{\mathrm{start}}^1) \leq L_{\widehat{\sigma}_1}^p$ for all $p \in [M]$.
        \item $b = 0$, the estimators $\widehat{\mu}^p_{\widehat{\sigma}_1}(t_{\mathrm{start}}^1 + 1:2t_{\mathrm{start}}^1 ) > L_{\widehat{\sigma}_1}^p$ for all $p \in [M]$. 
    \end{itemize}
To prove the third item, notice that in case $b=1$ is to be transmitted $\widehat{\mu}^p_{\widehat{\sigma}_1}(t_{\mathrm{start}}^1 +1:   2t_{\mathrm{start}}^1) \leq \mu_{\mathrm{collision}} + D(N_{\widehat{\sigma}_1}(t_{\mathrm{start}}^1)) \leq \mu_{\mathrm{collision}} +  \frac{\mu_{\widehat{\sigma}_1}  - \mu_{\mathrm{collision} }}{r-2}$ where $r = 2(\tilde{C}-2)$ (see Equation~\ref{equation::upper_bound_D_r_minus_2} for a proof). Therefore $\widehat{\mu}^p_{\widehat{\sigma}_1}(t_{\mathrm{start}}^1 +1:   2t_{\mathrm{start}}^1)  \leq \mu_{\mathrm{collision}} + \frac{\mu_{\widehat{\sigma}_1}  - \mu_{\mathrm{collision} }}{14} < \mu_{\mathrm{collision}} + \frac{9\left(\mu_{\widehat{\sigma}_1}  - \mu_{\mathrm{collision} }\right)}{28} $. Similarly in case $b= 0$, $\widehat{\mu}^p_{\widehat{\sigma}_1}(t_{\mathrm{start}}^1 +1:   2t_{\mathrm{start}}^1) \geq \mu_{\mathrm{collision}} - \frac{\mu_{\widehat{\sigma}_1}  - \mu_{\mathrm{collision} }}{14} >  \mu_{\mathrm{collision}} + \frac{16\left(\mu_{\widehat{\sigma}_1}  - \mu_{\mathrm{collision} }\right)}{28}$.

Lemmas~\ref{lemma::arm_sigma_1_large_empirical_mean_unknown_collision} and~\ref{lemma::arm_hat_sigma_one_comparable_sigma_one_unknown_collision} comprise the bulk of the necessary steps to adapt our results where $\mu_{\mathrm{collision}} = 0$ to the more general setting where $\mu_{\mathrm{collision}} > 0$ and therefore these finalize the proof of Lemma~\ref{lemma::collision_test_works_main}. Let's assume the constant defining the $\tilde{C}-$blowup connectivity graph that defines time $t_{\mathrm{first}}^1$ equals $\tilde{C}$. From Lemma~\ref{lemma::first_fundamental_lemma_main} we can conclude that,

\begin{equation}\label{equation::support_equation_unkonown_collision_reward1_appendix}
    \frac{N_{\sigma_1}^p(t_{\mathrm{first}}^1)}{g(N_{\sigma_1}^p(t_{\mathrm{first}}^1) } \geq  \frac{2(\tilde{C} - 2)^2}{\left(  \mu_{\sigma_1} - \mu_{\mathrm{collision}} \right)^2}.
\end{equation}

And since $t_{\mathrm{start}}^1 \geq t_{\mathrm{first}}^1$ and $t_{\mathrm{first}}^1 \geq Ks_{\mathrm{boundary2}}$, the ratio $\frac{N_{\sigma_1}^p(t)}{g(N_{\sigma_1}^p(t) )}$ is non decreasing for all $t \geq t_{\mathrm{first}}^1$. Thus,
\begin{equation}\label{equation::support_equation_unkonown_collision_reward_appendix}
    \frac{N_{\sigma_1}^p(t_{\mathrm{start}}^1)}{g(N_{\sigma_1}^p(t_{\mathrm{start}}^1) }  \geq \frac{2(\tilde{C} - 2)^2}{\left(  \mu_{\sigma_1} - \mu_{\mathrm{collision}} \right)^2}.
\end{equation}

\begin{lemma}\label{lemma::arm_sigma_1_large_empirical_mean_unknown_collision}
 If $\mathcal{E}$ holds and $\tilde{C} \geq 9$,

 $$\widehat{\sigma}_1 \in \left\{ i \in [K] \text{ s.t. }\widehat{\mu}_i^p(t_{\mathrm{start}}^1) - \widehat{\mu}_{\mathrm{collision}}^p(t_{\mathrm{start}}^1) \geq \frac{1}{2} \max_{j \in [K]} \left(\widehat{\mu}_j^p(t_{\mathrm{start}}^1) - \widehat{\mu}_{\mathrm{collision}}^p (t_{\mathrm{start}}^1)  \right)\right\}$$ for all $p \in \{2, \cdots, M\}$ and $\widehat{\mu}^p_{\widehat{\sigma}_1}(t_{\mathrm{start}}^1) -D(N_{\widehat{\sigma}_1}(t_{\mathrm{start}}^1)) \geq \left(1- \frac{4}{r}\right)\left(\mu_{\sigma_1} - \mu_{\mathrm{collision}}\right) + \mu_{\mathrm{collision}}$.

\end{lemma}

\begin{proof}

By Equations~\ref{equation::support_equation_unkonown_collision_reward1_appendix} and~\ref{equation::support_equation_unkonown_collision_reward_appendix}, similar to the proof of Lemma~\ref{lemma::zero_test_works_main} we conclude that $D(N_{\widehat{\sigma}_1}(t_{\mathrm{start}}^1)) = \sqrt{\frac{g(N_{\sigma_1}^p(t_{\mathrm{start}}^1))}{2 N_{\sigma_1}^p(t_{\mathrm{start}}^1) } }\leq \frac{\mu_{\sigma_1} - \mu_{\mathrm{collision}}}{r}$ where $r = 2(\tilde{C}-2)$ if $\mathcal{E}$ holds, $\widehat{\mu}_i^p(t_{\mathrm{start}}^1) \in [\mu_i - \frac{\mu_{\sigma_1}-\mu_{\mathrm{collision}}}{r}, \mu_i + \frac{\mu_{\sigma_1}-\mu_{\mathrm{collision}}}{r}]$ for all $i \in [K] \cup \{\mathrm{collision} \}$ and $\widehat{\mu}_i^1(t_{\mathrm{first}}^1) \in [\mu_i - \frac{\mu_{\sigma_1}-\mu_{\mathrm{collision}}}{r}, \mu_i + \frac{\mu_{\sigma_1}-\mu_{\mathrm{collision}}}{r}]$ for all $i \in [K] \cup \{\mathrm{collision}\}$. These facts in conjunction with the definition of $\widehat{\sigma}_1$ imply,
\begin{equation}\label{equation::sigma_one_in_terms_sigma_one_hat}
    \mu_{\sigma_1} - \frac{\mu_{\sigma_1} - \mu_{\mathrm{collision}}}{r} \leq \widehat{\mu}_{\sigma_1}^1(t_{\mathrm{first}}^1) \leq \widehat{\mu}_{\widehat{\sigma}_1}^1(t_{\mathrm{first}^1}) \leq \mu_{\widehat{\sigma}_1} + \frac{\mu_{\sigma_1} - \mu_{\mathrm{collision}}}{r}.
\end{equation}
and therefore $\mu_{\sigma_1} \leq \mu_{\widehat{\sigma}_1} + \frac{2(\mu_{\sigma_1} - \mu_{\mathrm{collision}})}{r}$. Thus,
\begin{align*}
    \widehat{\mu}_{\widehat{\sigma}_1}^p\left(t_{\mathrm{start}}^1  \right) -  \widehat{\mu}_{\mathrm{collision}}^p\left(t_{\mathrm{start}}^1  \right) &\geq \mu_{\widehat{\sigma}_1} - \frac{\mu_{\sigma_1} - \mu_{\mathrm{collision}}}{r} - \frac{\mu_{\sigma_1} - \mu_{\mathrm{collision}}}{r} - \mu_{\mathrm{collision}} \\
    &\geq \mu_{\sigma_1} - \frac{2(\mu_{\sigma_1} - \mu_{\mathrm{collision}})}{r} - \frac{2(\mu_{\sigma_1} - \mu_{\mathrm{collision}})}{r} -\mu_{\mathrm{collision}}   \\
    &= (\mu_{\sigma_1} - \mu_{\mathrm{collision}}) \left( 1-\frac{4}{r}\right)
\end{align*}

Similarly for any $j \in [K]$,
\begin{align*}
    \widehat{\mu}_j^p(t_{\mathrm{start}}^1) - \widehat{\mu}_{\mathrm{collision}}^p( t_{\mathrm{start}}^1) &\leq \mu_{\sigma_1} + \frac{\mu_{\sigma_1} - \mu_{\mathrm{collision}}}{r} - \mu_{\mathrm{collision}} +   \frac{\mu_{\sigma_1} - \mu_{\mathrm{collision}}}{r}\\
    &\leq \mu_{\sigma_1} + \frac{2(\mu_{\sigma_1} - \mu_{\mathrm{collision}})}{r} - \mu_{\mathrm{collision}} \\
    &\leq (\mu_{\sigma_1} - \mu_{\mathrm{collision}})\left(1+ \frac{2}{r}\right)
\end{align*}

Since $\tilde{C} \geq 9$, $r \geq 14$, it follows that $\left( \frac{1}{2} + r \right) \leq  \left( 1-\frac{4}{r}\right) $ and therefore for any $j \in [K]$

\begin{align*}
 \frac{\widehat{\mu}_j^p(t_{\mathrm{start}}^1) - \widehat{\mu}_{\mathrm{collision}}^p( t_{\mathrm{start}}^1) }{2}  &\leq (\mu_{\sigma_1} - \mu_{\mathrm{collision}})\left(\frac{1}{2}+ \frac{1}{r}\right) \leq (\mu_{\sigma_1} - \mu_{\mathrm{collision}})\left(1- \frac{4}{r}\right)\\
& \leq   \widehat{\mu}_{\widehat{\sigma}_1}^p\left(t_{\mathrm{start}}^1  \right) -  \widehat{\mu}_{\mathrm{collision}}^p\left(t_{\mathrm{start}}^1  \right)
\end{align*}

Thus implying the first result. The second statement is a result of the inequality,

\begin{align*}
    \widehat{\mu}_{\widehat{\sigma}_1}^p ( t_{\mathrm{start}}^1) -  D(N_{\widehat{\sigma}_1}(t_{\mathrm{start}}^1)) &\geq \mu_{\widehat{\sigma}_1} - 2D(N_{\widehat{\sigma}_1}(t_{\mathrm{start}}^1)) \\
    &\geq \mu_{\sigma_1} - \frac{2(\mu_{\sigma_1} - \mu_{\mathrm{collision}})}{r} - \frac{2(\mu_{\sigma_1} - \mu_{\mathrm{collision}})}{r} \\
    &= \mu_{\mathrm{collision}}+ (\mu_{\sigma_1} - \mu_{\mathrm{collision}}) \left(1- \frac{4}{r}\right)
\end{align*}

The result follows.

\end{proof}

By Equation~\ref{equation::sigma_one_in_terms_sigma_one_hat} in the proof of Lemma~\ref{lemma::arm_sigma_1_large_empirical_mean_unknown_collision} we infer that:

\begin{equation*}
    \mu_{\widehat{\sigma}_1} \leq \mu_{\sigma_1} \leq \mu_{\widehat{\sigma}_1} + \frac{2\left( \mu_{\sigma_1} - \mu_{\mathrm{collision}} \right) }{ r}
\end{equation*}

Therefore the gap between $\mu_{\widehat{\sigma}_1}$ and $\mu_{\mathrm{collision}}$ is lower bounded by,
\begin{equation}\label{equation::lower_bound_empirical_gaps_true_gap}
    \mu_{\widehat{\sigma}_1} - \mu_{\widehat{\sigma}_K^p} \geq (\mu_{\sigma_1} - \mu_{\mathrm{collision}})\left( 1- \frac{2}{r}\right).
\end{equation}

Most importantly Equation~\ref{equation::lower_bound_empirical_gaps_true_gap} shows the gap $\mu_{\widehat{\sigma}_1} - \mu_{\mathrm{collision}}$ is at least a constant multiple of the gap $\mu_{\sigma_1} - \mu_{\mathrm{collision}}$.

\begin{lemma}\label{lemma::arm_hat_sigma_one_comparable_sigma_one_unknown_collision}

The witnesses $L_{\widehat{\sigma}_1}^p( t_{\mathrm{start}}^1) = \frac{\widehat{\mu}^p_{\widehat{\sigma}_1}(t_{\mathrm{start}}^1) - \widehat{\mu}_{\mathrm{collision}}^p(t_{\mathrm{start}}^1)}{2}  + \widehat{\mu}_{\mathrm{collision}}^p(t_{\mathrm{start}}^1) $ satisfy,

$$L_{\widehat{\sigma}_1}^p( t_{\mathrm{start}}^1) \in \mu_{\mathrm{collision}}  + \left[\left(\frac{1}{2} -\frac{2}{r-2}\right)\left(\mu_{\widehat{\sigma}_1} - \mu_{\mathrm{collision}}\right),\left(\frac{1}{2} +\frac{2}{r-2}\right)\left( \mu_{\widehat{\sigma}_1} - \mu_{\mathrm{collision}}\right)\right] $$ for all $p \in \{ 2, \cdots, M\}$ where $r = 2(\tilde{C}-2)$.

\end{lemma}

\begin{proof}
By Equation~\ref{equation::support_equation_unkonown_collision_reward_appendix}, similar to the discussion above, $D(N_{\widehat{\sigma}_1}(t_{\mathrm{start}}^1)) \leq \frac{\mu_{\sigma_1} - \mu_{\mathrm{collision}}}{r}$ (for $r = 2(\tilde{C}-2)$) if $\mathcal{E}$ holds and therefore $\mu_{\sigma_1} \leq \mu_{\widehat{\sigma}_1} + \frac{2(\mu_{\sigma_1} - \mu_{\mathrm{collision}})}{r}$. Hence,
\begin{equation*}
    \mu_{\sigma_1} - \mu_{\mathrm{collision}} \leq \mu_{\widehat{\sigma}_1} + \frac{2(\mu_{\sigma_1} - \mu_{\mathrm{collision}})}{r} - \mu_{\mathrm{collision} }. 
\end{equation*}
Thus implying,
\begin{equation*}
    \mu_{\sigma_1} - \mu_{\mathrm{collision}} \leq \frac{r}{r-2}\left( \mu_{\widehat{\sigma}_1}  - \mu_{\mathrm{collision} } \right)
\end{equation*}
and therefore that
\begin{equation}\label{equation::upper_bound_D_r_minus_2}
    D(N_{\widehat{\sigma}_1}(t_{\mathrm{start}}^1)) \leq \frac{\mu_{\widehat{\sigma}_1}  - \mu_{\mathrm{collision} }}{r-2}
\end{equation}
then,
\begin{align*}
    \widehat{\mu}^p_{\widehat{\sigma}_1}(t_{\mathrm{start}}^1)  - \widehat{\mu}_{\mathrm{collision}}^p(t_{\mathrm{start}}^1) &\leq \mu_{\widehat{\sigma}_1}  - \mu_{\mathrm{collision}}  + \frac{2(\mu_{\widehat{\sigma}_1}  - \mu_{\mathrm{collision} })}{r-2} \\
    &= \left(1+ \frac{2}{r-2}\right) (\mu_{\widehat{\sigma}_1}  - \mu_{\mathrm{collision} })
\end{align*}
Similarly,
\begin{align*}
    \widehat{\mu}^p_{\widehat{\sigma}_1}(t_{\mathrm{start}}^1)  - \widehat{\mu}_{\mathrm{collision}}^p(t_{\mathrm{start}}^1) &\geq \mu_{\widehat{\sigma}_1}  - \mu_{\mathrm{collision} }  - \frac{2(\mu_{\widehat{\sigma}_1}  - \mu_{\mathrm{collision} })}{r-2} \\
    &= \left(1 -\frac{2}{r-2}\right) (\mu_{\widehat{\sigma}_1}  - \mu_{\mathrm{collision} })
\end{align*}
Since $\mu_{\mathrm{collision} } - \frac{\mu_{\widehat{\sigma}_1}  - \mu_{\mathrm{collision} }}{r-2} \leq \widehat{\mu}_{\mathrm{collision}}^p(t_{\mathrm{start}}^1) \leq \mu_{\mathrm{collision} }  + \frac{\mu_{\widehat{\sigma}_1}  - \mu_{\mathrm{collision} }}{r-2} $,

\begin{align*}
  \mu_{\mathrm{collision} } + \left(\frac{1}{2} -\frac{2}{r-2}\right) (\mu_{\widehat{\sigma}_1}  - \mu_{\mathrm{collision} }) &\leq  \widehat{\mu}_{\mathrm{collision}}^p(t_{\mathrm{start}}^1) + \frac{\widehat{\mu}^p_{\widehat{\sigma}_1}(t_{\mathrm{start}}^1)  - \widehat{\mu}_{\mathrm{collision}}^p(t_{\mathrm{start}}^1)}{2}  \\
  &\leq   \mu_{\mathrm{collision}} + \left(\frac{1}{2} +\frac{2}{r-2}\right) (\mu_{\widehat{\sigma}_1}  - \mu_{\mathrm{collision} })
\end{align*}

\end{proof}

Finally the following `inverted' version of Lemma~\ref{lemma::zero_test_preliminary_bound} will prove useful,

\begin{restatable}{lemma}{zerotestpreliminaryboundtwo}\label{lemma::unknown_collision_test_preliminary_bound}
Let $\delta' \in (0,1)$. If $X$ is a random variable with support in $[0,1]$, distribution $\mathbb{P}_X$ and mean $\mu_X$ satisfying $ \mu_X \leq U$, then with probability at least $1-\delta'$ for all $N$ such that $\frac{N}{B(N, \delta')} \geq \max\left(  \frac{2}{U-\mu_X }, \frac{16 \min(\mu_X, 1-\mu_X)}{(U-\mu_X )^2} \right)$ we have,
\begin{equation*}
    \widehat{\mu}_X \leq U, 
\end{equation*}
where $B(n, \delta') = 2\ln \ln(2n) + \ln \frac{5.2}{\delta'}$.
\end{restatable}

\begin{proof}

Let $\alpha = \min(\mu_X, 1-\mu_X)$. A simple use of the reversed version of Lemma~\ref{lemma::bernstein_concentration_application} implies that with probability at least $1-\delta'$ for all $n\in \mathbb{N}$:
\begin{align*}
    \mu_X + 2 \sqrt{ \frac{ \alpha B(n, \delta')}{n}} - \frac{ B(n, \delta')}{n}\geq \widehat{\mu}_X.  
\end{align*}
The LHS of this inequality attains a value of at most $U$ whenever:
\begin{equation*}
    U-\mu_X \geq 2 \sqrt{ \frac{ \alpha B(n, \delta')}{n}} - \frac{ B(n, \delta')}{n}.
\end{equation*}
We finalize  the proof by noting that for all $n$ such that $\frac{n}{B(n, \delta')} \geq \max\left(  \frac{2}{U-\mu_X }, \frac{16 \min(\mu_X, 1-\mu_X)}{(U-\mu_X )^2} \right)$ we have that $\frac{U-\mu_X }{2} \geq 2 \sqrt{ \frac{ \alpha B(n, \delta')}{n}}$ and $\frac{U- \mu_X }{2} \geq \frac{B(n, \delta')}{n}$.  
\end{proof}

\subsection{Complex Restart Strategy}\label{appendix::complex_restart}

Here we describe a way to reuse some of the collected samples so far and warm start the estimators. Let's assume $t_{\mathrm{comm1}}^1  = \min_{t \in \mathbb{N}} \text{ s.t. } \left\lfloor \frac{t/K}{g(t/K)}\right\rfloor = 9^{\tilde{w}}$ for some $\tilde{u}$ and define $$t_{\mathrm{restart}} = \min_{t \in \mathbb{N}} \text{ s.t. } \left\lfloor \frac{t/K}{g(t/K)}\right\rfloor = 9^{\tilde{w}-3}.$$ Let's see that whenever $\mathcal{E}$ holds $t_{\mathrm{first}}^p > t_{\mathrm{restart}}$ for all $p \in [M]$. Similar to the proof of Lemma~\ref{lemma::listen_agrees_comm} let's denote by $9^u$ the unique power of nine in the interval $\left[\frac{128}{\max_{i} \Delta_{\sigma_i, \sigma_{i+1}}^2}   , \frac{1152}{\max_{i} \Delta_{\sigma_i, \sigma_{i+1}}^2}    \right)$. Recall that whenever the good event $\mathcal{E}$ holds $\frac{s^p_{\mathrm{first}}}{g(s^p_{\mathrm{first}}) }\in    \left[\frac{128}{\max_{i} \Delta_{\sigma_i, \sigma_{i+1}}^2}   , \frac{1152}{\max_{i} \Delta_{\sigma_i, \sigma_{i+1}}^2}    \right) $ for all $p \in [M]$ and $$t_{\mathrm{comm1}}^1 \in \left\{ \min_{t \in \mathbb{N}} \text{ s.t. } \left\lfloor \frac{t/K}{g(t/K)}\right\rfloor = 9^{u+1},  \min_{t \in \mathbb{N}} \text{ s.t. } \left\lfloor \frac{t/K}{g(t/K)}\right\rfloor = 9^{u+2} \right\}. $$ Since by definition $\tilde{w}-3 < u$ it must be the case that $t_{\mathrm{restart}} < t_{\mathrm{first}}^p$ for all $p \in [M]$.

When jumping into these smaller problems, all players will warm-start their empirical reward estimators at $\{ \hat{\mu}_i^p(t_{\mathrm{restart}} )\}_{i \in [K], p \in [M]}$, throwing away all the information gathered during the communication rounds. 

Each player will now re-index time to suit the sub-problem it has landed on by throwing away the data corresponding to historical rounds where samples not belonging to the connected component assigned to her were collected from $t =1$ to $t_{\mathrm{restart}}$. This procedure ensures each sub-problem is at a state where there is no player for which the condition $\mathbf{conn}^p\left(sK, 5\right) \geq 2$ has been triggered, while ensuring a substantial proportion of the data collected so far can be reused. 

Since the proportion of samples that can be reused using this strategy is constant, no substantial speedup can be gained from following this strategy.

\section{Missing Proofs}

In this section we present the proofs of all those lemmas for which having the proof present in the main or the Appendix discussion section would have hindered the flow of the text.

\subsection{Proof of Lemma~\ref{lemma::first_fundamental_lemma_main}}\label{section::supporting_t0_lemmas}

We restate Lemma~\ref{lemma::first_fundamental_lemma_main} for the reader's convenience.

\fundamentallemmaupperlowerboundmain*

\begin{proof}
Notice that whenever $\mathcal{E}$ holds:

\begin{equation*}
    \widehat{\mu}_{\sigma_i}(t) - \widehat{\mu}_{\sigma_j}(t) \geq \Delta_{\sigma_i, \sigma_j} - 2D(N(t)).
\end{equation*}
Hence whenever $D(N(t)) \leq \frac{\Delta_{\sigma_i, \sigma_j}}{C+2}$,
\begin{equation*}
    \widehat{\mu}_{\sigma_i}(t) - \widehat{\mu}_{\sigma_j}(t) \geq \Delta_{\sigma_i, \sigma_j} - 2D(N(t)) \geq C D(N(t)),
\end{equation*}
and therefore condition~\ref{equation::elimination_condition} will trigger. This implies that special round $t-1$, being one before condition~\ref{equation::elimination_condition} is ever triggered must satisfy $D(N(t-1)) > \frac{\Delta_{\sigma_i, \sigma_j}}{C+2}$. Since $\delta < \frac{1}{2}$ Lemma~\ref{lemma::slow_decrease_D} ensures that $D(N(t-1)) < 2D(N(t))$ and therefore that 
\begin{equation}\label{equation::lower_bound_D}
    D(N(t)) > \frac{\Delta_{\sigma_i, \sigma_j}}{2(C+2)}.
\end{equation}
Similarly note that whenever $\mathcal{E}$ holds 
$$\widehat{\mu}_{\sigma_i}(t) - \widehat{\mu}_{\sigma_j}(t) - 2D(N(t))  \leq \mu_{\sigma_i} + D(N(t)) - \mu_{\sigma_j} + D(N(t)) -2D(N(t)) \leq \Delta_{\sigma_{i}, \sigma_{j}}.$$ 
and therefore if the condition $  \widehat{\mu}_{\sigma_i}(t) - \widehat{\mu}_{\sigma_j}(t) \geq C D(N(t)) $ was true, then,

\begin{equation*}
(C-2)D(N(t)) \leq      \widehat{\mu}_{\sigma_i}(t) - \widehat{\mu}_{\sigma_j}(t) - 2D(N(t))  \leq \Delta_{\sigma_{i}, \sigma_{j}}
\end{equation*}
Therefore,
\begin{equation}\label{equation::upper_bound_D}
    D(N(t)) \leq \frac{ \Delta_{\sigma_i, \sigma_j} }{C-2}.
\end{equation}

We now turn our attention to lower and upper bounding $\widehat{\Delta}_{\sigma_{i}, \sigma_{j}} $. Since $\widehat{\Delta}_{\sigma_{i}, \sigma_{j}} =\widehat{\mu}_{\sigma_i}(t) - \widehat{\mu}_{\sigma_j}(t) - 2D(N(t))$ we can conclude that $\widehat{\Delta}_{\sigma_{i}, \sigma_{j}} \leq \Delta_{\sigma_{i}, \sigma_{j}}$. We use Equation~\ref{equation::upper_bound_D} to produce a lower bound,
\begin{equation*}
    \widehat{\Delta}_{\sigma_{i}, \sigma_{j}} = \widehat{\mu}_{\sigma_i}(t) - \widehat{\mu}_{\sigma_j}(t) - 2 D(N(t))  \geq \Delta_{\sigma_i, \sigma_j} - 4 D(N(t)) \geq \frac{C-3}{C-2} \Delta_{\sigma_i, \sigma_j}.
\end{equation*}

Plugging in the definition of $D(N(t))$ and using the lower and upper bounds of equations~\ref{equation::lower_bound_D} and~\ref{equation::upper_bound_D} yields:
\begin{equation*}
 \frac{ 2(C-2)^2\ln(4(N(t))^2MK/\delta)}{\Delta_{\sigma_i, \sigma_j}^2} \leq     N(t) \leq  \frac{ 8(C+2)^2\ln(4(N(t))^2MK/\delta)}{\Delta_{\sigma_i, \sigma_j}^2}.
\end{equation*}

\end{proof}

\subsection{Proof of Lemma~\ref{lemma::upper_bound_ratio_scomm1}}\label{section::proof_upper_bound_ratio_scomm1}

We restate Lemma~\ref{lemma::upper_bound_ratio_scomm1} for readability.

\lemmaboundingscommone*

\begin{proof}
Notice that by definition there is a $u \in \mathbb{N}$ such that,
\begin{equation}\label{equation::successive_power_nine_ineqs}
    9^{u-1} < \left\lfloor \frac{s_{\mathrm{first}}^1}{g(s^1_\mathrm{first})} \right \rfloor \leq 9^u =   \left \lfloor \frac{s_{\mathrm{comm}}^1}{g(s^1_\mathrm{comm})} \right \rfloor < 9^{u+1} =  \left \lfloor \frac{s_{\mathrm{comm1}}^1}{g(s^1_\mathrm{comm1})} \right \rfloor
\end{equation}

Since by assumption $\delta \leq \frac{1}{162}$ it is easy to see that $4\left(s_{\mathrm{first}}^1\right)^2 MK/\delta \geq 162$, it follows that $g(162 s_{\mathrm{first}}^1) \leq 2 g(s_\mathrm{first}^1)$. Thus by inequality~\ref{equation::successive_power_nine_ineqs},
\begin{equation*}
   \frac{s_{\mathrm{comm1}}^1}{g(s^1_\mathrm{comm1})}  \leq 9^{u+1} + 81 = (9^{u-1} + 1) \cdot 81  \leq 81 \frac{s_\mathrm{first}^1}{g(s_\mathrm{first}^1)} \leq \frac{162 s_\mathrm{first}^1}{g(162 s_\mathrm{first}^1)}
\end{equation*}
Recall that by definition $s_{comm1}^1$ is the first integer such that $ \left \lfloor \frac{s_{\mathrm{comm1}}^1}{g(s^1_\mathrm{comm1})} \right \rfloor = 9^{u+1}$. Since $D(s)$ is decreasing for all $s \geq s_{\mathrm{boundary2}}$  and $s_{\mathrm{first}}^1$ is assumed to be at least $ s_{\mathrm{boundary2}}$, we can conclude that $162 s_{\mathrm{first}}^1 \geq s_{\mathrm{comm1}}^1$. The first result follows. Since $g(s)$ is an increasing function we conclude that
\begin{equation*}
    \frac{162 s_\mathrm{first}^1}{g(162 s_\mathrm{first}^1)} \leq \frac{162 s_\mathrm{first}^1}{g( s_\mathrm{first}^1)}
\end{equation*}

The second result follows. 
\end{proof}

\subsection{Proof of Lemma~\ref{lemma::upper_bounding_s_comm_1}}\label{section::proof_upper_bounding_s_comm_1}

We restate Lemma~\ref{lemma::upper_bounding_s_comm_1} for readability.

\lemmaboundingscommonenolog*

\begin{proof}
As a consequence of Lemma~\ref{lemma::upper_bound_ratio_scomm1}, we see that $  \frac{s_{comm1}^1}{g(s_{comm1}^1)} \leq \frac{162 s_\mathrm{first}^1}{g( s_\mathrm{first}^1)}$. If $\mathcal{E}$ holds, Equation~\ref{equation::sandwich_condition_1} implies that $ \frac{s_{comm1}^1}{g(s_{comm1}^1)} \leq \frac{186624 }{\max_{i} \Delta_{\sigma_i, \sigma_{i+1}}^2}$. Let $h(n) = \frac{n}{g(n)}$. Notice that  $h'(n) = \frac{\log(4MK n/\delta) - 1}{\log^2(4MK n/\delta)}$. Since $\delta <\frac{1}{162}$ we conclude that $h'(n) > 0$ for all $n \geq 1$. Since by definition both $s_{\mathrm{comm1}}^1$ and $\frac{186624 }{\max_{i} \Delta_{\sigma_i, \sigma_{i+1}}^2}$ are both at least $1$, and $\frac{4MK}{\delta} \times \frac{186624 }{\max_{i} \Delta_{\sigma_i, \sigma_{i+1}}^2}\geq  4$ for all, a simple use of Lemma~\ref{lemma::supporting_lemma_miscellaneous} where $x = s_{\mathrm{comm1}}^1, c =4MK/\delta  $ and $b =\frac{186624 }{\max_{i} \Delta_{\sigma_i, \sigma_{i+1}}^2} $ implies that,
\begin{equation*}
    s_{\mathrm{comm1}}^1 \leq \frac{746496 }{\max_{i} \Delta_{\sigma_i, \sigma_{i+1}}^2} \log\left(  \frac{746496 MK }{\delta\max_{i} \Delta_{\sigma_i, \sigma_{i+1}}^2}  \right).
\end{equation*}

\end{proof}

\subsection{Proof of Lemma~\ref{lemma::bounding_collision_regret_unit2}}\label{section::proof_bounding_collision_regret_unit2} 

We restate Lemma~\ref{lemma::bounding_collision_regret_unit2} for readability.

\lemmaboundingfcommone*

\begin{proof}
\begin{align*}
    f(t_{\mathrm{comm1}}^1) &\stackrel{(i)}{\leq}  48 B\left(f(t_{\mathrm{comm1}}^1), \frac{\delta}{4K^2M}\right)\sqrt{\frac{t_{\mathrm{comm1}}^1/K}{2g(t_{\mathrm{comm1}}^1/K)} } \\
    &\stackrel{(ii)}{\leq} 48  B\left(f(t_{\mathrm{comm1}}^1), \frac{\delta}{4K^2M}\right) \sqrt{ \frac{162 t_{\mathrm{first}}^1 / K}{2g(t_{\mathrm{first}}^1/K)} }\\
    &\stackrel{(iii)}{\leq} 48 B\left(f(t_{\mathrm{comm1}}^1), \frac{\delta}{4K^2M}\right)\frac{\sqrt{1152*162} }{\max_{i} \Delta_{\sigma_i, \sigma_{i+1}}} \\
    &=  \frac{20736 B\left(f(t_{\mathrm{comm1}}^1), \frac{\delta}{4K^2M}\right) }{\max_{i} \Delta_{\sigma_i, \sigma_{i+1}}},
\end{align*} 

where inequality $(i)$ follows from elementary properties of $f(\cdot)$ (see Lemma~\ref{lemma::upper_bound_f_ratio_double} in Appendix~\ref{section::properties_of_f}) along with the assumption $t_{\mathrm{comm1}}^1\geq t_{\mathrm{boundary1}}$, and $(ii)$ follows from Lemma~\ref{lemma::upper_bound_ratio_scomm1} along with the assumptions $s_{\mathrm{first}}^1 \geq s_{\mathrm{boundary2
}}$ and $\delta \leq \frac{1}{162}$. Inequality $(iii)$ follows because $\mathcal{E}$ holds, $N(t_{\mathrm{first}}^1) = t_{\mathrm{first}}^1/K$ and  Equation~\ref{equation::sandwich_condition_1} implies,
\begin{equation*}
 \frac{N(t_{\mathrm{first}}^1)}{g(N(t_{\mathrm{first}}^1))} =\frac{t_{\mathrm{first}}^1/K}{g(t_{\mathrm{first}}^1/K)}< \frac{1152 }{\max_{i} \Delta_{\sigma_i, \sigma_{i+1}}^2}.
\end{equation*}

Finally since $t_{\mathrm{comm1}}^1 \geq t_{\mathrm{boundary3}}$, and $B(n, \frac{\delta}{4K^2M}) $ is an increasing function of $n$,
\begin{equation*}
    B(f(t_{\mathrm{comm1}}^1), \frac{\delta}{4K^2 M} ) \leq B\left(s_{\mathrm{comm1}}^1, \frac{\delta}{4K^2 M}\right)
\end{equation*}
Finally, following the same argument from $(ii)$ and $(iii)$ above by applying Lemma~\ref{lemma::upper_bound_ratio_scomm1} to the ratio $\frac{s_{\mathrm{comm1}}^1}{g(s_{\mathrm{comm1}}^1)} \leq \frac{162 s_\mathrm{first}^1}{g( s_\mathrm{first}^1)}$  and using the fact that $\mathcal{E}$ holds (and therefore Equation~\ref{equation::sandwich_condition_1}) and that $B( n, \frac{\delta}{4K^2M}) $ is increasing in $n$, we can conclude that 
\begin{align*}
    B\left(s_{\mathrm{comm1}}^1, \frac{\delta}{4K^2 M}\right) &\leq B\left(\frac{186624}{\max_{i} \Delta_{\sigma_i, \sigma_{i+1}}^2}g(s_{\mathrm{comm1}}^1), \frac{\delta}{4K^2 M}\right) \\
    &\leq B\left(\frac{186624}{\max_{i} \Delta_{\sigma_i, \sigma_{i+1}}^2}\log\left(\frac{MK186624}{\delta\max_{i} \Delta_{\sigma_i, \sigma_{i+1}}^2}\right), \frac{\delta}{4K^2 M}\right)
\end{align*}

\end{proof}

\section{ The Zero Test - Supporting Lemmas}\label{section::appendix_zero_test_supporting}

In order to answer this question we will make use of the following Lemmas:

\begin{lemma}\label{lemma::variance_bounded_rv}
If $X$ is a random variable with support in $[0,1]$ with mean $\mu_X$ then: $\mathrm{Var}(X) \leq \mu_X (1-\mu_X)$.
\end{lemma}
\begin{proof}
By definition $\mathrm{Var}(X) = \mathbb{E}[X^2] - \mu_X^2$. Since $X \in [0,1]$ then $\mathbb{E}[X^2] \leq \mathbb{E}[X]$. The result follows.
\end{proof}

\begin{restatable}{lemma}{lemmahoward}[Uniform empirical Bernstein bound]\label{lemma:uniform_emp_bernstein}
In the terminology of \citet{howard2018uniform}, let $S_n = \sum_{i=1}^n Y_i$ be a sub-$\psi_P$ process with parameter $c > 0$ and variance process $W_n$. Then with probability at least $1 - \delta'$ for all $n \in \mathbb{N}$. %
\begin{align*}
    S_n &\leq  1.44 \sqrt{(W_n \vee m) \left( 1.4 \ln \ln \left(2 \left(\frac{W_n}{m} \vee 1\right)\right) + \ln \frac{5.2}{\delta'}\right)}\\
   & \qquad + 0.41 c  \left( 1.4 \ln \ln \left(2 \left(\frac{W_n}{m} \vee 1\right)\right) + \ln \frac{5.2}{\delta'}\right),
\end{align*}
where $m > 0$ is arbitrary but fixed.
\end{restatable}
\begin{proof}
Setting $s = 1.4$ and $\eta = 2$ in the polynomial stitched  boundary in Equation~(10) of \citet{howard2018uniform} shows that $u_{c, \delta'}(v)$ is a sub-$\psi_G$ boundary for constant $c$ and level $\delta$ where 
\begin{align*}
    u_{c, \delta'}(v) &= 1.44 \sqrt{(v \vee 1) \left( 1.4 \ln \ln \left(2 (v \vee 1)\right) + \ln \frac{5.2}{\delta'}\right)}\\
   &\quad  + 1.21 c  \left( 1.4 \ln \ln \left( 2 (v \vee 1)\right) + \ln \frac{5.2}{\delta'}\right).
\end{align*}
By the boundary conversions in Table~1 in \citet{howard2018uniform} $u_{c/3, \delta'}$ is also a sub-$\psi_P$ boundary for constant $c$ and level $\delta'$. The desired bound then follows from Theorem~1 by \citet{howard2018uniform}.
\end{proof}

  We now apply the results of Lemma~\ref{lemma:uniform_emp_bernstein} to a random variable $X$ satisfying the assumptions of Lemma~\ref{lemma::variance_bounded_rv}:

\begin{lemma}\label{lemma::bernstein_concentration_application}
Let $\delta' \in (0,1)$. If $X$ is a random variable with support in $[0,1]$ with mean $\mu_X$ and law $\mathbb{P}_X$ and let $\{ X_i\}_{i=1}^\infty$ be i.i.d. samples from $\mathbb{P}_X$, then with probability at least $1-\delta'$ and for all $n \in \mathbb{N}$ simultaneously:
\begin{equation*}
\mu_X - 2 \sqrt{ \frac{ \min(\mu_X, 1-\mu_X) B(n, \delta')}{n}} - \frac{ B(n, \delta')}{n} \leq \frac{1}{n} \sum_{i=1}^n X_i, 
\end{equation*}
where $B(n, \delta') = 2\ln \ln(2n) + \ln \frac{5.2}{\delta'}$.
\end{lemma}
\begin{proof}
Consider the martingale difference sequence $Y_i = X_i - \mu_X$. The process $S_n = \sum_{i=1}^n Y_i$ with variance process $W_n = n \mathrm{Var}(X)$ satisfies the sub-$\psi_P$ condition of~\citet{howard2018uniform} with constant $c=1$ (see Bennett case in Table 3 of~\citet{howard2018uniform}). By Lemma~\ref{lemma:uniform_emp_bernstein} the bound:
\begin{align*}
    S_n &\leq  1.44 \sqrt{(W_t \vee m) \left( 1.4 \ln \ln \left(2 \left(\frac{W_t}{m} \vee 1\right)\right) + \ln \frac{5.2}{\delta'}\right)}\\
   & \qquad + 0.41 c  \left( 1.4 \ln \ln \left(2 \left(\frac{W_t}{m} \vee 1\right)\right) + \ln \frac{5.2}{\delta'}\right)
\end{align*}
holds for all $n \in \mathbb{N}$ with probability at least $1-\delta'$. Observe that as a consequence of Lemma~\ref{lemma::variance_bounded_rv}, the variance process $W_n$ satisfies $W_n \leq n \min(\mu_X, 1-\mu_X)$. If we set $m = t\mu_X$, we can futher upper bound the RHS as: %
\begin{equation*}
S_n \leq    1.44 \sqrt{n\min(\mu_X, 1-\mu_X) \left(1.4 \ln \ln (2n) + \ln \frac{5.2}{\delta'}  \right)} + 0.41 \left( 1.4 \ln \ln(2n) + \ln \frac{5.2}{\delta'} \right).  
\end{equation*}
The result follows.
\end{proof}

\section{Ancillary Technical Lemmas}

\subsection{Properties of \texorpdfstring{$D(\cdot)$}{D(.)}}

\begin{lemma}\label{lemma::decreasing_D}
The function $D: \mathbb{R} \rightarrow \mathbb{R}$ defined as $D(\ell) = \sqrt{\frac{2g(\ell)}{\ell} }$ for $g(\ell) = \log\left(4\ell^2 MK/\delta \right) $ is increasing for $\ell \geq 1$ whenever $\delta < \frac{1}{2}$.
\end{lemma}

\begin{proof}
Let $c' =4 MK/\delta $ and consider the function $h(\ell) = \frac{\log(c'\ell^2)}{\ell}$. The derivative of $h$ equals
\begin{equation*}
    h'(\ell) = \frac{2 - \log(c'\ell^2)}{\ell^2}.
\end{equation*}
Therefore $h'(\ell) \leq 0$ iff $2 \leq  \log(c'\ell^2)$, which holds iff $\exp(2) \leq c'\ell^2$. As long as $\delta < \frac{1}{2}$, the constant $c' > \exp(2)$ which implies the result. 
\end{proof}

\begin{lemma}\label{lemma::slow_decrease_D}
For any $\ell \geq 1$, and whenever $\delta< \frac{1}{2}$ the function $D(\cdot)$ doesn't decrease too fast:
$$2D(\ell +1 ) > D(\ell)$$ 
\end{lemma}
\begin{proof}
Observe that $\log(4\ell^2 ML/\delta) \leq \log(4(\ell+1)^2 ML/\delta)$ since $\log(\cdot)$ is an increasing function. Similarly for all $\ell \geq 1$ we have that $\sqrt{ \frac{1}{\ell}} \leq  \sqrt{\frac{2}{\ell+1} } < 2\sqrt{\frac{1}{\ell+1} }$. Therefore:
$$D(\ell) = \sqrt{\frac{2 \log(4\ell^2 ML/\delta) }{\ell} } < 2 \sqrt{\frac{2 \log(4(\ell+1)^2 ML/\delta) }{\ell+1} }= 2 D(\ell +1 ). $$ 
\end{proof}

\subsection{Properties of \texorpdfstring{$f(\cdot)$}{f(.)}}\label{section::properties_of_f}

Let's tart by showing that $f(n)$ can be upper-bounded. 

\begin{lemma}\label{lemma::upper_bound_f_ratio_double}
If $\frac{f(n)-1}{B\left(f(n)-1, \frac{\delta}{4K^2M} \right)}  \geq 1$ then 
\begin{equation*}
\frac{f(n)}{B(f(n), \frac{\delta}{4K^2M})} \leq 48\sqrt{ \frac{n/K}{2Kg(n/K)} }
\end{equation*}
\end{lemma}

\begin{proof}
By definition of $f(n)$,
\begin{equation*}
    \frac{f(n)-1}{B(f(n)-1, \frac{\delta}{4K^2M})} < 24 \sqrt{ \frac{n/K}{2Kg(n/K)} } \leq \frac{f(n)}{B(f(n), \frac{\delta}{4K^2M})} 
\end{equation*}

Since $B(f(n), \frac{\delta}{4K^2M}) \geq 1 $ and $B(f(n)-1, \frac{\delta}{4K^2M}) \leq B(f(n), \frac{\delta}{4K^2M})$,
\begin{equation*}
   \frac{f(n)}{B\left(f(n), \frac{\delta}{4K^2M} \right)} -   \frac{f(n)-1}{B\left(f(n)-1, \frac{\delta}{4K^2M} \right)} \leq \frac{f(n)}{B\left(f(n), \frac{\delta}{4K^2M} \right)} -   \frac{f(n)-1}{B\left(f(n), \frac{\delta}{4K^2M} \right)} \leq 1
\end{equation*}
We can conclude that
\begin{equation*}
      \frac{f(n)}{B\left(f(n), \frac{\delta}{4K^2M} \right)}  \leq \frac{f(n)-1}{B\left(f(n)-1, \frac{\delta}{4K^2M} \right)}  + 1 \stackrel{(i)}{\leq}  2\frac{f(n)-1}{B\left(f(n)-1, \frac{\delta}{4K^2M} \right)}  \leq 48 \sqrt{ \frac{n/K}{2Kg(n/K)} }. 
\end{equation*}
Inequality $(i)$ holds because by assumption $\frac{f(n)-1}{B\left(f(n)-1, \frac{\delta}{4K^2M} \right)}  \geq 1$.
\end{proof}

\subsection{Miscellaneous }\label{section::miscellaneous}

The following lemma will prove useful in upper bounding $s_{\mathrm{comm1}}^1$.

\begin{lemma}\label{lemma::supporting_lemma_miscellaneous}
Let $h(x) = \frac{x}{\log(cx)}$ for $c > 0$. Let $x_0$ be the first positive real number such that\footnote{It is easy to see that $h'(x) = \frac{\log(cx) - 1}{\log^2(cx)} $ so that $x_0 = \frac{e}{c}$.} for all $x' \geq x_0$, $h'(x) \geq 0$.  Let $x, b \geq x_0$ such that  $h(x) = \frac{x}{\log(cx)} \leq b$ and $cb \geq 4$ then
\begin{equation*}
    x \leq 4b\log(cb)
\end{equation*}
\end{lemma}

\begin{proof}
We shall show the desired result by the way of contradiction. Let $x' \geq x_0$ be such that $x' > 4b\log(cb)$.  The following inequalities hold 

\begin{align}
    \frac{x'}{\log(cx')} &\stackrel{(i)}{\geq} \frac{4b\log(cb)}{\log(4cb\log(cb))} \notag \\
    &= \frac{4b\log(cb)}{\log(cb) + \log(4) + \log(\log(cb))} \label{equation::miscellaneous_denominator}
\end{align}
Inequality $(i)$ holds because we have assumed $cb \geq 4 > 3$ and therefore $\log(cb) \geq 1$, $b \geq x_0$ and therefore that $4b\log(cb) \geq x_0$.  Since $cb \geq 4 > 1$, it follows that $\log(cb ) \geq \log(\log(cb))$. We can upper bound the denominator of \ref{equation::miscellaneous_denominator} by $3\log(cb)$ thus,
\begin{equation}
     \frac{x'}{\log(cx')} \geq \frac{4b\log(cb)}{3\log(cb) } > b
\end{equation}
Since $x$ is assumed to satisfy $\frac{x}{\log(cx)} \leq b$ this concludes the proof.
\end{proof}

\end{document}